\documentclass[twoside,11pt]{article}

\usepackage[preprint]{jmlr2e} 
\usepackage{blindtext}

\usepackage{lastpage}
\jmlrheading{26}{2025}{1-\pageref{LastPage}}{10/24; Revised
4/25}{9/25}{24-1789}{Gabriel Arpino, Xiaoqi Liu, Julia Gontarek, Ramji Venkataramanan}
\ShortHeadings{Inferring Change Points in High-Dimensional Regression}{Arpino, Liu, Gontarek, and Venkataramanan}

\firstpageno{1}

\usepackage{my_sty}

\begin{document}

\title{Inferring Change Points in High-Dimensional Regression via Approximate Message Passing}


\author{\name Gabriel Arpino \email ga442@cam.ac.uk \\
       \name Xiaoqi Liu \email xl394@cam.ac.uk \\
       \name Julia Gontarek \email jg991@cantab.ac.uk \\
       \name Ramji Venkataramanan \email rv285@cam.ac.uk \\
       \addr Department of Engineering, University of Cambridge \\
       Cambridge, CB2 1PZ, United Kingdom
       }

\editor{Mahdi Soltanolkotabi}

\maketitle

\begin{abstract}%
We consider the problem of localizing change points in a generalized linear model (GLM), a model that covers many widely studied problems in statistical learning including linear, logistic, and rectified linear regression. We propose a novel and computationally efficient approximate message passing (AMP) algorithm for estimating both the signals and the change point locations, and rigorously characterize its performance in the high-dimensional limit where the number of parameters $p$ is proportional to the number of samples $n$. This characterization is in terms of a state evolution recursion, which allows us to precisely compute performance measures such as the asymptotic Hausdorff error of our change point estimates, and allows us to tailor the algorithm to take advantage of any prior structural information of the signals and change points. Moreover, we show how our AMP iterates can be used to efficiently compute a Bayesian posterior distribution over the change point locations in the high-dimensional limit. We validate our theory via numerical experiments, and demonstrate the favorable performance of our estimators on both synthetic and real data in the settings of linear, logistic, and rectified linear regression.
\end{abstract}


\section{Introduction}
Heterogeneity is a common feature of large, high-dimensional data sets. When the data are ordered by time, a simple form of heterogeneity is a change in the data generating mechanism at certain unknown instants of time.  If these `change points' were known, or  estimated accurately, the data set could be partitioned into homogeneous subsets, each amenable to analysis via standard statistical techniques \citep{fryzlewicz_wild_2014}.  Models with change points have been studied in a  variety of statistical contexts, such as the detection of changes in: signal means \citep{wang_univariate_2020,liu2021minimax, li_automatic_2023}; covariance structures \citep{cho2015multiple,wang_optimal_2021b};  graphs \citep{londschien_random_2023,bhattacharjee_change_2020,fan2018approximate}; dynamic networks \citep{wang_optimal_2021a}; and functionals \citep{madrid_padilla_optimal_2022}. Change point models have found application in a range of fields including genomics \citep{braun2000multiple}, neuroscience \citep{AstonKirch12}, causality \citep{he_leveraging_2022}, and economics \citep{andreou2002detecting}.

In this paper, we consider high-dimensional generalized linear models (GLMs) with change points. We are given a sequence of data $(y_i, \X_i) \in \reals \times \reals^p$, for $i\in [n]$, from the model
\begin{align}
 y_i = q((\X_i)^\top \bbeta^{(i)}, \varepsilon_i), \quad i = 1, \dots, n. \label{eq:chgpt-model}
\end{align}
Here, $\bbeta^{(i)} \in \reals^p$ is the unknown regression vector (signal) for the $i$th sample,  $\X_i \in \reals^p$ is the (known) covariate vector, $\varepsilon_i$ is noise, and $q : \reals \times \reals \to \reals$ is a known function. We denote the unknown change points, that is, the sample indices where the regression vector changes,  by $\eta_1, \ldots, \eta_{L^*-1}$. Specifically, we have 
\begin{align*}
   1 = \eta_0 < \eta_1 < \dots < \eta_{L^*} = n + 1, 
\end{align*}
with $\bbeta^{(i)} \neq \bbeta^{(i -1)}$ if and only if $i \in \{ \eta_{\ell}\}_{\ell = 1}^{L^*-1}$.  We note that $L^*$ is the number of distinct  signals in the sequence $\{\bbeta^{(i)}\}_{i=1}^n$, and $(L^*-1)$ is the number of change points. The number of change points is not known, but an upper bound $L$ on the value of $L^*$ is available.  The goal is to estimate the change point locations as well as the $L^*$ signals. We would also like to quantify the uncertainty in these estimates, for example, via confidence sets or a posterior distribution. As we describe below, the model \eqref{eq:chgpt-model} covers many widely studied regression models including linear, logistic, and rectified linear regression. 

\subsection{Linear Regression with Change Points}

In this model, the data $(y_i, \X_i) \in \reals \times \reals^p$ are generated as follows: 
\begin{align}
y_i = (\X_i)^\top \bbeta^{(i)} + \varepsilon_i, \quad i = 1, \dots, n. \label{eq:linear-chgpt-model}
\end{align}
Notice that this corresponds to model \eqref{eq:chgpt-model} with $q(z, v) := z + v$, where for $i \in [n]$ the projected signal $(\X_i)^\top \bbeta^{(i)}$ is observed in  additive noise $\varepsilon_i$. When $L^* = 1$, this reduces to standard linear regression. 

Linear regression with change points in the high-dimensional regime (where the dimension $p$ is comparable to, or exceeds, the number of samples $n$) has been studied in a number of recent works, such as \cite{lee2016lasso, leonardi2016computationally,kaul2019efficient,rinaldo_localizing_2021,xu_change_2022, li_divide_2023, bai2023unified}. Most of these papers consider the setting where the signals are sparse (the number of non-zero entries in  $\bbeta^{(i)} \in \reals^p$ is $o(p)$), and  analyze procedures  that combine the LASSO estimator (or a variant) with a partitioning technique such as dynamic programming.  The  work of \citet{gao_sparse_2022} assumes sparsity on the difference between signals across a change point, and \citet{cho2024detection,liu2024change} consider general non-sparse signals. The recent works of \citet{liu2024change, yang_robust_2024} propose change point detection methods for high-dimensional linear models that do not rely on strict sub-Gaussian error assumptions, with performance comparable to that of \cite{xu_change_2022} in the case of light-tailed additive noise distributions.

\subsection{Logistic Regression with Change Points}
For the logistic model, we first define the function $\zeta(z) := \log{(1 + e^z)}$. The data $(y_i, \X_i) \in \{0, 1\} \times \reals^p$ are then generated according to the following model: 
\begin{align}
    \P\left[y_i = 1 \middle| (\X_i)^\top \bbeta^{(i)} \right] = \frac{e^{(\X_i)^\top \bbeta^{(i)}}}{1 + e^{(\X_i)^\top \bbeta^{(i)}}} = \zeta'((\X_i)^\top \bbeta^{(i)}), \quad i = 1, \dots, n. \label{eq:logistic-chgpt-model}
\end{align}
We may view this as an instance of \eqref{eq:chgpt-model} with $\varepsilon_1, \dots, \varepsilon_n \distas{i.i.d} U[0, 1]$ and $q(z, v) = \ind\{v \leq \zeta'(z) \}$, so that $y_i = q((\X_i)^\top \bbeta^{(i)}, \varepsilon_i) := \ind\{\varepsilon_i \leq \zeta'((\X_i)^\top \bbeta^{(i)})\}$ for each $i \in [n]$. When $L^* = 1$, model \eqref{eq:logistic-chgpt-model} corresponds to the standard logistic model. 

Logistic models are widely used for classification, and logistic regression with change points has been studied in  number of contexts. In epidemiology, for example, logistic regression with unknown change points has been   used to model the relationship between the continuous exposure variable and disease risk \citep{pastor-barriuso_transition_2003}. In medicine, it has been used to identify relevant immune response biomarkers in patients with potentially infectious diseases \citep{fong_change_2015}. 

\subsection{Rectified Linear Regression with Change Points}
In the rectified linear regression model, the data $(y_i, \X_i) \in \reals_{\geq 0} \times \reals^p$ are generated as follows: 
\begin{align}
    y_i = \max\{ (\X_i)^\top \bbeta^{(i)}, 0 \} + \varepsilon_i, \quad i = 1, \dots, n. \label{eq:ReLU-chgpt-model}
\end{align}
This is an instance of \eqref{eq:chgpt-model} with $q(z, v) := \max\{z, 0 \} + v$. When $L^* = 1$, the model \eqref{eq:ReLU-chgpt-model} is referred to as the ReLU (rectified linear unit) with additive noise. The output function acts as a threshold that maps negative values of its first input to zero. 

Rectified linear regression  is an important primitive in the theory of deep learning, and the problem without change points has been extensively studied in machine learning in recent years,  both from the perspective of  designing new algorithms and from that of computational hardness, as seen in the works of \cite{soltanolkotabi_learning_2017,diakonikolas_relu_2021}.

\paragraph{Challenges in the High-Dimensional Setting}
Change point detection in generic, low-dimensional, GLMs was studied by
\cite{hofrichter}, and  \citet{wang_efficient_2023} recently proposed a method for detecting change points in high-dimensional GLMs with sparse regression vectors.
Assuming $s$-sparse signals, they propose a change point estimator that combines an $\ell_1$-penalized estimator with a partitioning technique, and show that it is consistent under the sparsity requirement $s = o(\sqrt{n} / \log{p})$. 

Although existing procedures for high-dimensional change point regression such as \cite{wang_efficient_2023} incorporate sparsity-based constraints, they cannot be easily adapted to take advantage of other kinds of signal priors. Moreover, they are not well-equipped to exploit prior information on the change point locations.  Bayesian approaches to change point detection have been studied in several works, such as those of \cite{fearnhead2006exact, lungu2022changepoint}, however they  mainly focus  on (low-dimensional) time-series. %

\paragraph{Main Contributions} We propose an approximate message passing (AMP) algorithm for estimating the signals $\{\bbeta^{(i)}\}_{i = 1}^n$ and the change point locations $\{ \eta_{\ell}\}_{\ell = 1}^{L^*-1}$ in model \eqref{eq:chgpt-model}. 
Under the assumption that the covariates are i.i.d. Gaussian, we give an exact characterization of the performance of the algorithm in the limit as both the signal dimension $p$ and the number of samples $n$ grow, with $n/p$ converging to a constant $\delta$ (Theorem \ref{thm:SE}). The AMP algorithm is iterative, and defined via a pair of `denoising' functions for each iteration. We show how these functions can be tailored to take advantage of any  prior information on the signals and the change points (Proposition \ref{prop:opt_ensemble_ft_gt}). We then show how the change points can be estimated using the iterates of the AMP algorithm, and how the uncertainty can be quantified via a limiting posterior distribution (Section \ref{sec:chgpt_estimation}). Our theory enables asymptotic guarantees on the accuracy of the change point estimator and on the posterior distribution (Propositions \ref{prop:hausdorff_asymptotics}, \ref{prop:pointwise_posterior}). 

In Section \ref{sec:experiments} we present a range of numerical experiments that demonstrate the superior performance of AMP compared to other state-of-the-art algorithms for high-dimensional regression with change points. We consider synthetic data as well as a real Myocardial Infarctions (MI) data set sorted by age. In the latter, we identify a change point that is consistent with medical literature, and further validate this finding using standard logistic regression analysis. 

Although our theoretical results do not explicitly need assumptions on the separation between adjacent change points, they are most interesting when the separation is of order $n$, that is, $\Delta := \min_{ \ell \in [L]} (\eta_\ell - \eta_{\ell-1}) = O(n) $. This separation is natural in our regime, where the number of samples $n$ is proportional to $p$ and the number of degrees of freedom in the signals also grows linearly in $p$. Existing results on high-dimensional change point regression usually assume signals that are $s$-sparse with $s$ sufficiently small, and demonstrate  change point estimators that are consistent when $\Delta  = \omega\big(s \log p/\kappa^2\big)$, where $\kappa$ is a constant determined by the separation between the signals \citep{wang_2021_statistically, li_divide_2023, wang_efficient_2023}. In contrast, we do not assume signal sparsity that is sublinear in $n$, so the change point estimation error will not converge to zero unless $n/p \to \infty$. We therefore quantify the AMP performance via precise asymptotics for the estimation error and the limiting posterior distribution.

\paragraph{Approximate Message Passing}  AMP, a family of iterative algorithms first proposed for linear regression \citep{Kab03, Don09, Krz12}, has  been applied to a variety of high-dimensional estimation problems including estimation in generalized linear models \citep{Ran11,Sch14,Mai20,Mon21} and low-rank matrix estimation \citep{Fle18,Les17,Mont21, Bar20}. An attractive feature of AMP algorithms is that under suitable model assumptions, their performance in the high-dimensional limit can be characterized by a succinct deterministic recursion called \emph{state evolution}. The state evolution characterization has been used to show that AMP achieves Bayes-optimal performance for some models \citep{Des14,Don13,Bar19}. We refer the reader to Appendix \ref{app:background} for further background on AMP for Generalized Linear Models.

An important feature of our AMP algorithm, in contrast to the above works, is the use of \emph{non-separable} denoising functions.
(We say a function $g : \reals^{m \times L} \to \reals^{m \times L}$ is separable if it acts row-wise identically on the input matrix.)  Even with simplifying assumptions on the signal and noise distributions, non-separable AMP denoisers are required to handle the heterogeneous dependence structure caused by change points, and allow for precise uncertainty quantification around possible change point locations. Our main state evolution result (Theorem \ref{thm:SE}) leverages recent results by \citet{berthier_state-evolution_2019} and \citet{gerbelot_graph-based_2023} for AMP with non-separable  denoisers.   Non-separable denoisers for AMP have been been studied for generalized linear models and for matrix tensor product models, to exploit the dependence within the signal \citep{som2012compressive, metzler2016denoising, ma2019analysis, rossetti2023approximatemessagepassingmatrix} or between the covariates \citep{zhang_spectral_2023}. In our setting, non-separable denoisers are required to exploit the heterogeneous dependence structure caused by change points.

Although we assume i.i.d. Gaussian covariates, based on recent AMP universality results \citep{Wang22Universality}, we expect the results apply to a broad class of i.i.d. designs. An interesting direction for future work is to generalize our results to rotationally invariant designs, a much broader class for which AMP-like algorithms have been proposed for regression without change points  \citep{ma2017orthogonal, rangan2019vector, takeuchi2020rigorous, pandit2020inference}. 

\section{Preliminaries} \label{sec:prelim}
\paragraph{Notation}
We let $[n]=\{1, 2, \dots, n\}$.
We use boldface notation for matrices and vectors. For vectors $\x, \y \in \reals^n$, we write $\x \leq \y$ if $x_i \leq y_i$ for all $i \in [n]$, and let $[\x, \y] := \{\v \in \reals^n : x_i \leq v_i \leq y_i \; \forall i \in [n]\}$. 
For a matrix $\A \in \reals^{m \times L}$ and $i \in [m], j \in [L]$, we let $\A_{[i, :]}$ and $\A_{[:, j]}$ denote its $i$th row and $j$th column respectively. Similarly, for a vector $\bpsi \in [L]^n$, we let $\A_{[:, \bpsi]} \in \reals^m$ denote the vector whose  $i$-th entry is $\A_{i, \psi_i}$. 
For two sequences (in $n$) of random variables $X_n$, $Y_n$, we write $X_n \stackrel{\P}\simeq Y_n$ when their difference converges 
in probability to 0, that is, $\lim_{n\to \infty}\P(|X_n - Y_n|>\epsilon)=0$ for any $\epsilon > 0$.
Denote  the covariance matrix of random vector $\Z\in \R^q$ as $\cov(\Z)\in\R^{q\times q}$. 
We refer to all random elements, including  vectors and matrices, as random variables. When referring to probability densities, we include probability mass functions, with integrals interpreted as sums when the distribution is discrete.
Given a matrix $\B \in \reals^{p \times L}$, the empirical distribution of the entries of $\B$ is the measure $\mu_p(A) := \frac{1}{p}|\{ \B_j \in A: j \in [p]\}|$ for any  measurable set $A \subseteq \reals^L$. We say that the empirical distribution of the entries of $\B$ converges weakly to the law $\P_{\bar{\B}}$ if $ \frac{1}{p}\sum_{j=1}^p f(\B_j) \to \E_{\P_{\bar{\B}}}[f(\bar{\B})]$ as $p \to \infty$ for all bounded continuous functions $f: \reals^L \to \reals$. 

\paragraph{Model Assumptions}
In model \eqref{eq:chgpt-model}, we assume independent Gaussian covariate vectors $\X_i \distas{i.i.d} \N(0, \I_p/n)$ for $i \in [n]$. We consider the high-dimensional regime where $n, p \to \infty$ and $\frac{n}{p}$ converges to a constant $\delta > 0$. Following the change point literature, we assume the number of change points $(L^*-1)$ is fixed and does not scale with $n,p$. 
\paragraph{Pseudo-Lipschitz Functions} We will state our results in terms of uniformly \textit{pseudo-Lipschitz} functions \citep{berthier_state-evolution_2019}. For $C > 0$ and $r \in [1, \infty)$, let $PL_{n, m, q}(r, C)$ be the set of functions $\phi: \reals^{n \times q} \to \reals^{m \times q}$ such that 
$$\frac{\|\phi(\x) - \phi(\tilde{\x})\|_F}{\sqrt{m}} \leq C\left(1 + \left(\frac{\|\x\|_F}{\sqrt{n}}\right)^{r - 1} + \left(\frac{\|\tilde{\x}\|_F}{\sqrt{n}}\right)^{r - 1}\right)\frac{\|\x - \tilde{\x} \|_F}{\sqrt{n}}, \quad \forall \x, \tilde{\x} \in \reals^{n\times q}.$$ 
Note that $\cup_{C > 0} PL_{n, m, q}(r_1, C) \subseteq \cup_{C > 0} PL_{n, m, q}(r_2, C)$ for any $1 \leq r_1 \leq r_2$. A function $\phi \in PL_{n, m, q}(r, C)$ is called pseudo-Lipschitz of order $r$. A family of pseudo-Lipschitz functions is said to be \textit{uniformly} pseudo-Lipschitz if all functions of the family are pseudo-Lipschitz with the same order $r$ and the same constant $C$.  For $\x, \y \in \reals^n$, the mean squared error $\phi(\x,\y)=\langle \x - \y, \x - \y \rangle/n$ and the normalized  squared correlation $\phi(\x, \y)=|\langle \x, \y \rangle|/n$ are examples of uniformly pseudo-Lipschitz functions. 

\section{AMP Algorithm and Main Results} \label{sec:main}
We stack the feature vectors to form the design matrix $\X = [\X_1, \dots, \X_n]^\top\in \reals^{n\times p}$. Similarly, we stack the observations and noise elements to form $\y:= [y_1, \dots, y_n]^\top \in\R^n$ and $\bvarepsilon := [\varepsilon_1, \dots, \varepsilon_n]^\top \in \reals^n$, respectively.  
Let $\eeta := [ \eta_1,\dots,  \eta_{L^*}, \dots,  \eta_{L}]$ be the vector containing the true change points, and let \[\B := [\bbeta^{(\eta_1)}, \dots, \bbeta^{(\eta_{L^*})}, \dots,  \bbeta^{(\eta_L)}] \in \reals^{p \times L}\] be the matrix containing the true signals. Since the algorithm assumes no knowledge of $L^*$, other than $L^*\le L$, the  columns $L^*+1, \ldots, L$ of $\bB$ can all be taken to be  zero. Similarly, $\eta_{L^*} = \eta_{L^*+1}=\dots =\eta_{L}=n$.

Recall that for $i\in [n]$, each observation $y_i$ is generated from model \eqref{eq:chgpt-model}. 
Let $\bPsi \in [L]^n$ be the \textit{signal configuration vector}, whose $i$-th entry stores the  index of the signal underlying observation $y_i.$ That is, for  $i \in [n]$ and $\ell \in [L]$, let
$\Psi_i = \ell$ if and only if $\bbeta^{(i)}$ equals the $\ell$th column of the signal matrix $\bB$. We note that there is a one-to-one correspondence between $\bPsi$ and the change point vector $\eeta$. With a slight abuse of notation, we can rewrite the response vector $\y$ in a more general form:
\begin{equation} 
\y =: q(\X \B, \bPsi, \bvarepsilon)  \in \reals^n, \label{eq:model_psi}
\end{equation}
where $q$ is the known function in model \eqref{eq:chgpt-model}, expanded to act row-wise on matrix inputs and to incorporate $\bPsi$, where $\left( q(\X \B, \bPsi, \bvarepsilon) \right)_i := q(\X \bbeta^{(\Psi_i)}, \varepsilon_i)$ for $i \in [n]$.  We note that the mixed linear regression model \citep{Yi14, zhang2022precise,tan_mixed_2023} can also be written in the form in \eqref{eq:model_psi}, with a crucial difference. In mixed regression,
the components of $\bPsi$ are assumed to be drawn independently from some distribution on $[L]$, that is, 
each $y_i$ is independently generated from one of the $L$ signals. In the change point setting, the entries of $\bPsi$ are \emph{dependent}, since they change value only at indices $\eta_1, \ldots, \eta_{(L^*-1)}$. 
\paragraph{AMP Algorithm} We now describe the AMP algorithm for estimating  $\B$ and $\eeta$. In each iteration $t \ge 1$, 
the algorithm produces an updated estimate of the signal matrix $\B$, which we call $\B^t$, and of the linearly transformed signal $\TTheta:=\X\B$, which we call $\TTheta^t$. 
These estimates have distributions that can be described by a deterministic  low-dimensional matrix recursion called \textit{state evolution}. In Section \ref{sec:chgpt_estimation}, we show how the estimate $\TTheta^t$ can be combined with $\y$ to infer $\eeta$ with precisely quantifiable error.

Starting with an initializer $\B^0\in\reals^{p\times L}$ and defining $\hat{\bR}^{-1}:=\0_{n\times L}, $ for $t \geq 0$ the algorithm computes:
\begin{align}
\begin{split}
\label{eq:amp}
    &\TTheta^{t} = \X \hat{\B}^t -  \hat{\bR}^{t-1} (\F^t)^\top\,, \; \; \; \hat{\bR}^t = g^t\left(\TTheta^t, \y\right) \, ,  \\
    &\B^{t+1} = \X^\top \hat{\bR}^t -  \hat{\B}^{t} (\C^t)^\top\,, \; \; \; \hat{\B}^{t} = f^{t}\left(\B^{t}\right), \, \\
\end{split}
\end{align}
where the denoising functions $g^t: \R^{n\times L}\times \R^{n}\to \R^{n\times L}$ and $f^t: \R^{p\times L}\to \R^{p\times L}$ are used to define the matrices $\F^t$, $\C^t$ as follows:
\begin{align*}
    \C^t &= \frac{1}{n} \sum_{i=1}^n \partial_i{g^t_i}\left(\TTheta^t, \y\right)\, , \; \; \; \F^{t} = \frac{1}{n} \sum_{j=1}^p \d_j{f_j^{t}}(\B^{t}).
\end{align*}
Here $\partial_i{g^t_i}\left(\TTheta, \y\right)$ is the $L\times L$ Jacobian  of $g^t_i$ w.r.t. the $i$th row of $\TTheta$.
Similarly, $\d_j{f_j^{t}}(\B^{t})$ is the $L \times L$ Jacobian of $f_j^t$ with respect to the $j$-th row of its argument. The time complexity of each iteration in \eqref{eq:amp} is $O(npL + r_n)$, where $r_n$ is the time complexity of computing $f^t, g^t$. 

Crucially, the denoising functions $g^t$ and $f^t$ are {not restricted} to being separable. (Recall that a separable function acts row-wise identically on its input, that is,  $g : \reals^{m \times L} \to \reals^{m \times L}$ is separable if for all $\U \in \reals^{m \times L}$ and $i \neq j$, we have $[g(\U)]_i = g_i(\U_i) = g_j(\U_i)$.) 
Hence, in general, we have 
\begin{equation*}
g^t(\TTheta, \y) = \begin{bmatrix}
g^t_1(\TTheta, \y) \\ \vdots \\g^t_n(\TTheta, \y)
\end{bmatrix}
, \quad
f^t(\B) = \begin{bmatrix}
f_1^t(\B) \\ \vdots \\ f_n^t(\B) 
\end{bmatrix}.
\end{equation*}
Our non-separable approach is required to handle the heterogeneous dependence structure created by the change points. For example, if there was one change point uniformly distributed in  $[n]$, then $g^t(\TTheta, \y)$ should take into account that $y_i$ is more likely to have come from the the first signal for $i$ close to $1$, and from the second signal for $i$ close to $n$. This is in contrast to existing AMP algorithms for mixed regression \citep{tan_mixed_2023}, where  under standard model assumptions, it suffices to consider separable denoisers.

\paragraph{State Evolution}
The memory terms $-\hat{\bR}^{t-1} (\F^t)^\top$ and $-\hat{\B}^{t} (\C^t)^\top$ in our AMP algorithm \eqref{eq:amp} debias the iterates $\TTheta^t$ and $\B^{t+1}$, and enable a succinct distributional characterization.  In the high-dimensional limit as $n,p\to\infty$ (with $n/p \to \delta$),  the empirical distributions of $\TTheta^t$ and $\B^{t+1}$ are quantified through the random variables $\V_{\TTheta}^t$ and $\V_{\B}^{t+1}$ respectively, where 
\begin{align}
&\V_{\TTheta}^t := \Z \brho^{-1} \bnu^t_{\TTheta} + \G^t_{\TTheta} \,  \in \reals^{n \times L}, \label{eq:V_TTheta}\\ 
&\V_{\B}^{t+1} := \B \bnu^{t+1}_{\B} + \G^{t+1}_{\B} \, \in \reals^{p \times L}. \label{eq:V_B}
\end{align}
The matrices $\brho, \bnu_{\TTheta}^t, \bnu_{\B}^{t+1} \in \reals^{L \times L}$ are deterministic and defined below. The random matrices $\Z$, $\bG_{\TTheta}^t$, and $\bG_{\B}^{t+1} $ are independent of $\X$, and have i.i.d. rows following a Gaussian distribution. Namely, for $i \in [n]$ we have $Z_i \distas{i.i.d} \N(\0, \brho)$. For $i \in [n]$ and $s, r \geq 0$, $(\G_{\TTheta}^t)_i \distas{i.i.d} \N(\0, \bkappa_{\TTheta}^{t, t})$ with $\cov((\G_{\TTheta}^r)_i, (\G_{\TTheta}^s)_i) = \bkappa_{\TTheta}^{r, s}$. Similarly, for $j \in [p]$, $(\G_{\B}^t)_j \distas{i.i.d} \N(\0, \bkappa_{\B}^{t, t})$ with $\cov((\G_{\B}^r)_j, (\G_{\B}^s)_j) = \bkappa_{\B}^{r, s}$. The $L\times L$ deterministic matrices $\bnu_{\TTheta}^t, \bkappa_{\TTheta}^{r,s}, \bnu_{\B}^{t},$ and $\bkappa_{\B}^{r,s}$ are defined below via the \emph{state evolution} recursion.

Given the initializer $\B^0$ for our algorithm \eqref{eq:amp}, we initialize state evolution by setting $\bnu_{\TTheta}^0 := \0$, and
\begin{align}
    \brho := \frac{1}{\delta}  \lim_{p \to \infty} \frac{1}{p} \B^\top \B, \qquad  \bkappa^{0, 0}_{\TTheta} := \frac{1}{\delta}  \lim_{p \to \infty} \frac{1}{p} f^0({\B}^0)^\top f^0({\B}^0). \label{eq:kappa_theta_0} 
\end{align}
 Let $\tilde{g}_i^t(\Z, \V_{\TTheta}^t, \bPsi, \bvarepsilon) := g_i^t(\V_{\TTheta}^t, q(\Z, \bPsi, \bvarepsilon))$ and let $\partial_{1i} \tilde{g}_i^t$ be the partial derivative (Jacobian) w.r.t. the $i$th row of the first argument.  Then, the state evolution matrices are defined recursively as follows:
\begin{align}
&{\bnu}_{\B}^{t+1} := \lim_{n \to \infty} \frac{1}{n} \E\left[\sum_{i=1}^n \partial_{1i} \tilde{g}_i^t(\Z, \V_{\TTheta}^t, \bPsi, \bvarepsilon) \right], \label{eq:nu_B_SE}\\
&\bkappa_{\B}^{s+1, t+1} := \lim_{n \to \infty} \frac{1}{n} \E\left[g^s\left( \V_{\TTheta}^s, q(\Z, \bPsi, \bvarepsilon)\right)^\top g^t\left(\V_{\TTheta}^t, q(\Z, \bPsi, \bvarepsilon)\right)  \right], \label{eq:kappa_B_SE} \\
&\bnu^{t+1}_{\TTheta} := \frac{1}{\delta} \lim_{p \to \infty} \frac{1}{p} \E\left[\B^\top f^{t+1}(\V_{\B}^{t+1}) \right], \label{eq:nu_theta_SE}\\
&\bkappa_{\TTheta}^{s+1, t+1} := \frac{1}{\delta} \lim_{p \to \infty} \frac{1}{p} \E\left[\left(f^{s+1}(\V_{\B}^{s+1}) - \B\brho^{-1} \bnu^{s+1}_{\TTheta}\right)^\top \left(f^{t+1}(\V_{\B}^{t+1}) - \B\brho^{-1} \bnu^{t+1}_{\TTheta} \right) \right]. \label{eq:kappa_theta_SE}
\end{align}
The expectations above are taken with respect to $ \Z, \V_{\TTheta}^t, \V_{\TTheta}^s, \V_{\B}^{t+1}$ and $\V_{\B}^{s+1}$, and depend on  $g^t$, $f^t$, $\B$,  
$\bvarepsilon$, and $\bPsi$. 
The limits in \eqref{eq:nu_B_SE}--\eqref{eq:kappa_theta_SE} exist under suitable regularity conditions on $f_t, g_t$,  and on the limiting empirical distributions on $\B, \bvarepsilon$; see  Appendix \ref{app:SE_limits}.
The parametric dependence on $\B$, $\bvarepsilon$ can also be removed under these conditions---this is discussed in the next subsection.
\subsection{State Evolution Characterization of AMP Iterates}\label{sec:master_theorem}

Recall that the matrices \eqref{eq:kappa_theta_0}--\eqref{eq:kappa_theta_SE} are used to define the random variables $(\V_{\TTheta}^t, \V_{\B}^{t+1})$ in \eqref{eq:V_TTheta},\eqref{eq:V_B}. Through these quantities, we now give a precise characterization of the AMP iterates $(\TTheta^t$, $\B^{t+1})$ in the high-dimensional limit. Theorem \ref{thm:SE} below shows that any real-valued pseudo-Lipschitz function of $(\TTheta^t$, $\B^{t+1})$ converges in probability to its expectation under the limiting random variables $(\V_{\TTheta}^t, \V_{\B}^{t+1})$. 
In addition to the model assumptions in Section \ref{sec:prelim}, we make the following assumptions:
\begin{enumerate}[font={\bfseries},label={(A\arabic*)}]
\item \label{it:main-res-ass-1} The following limits exist and are finite almost surely: 
$ \lim_{p \to \infty} \|\B^0\|_F/\sqrt{p}$, \\ $\lim_{p \to \infty} \|\B^\top \B\|_{F}/p$,  and $\lim_{n \to \infty} \|\bvarepsilon\|_2/\sqrt{n}$.
\item \label{it:main-res-ass-2} For each $t \geq 0$, let $\tilde{g}^t: (\u, \z) \mapsto g^t(\u, q(\z, \bPsi, \bvarepsilon))$, where $(\u, \z) \in \reals^{n \times L} \times \reals^{n \times L} $. For each $i \in [n], \; j \in [p],$ the following functions are uniformly pseudo-Lipschitz: $f^t, \tilde{g}^t, \d_j f_j^t, \partial_{1i} \tilde{g}_i^{t}$. 
\item \label{it:main-res-ass-3} For $s, t \geq 0$, the limits in \eqref{eq:kappa_theta_0}--\eqref{eq:kappa_theta_SE} exist and are finite.
\end{enumerate}
Assumptions $\ref{it:main-res-ass-1} - \ref{it:main-res-ass-3}$ are natural extensions of those required for classical AMP results in settings with separable denoising functions. They are  similar to those required by existing works on non-separable AMP  \cite{berthier_state-evolution_2019, gerbelot_graph-based_2023}, and generalize these to the model \eqref{eq:model_psi}, with a matrix signal $\bB$ and an auxiliary vector  $\bPsi \in [L]^n$. 

\begin{theorem} \label{thm:SE}
Consider the model \eqref{eq:model_psi} and the AMP algorithm in \eqref{eq:amp}, with the model assumptions in Section \ref{sec:prelim} as well as $\ref{it:main-res-ass-1}-\ref{it:main-res-ass-3}$. Then for $t \geq 0$ and any sequence of uniformly pseudo-Lipschitz functions $\bphi_{p}(\cdot \; ; \B) : \reals^{p \times (L(t+1))} \to \reals$ and $\bphi_{n}(\cdot \;; \bPsi, \bvarepsilon) : \reals^{n \times (L(t+2))} \to \reals$,
\begin{align}
&\bphi_n(\TTheta^0, \dots, \TTheta^t, \X \B ; \bPsi, \bvarepsilon) \stackrel{\P}{\simeq} \E_{\V_{\TTheta}^0, \dots, \V_{\TTheta}^t, \Z} \{ \bphi_n(\V_{\TTheta}^0, \dots, \V_{\TTheta}^t, \Z ; \bPsi, \bvarepsilon)\} \label{eq:SE_Theta}, \\
&\bphi_p(\B^1, \dots, \B^{t+1} ; \B) \stackrel{\P}{\simeq} \E_{\V_{\B}^1, \dots, \V_{\B}^{t+1}} \{ \bphi_p(\V_{\B}^1, \dots, \V_{\B}^{t+1} ; \B) \} \label{eq:SE_B},
\end{align}
as $n, p \to \infty$ with $n/p \to \delta$, where the random variables $\Z, \V_{\TTheta}^t$ and $\V_{\B}^{t+1}$ are defined in \eqref{eq:V_TTheta}, \eqref{eq:V_B}.
\end{theorem}
The proof of the theorem is given in Appendix \ref{appendix:SEproof}.  It involves reducing the AMP in \eqref{eq:amp} to a variant of the symmetric AMP iteration analyzed in \citep[Lemma 14]{gerbelot_graph-based_2023}. We use a generalized version of their iteration which allows for the auxiliary quantities $\X \B, \bPsi, \bvarepsilon$ to be included. Theorem \ref{thm:SE} implies that any pseudo-Lipschitz function $\bphi_n(\TTheta^t, \X\B ; \bPsi, \bvarepsilon) := \tilde{\bphi}_n(\TTheta^t, q(\X \B, \bPsi, \bvarepsilon)) = \tilde{\bphi}_n(\TTheta^t, \y)$ will converge in probability to a quantity involving an expectation over $(\V_{\TTheta}^t, \Z)$.  An analogous statement holds for $(\B^{t+1}, \B)$. 
\paragraph{Evaluating Performance Metrics} Using Theorem \ref{thm:SE}, we can evaluate performance metrics such as the mean squared error
between the signal matrix $\B$ and the estimate $\hat{\B}^t = f^t(\B^t)$. Taking $\varphi_p(\B^t, \B) = \|f^t(\B^t) - \B\|^2_F / p$ leads to $\|f^t(\B^t) - \B\|^2_F / p \stackrel{\P}\simeq \E[\|f^t(\V_{\B}^t) - \B\|^2_F / p]$, where the limiting value of the RHS (as $p \to \infty$) can be precisely computed  under suitable assumptions. In Section \ref{sec:chgpt_estimation}, we will choose $\varphi_n$ to capture metrics of interest for estimating change points.
\paragraph{Special Cases}
Theorem \ref{thm:SE} recovers two known special cases of (separable) AMP results, with complete convergence replaced by convergence in probability: linear regression when $L = 1$ \citep{feng_unifying_2022}, and mixed linear regression where $L >1$ and the empirical distributions of the rows of $\B, \bPsi, \bvarepsilon$ converge weakly to laws of well-defined random variables \citep{tan_mixed_2023}.  
\paragraph{Parametric Dependence of State Evolution on $\B$, $\bvarepsilon$, $\bPsi$}
The parametric dependence of the state evolution parameters in \eqref{eq:nu_B_SE}-\eqref{eq:kappa_theta_SE} on 
$\B, \bvarepsilon$ can be removed under reasonable assumptions.   A standard assumption in the AMP literature \citep{feng_unifying_2022} is: 
\begin{enumerate}[font={\bfseries},label={(S0)}]
\item \label{it:simp-ass-1} As $n, p \to \infty$, the empirical distributions of $\{\B_j \}_{j \in [p]}$  and $\{\varepsilon_i\}_{i \in [n]}$ converge weakly to laws $\P_{\bar{\B}}$  and  $\P_{\bar{\varepsilon}}$, respectively,   with bounded second moments. 
\end{enumerate}
In  Appendix \ref{app:SE_limits}, we give
conditions on $f_t, g_t$, which together with  \ref{it:simp-ass-1}, allow 
the state evolution equations to be simplified and written in terms of $\bar{\bB} \sim \P_{\bar{\B}}$ and $\bar{\varepsilon} \sim \P_{\bar{\varepsilon}}$ instead of $(\B, \bvarepsilon)$.      
We believe that the parametric dependence of the state evolution matrices on the signal configuration vector $\bPsi$ is fundamental. Since  the entries of $\bPsi$ change value only at a finite number of  change points, the state evolution parameters will depend on the limiting fractional values of these change points; see \ref{it:simp-ass-2} in Appendix \ref{app:SE_limits}. This is also consistent with recent change point regression literature, where the limiting distribution of the change point estimators in \citep{xu_change_2022} is shown to be  a function of the data generating mechanism.
\subsection{Choosing the Denoising Functions $f^t, g^t$} \label{sec:opt_denoisers}
The  performance of the AMP algorithm \eqref{eq:amp}  is determined by the functions $\{f^{t+1}, g^t\}_{t \geq 0}$. We  now describe how these functions can be chosen based on the state evolution recursion  to maximize estimation performance.  Using $(\V_{\TTheta}^t, \V_{\B}^{t+1})$ defined in \eqref{eq:V_TTheta}--\eqref{eq:V_B}, we define the  matrices
\begin{align}
    &\tilde{\Z}^t: = \V_{\TTheta}^t (\brho^{-1} \bnu_{\TTheta}^t)^{-1} = \Z + \G_{\TTheta}^t(\brho^{-1} \bnu_{\TTheta}^t)^{-1}, \label{eq:Z_tilde} \\
    &\tilde{\B}^{t+1}:= \V_{\B}^{t+1}(\bnu_{\B}^{t+1})^{-1} = \B + \G_{\B}^{t+1} (\bnu_{\B}^{t+1})^{-1}, \label{eq:B_tilde}
\end{align}
where for $i \in [n], j \in [p]$ we have $\Z_i \distas{i.i.d} \N(\bzero, \brho)$, $\G_{\TTheta, i}^t \distas{i.i.d} \N(\0, \bkappa_{\TTheta}^{t, t})$, and $\G_{\B, j}^t \distas{i.i.d} \N(\0, \bkappa_{\B}^{t, t})$. (If the inverse doesn't exist we post-multiply by the pseudo-inverse). A  natural objective is to minimize the trace of the covariance of the ``noise'' matrices in \eqref{eq:Z_tilde} and \eqref{eq:B_tilde}, given by
\begin{align}
&\trace\Bigg( \frac{1}{n}\sum_{i=1}^n\cov\left(\tilde{\Z}_i^t-\Z_i\right) \Bigg) = \trace\left( \left((\bnu_{\TTheta}^t)^{-1}\brho\right)^\top\bkappa_{\TTheta}^{t,t} (\bnu_{\TTheta}^t)^{-1}\brho\right) , \label{eq:trace_theta}\\
&\trace\Bigg( \frac{1}{p}\sum_{j=1}^p\cov
\left(\tilde{\B}_i^{t+1}-\B_i\right) \Bigg) = \trace\left([(\bnu_{\B}^{t+1})^{-1}]^\top\bkappa_{\B}^{t+1,t+1}(\bnu_{\B}^{t+1})^{-1}\right), \label{eq:trace_B}
\end{align}
where we recall from \eqref{eq:nu_B_SE}--\eqref{eq:kappa_theta_SE} that $\bnu_{\TTheta}^{t}, \bkappa_{\TTheta}^{t, t}$ are defined by $f^{t}$, and $\bnu_{\TTheta}^{t+1}, \bkappa_{\TTheta}^{t+1, t+1}$ by $g^t$. 

For $t \ge 1$, we would like to iteratively choose the denoisers $f^t, g^t$ to minimize the quantities on the RHS of \eqref{eq:trace_theta} and \eqref{eq:trace_B}. However, as discussed above $\bnu_{\TTheta}^{t}, \bkappa_{\TTheta}^{t, t}, \bnu_{\TTheta}^{t+1}, \bkappa_{\TTheta}^{t+1, t+1}$ depend on the unknown $\bPsi$, so any denoiser construction based quantities  cannot be executed in practice.

\paragraph{Ensemble State Evolution} 
To remove the parametric dependence on $\bPsi$ in the state evolution equations, we can postulate randomness over this variable and take expectations accordingly. Indeed, assume that \ref{it:simp-ass-1} holds, and postulate a prior distribution $\pi_{\bar{\bPsi}}$ over $[L]^n$.   
For example, $\pi_{\bar{\bPsi}}$ may be the uniform distribution over all the signal configuration vectors with two change points that are at least $n/10$ apart.  We emphasize that our theoretical results do not assume that the true  $\bPsi$ is  drawn according to $\pi_{\bar{\bPsi}}$. Rather,  the prior $\pi_{\bar{\bPsi}}$ allows us to encode any  knowledge we may have about the change point locations, and  use it to define efficient AMP  denoisers $f^t, g^t$. This is done via the following 
 \textit{ensemble state evolution} recursion, defined in terms of the independent random variables $\bar{\bPsi} \distas{} \P_{\bar{\bPsi}}$, $\bar{\bB} \distas{} \P_{\bar{\bB}} $, and $\bar{\varepsilon} \distas{} \P_{\bar{\varepsilon}}$.
 
Starting with initialization  $\bar{\bnu}_{\TTheta}^0 := \0$, $\bar{\bkappa}^{0, 0}_{\TTheta} := \lim_{p \to \infty}  \frac{1}{\delta p} f^0({\B}^0)^\top f^0({\B}^0)$, for $t \ge 0$ define:
\begin{align}
&\bar{\bnu}_{\B}^{t+1} := \lim_{n \to \infty} \frac{1}{n} \sum_{i=1}^n \E\left[\partial_{1} \tilde{g}_i^t(\Z_1, (\bar{\V}_{\TTheta}^t)_1, \bar{\Psi}_i, \bar{\varepsilon}) \right] \label{eq:nu_bar_B_SE} \\
&\bar{\bkappa}_{\B}^{t+1, t+1} := \lim_{n \to \infty} \frac{1}{n} \sum_{i=1}^n \E\left[g_i^t\left( (\bar{\V}_{\TTheta}^t)_1, q(\Z_1, \bar{\Psi}_i, \bar{\varepsilon})\right) g_i^t\left((\bar{\V}_{\TTheta}^t)_1, q(\Z_1, \bar{\Psi}_i, \bar{\varepsilon})\right)^\top  \right], \label{eq:kappa_bar_B_SE} \\
&\bar{\bnu}^{t+1}_{\TTheta} := \frac{1}{\delta} \lim_{p \to \infty} \frac{1}{p} \sum_{j = 1}^p \E\left[\bar{\B} f_j^{t+1}(\bar{\V}_{\B}^{t+1})^\top \right], \label{eq:nu_bar_theta_SE}\\
&\bar{\bkappa}_{\TTheta}^{t+1, t+1} := \frac{1}{\delta} \lim_{p \to \infty} \frac{1}{p} \sum_{j = 1}^p \E\left[\left(f_j^{t+1}(\bar{\V}_{\B}^{t+1}) - (\bar{\bnu}^{t+1}_{\TTheta})^\top \brho^{-1} \bar{\B} \right) \right. \nonumber \\
&\hspace{7cm} \left. \left(f_j^{t+1}(\bar{\V}_{\B}^{t+1}) -(\bar{\bnu}^{t+1}_{\TTheta})^\top \brho^{-1} \bar{\B} \right)^\top \right], \label{eq:kappa_bar_theta_SE}
\end{align}
where 
\begin{align}
&\bar{\V}_{\TTheta}^t := {\Z} \brho^{-1} \bar{\bnu}^t_{\TTheta} + \bar{\G}^t_{\TTheta} \in \reals^{n \times L}, \label{eq:V_bar_TTheta}\\ 
&\bar{\V}_{\B}^{t+1} := (\bar{\bnu}^{t+1}_{\B})^\top \bar{\B}  + (\bar{\G}^{t+1}_{\B})_1 \in \reals^{L}, \label{eq:V_bar_B}
\end{align}
and for $i \in [n], j \in [p]$ we have that $\Z_i \distas{i.i.d} \N(\bzero, \brho)$,  $\bar{\G}_{\TTheta, i}^t \distas{i.i.d} \N(\0, \bar{\bkappa}_{\TTheta}^{t, t})$ and  $\bar{\G}_{\B, j}^t \distas{i.i.d} \N(\0, \bar{\bkappa}_{\B}^{t, t})$.  

When $\pi_{\bar{\bPsi}}$ is a unit mass on the true configuration  $\bPsi$,  \eqref{eq:nu_bar_B_SE}--\eqref{eq:kappa_bar_theta_SE} reduce to the simplified state evolution  in Appendix \ref{app:SE_limits}. The limits in \eqref{eq:nu_bar_B_SE}--\eqref{eq:kappa_bar_theta_SE} exist under suitable regularity conditions on $f^t, g^t$, such as those in Appendix \ref{app:SE_limits}.

 We now propose a construction of $f^t, g^t$ based on minimizing the following alternative objectives to \eqref{eq:trace_theta}--\eqref{eq:trace_B}: 
\begin{align}
&\trace\left( \left((\bar{\bnu}_{\TTheta}^t)^{-1}\brho\right)^\top\bar{\bkappa}_{\TTheta}^{t,t} (\bar{\bnu}_{\TTheta}^t)^{-1}\brho\right), \label{eq:alternate_trace_theta} \\
&\trace\left([(\bar{\bnu}_{\B}^{t+1})^{-1}]^\top \bar{\bkappa}_{\B}^{t+1,t+1}(\bar{\bnu}_{\B}^{t+1})^{-1}\right), \label{eq:alternate_trace_B}
\end{align}
where the deterministic matrices $\bar{\bnu}_{\TTheta}^t$, $\bar{\bkappa}_{\TTheta}^{t,t}$, 
$\bar{\bnu}_{\B}^{t+1}$, $\bar{\bkappa}_{\B}^{t+1,t+1} \in \reals^{L \times L}$ are defined in \eqref{eq:nu_bar_B_SE}--\eqref{eq:kappa_bar_theta_SE}.

%
\begin{proposition}\label{prop:opt_ensemble_ft_gt}
   Assume the limits in \eqref{eq:nu_bar_B_SE}--\eqref{eq:kappa_bar_theta_SE} exist. Then, for $t\ge 1$:
    \begin{enumerate}
        \item \label{enum:ensemble_ft}
    Given $\bar{\bnu}_{\B}^t$, $\bar{\bkappa}_{\B}^{t,t}$, the quantity \eqref{eq:alternate_trace_theta} is minimized when 
\begin{align}\label{eq:opt_ensemble_ft_lem}
f_j^t(\U) = f_{j}^{*t}(\U) :=  \E[\bar{\B} | \bar{\V}_{\B}^t = \U_j],
\end{align}
for $\U \in \reals^{p \times L}, j \in [p]$.
\item \label{enum:ensemble_gt}
Given $\bar{\bnu}_{\TTheta}^t, \bar{\bkappa}_{\TTheta}^{t,t}$, the quantity \eqref{eq:alternate_trace_B} is minimized when \begin{align}
& g_i^t(\V, \u) = g_i^{*t} (\V, \u) := 
 \notag\\
 & \left[\cov\left(\Z_1 | (\bar{\V}_{\TTheta }^t)_1 = \V_i\right)\right]^{-1} \left(\E[\Z_1 | (\bar{\V}_{\TTheta}^t)_1  = \V_i, q(\Z_1, \bar{\Psi}_i, \bar{\varepsilon}) = u_i]  
 - \E[\Z_1 | (\bar{\V}_{\TTheta}^t)_1 = \V_i] \right),\label{eq:opt_ensemble_gt_lem}
\end{align}
for $\V \in \reals^{n \times L}, \u \in \reals^n, i \in [n]$. 
\end{enumerate}
\end{proposition}
The proof of Proposition \ref{prop:opt_ensemble_ft_gt}, given in Appendix \ref{sec:optimal_ft_gt_proof}, is similar to the derivation of the optimal denoisers for mixed regression in \cite{tan_mixed_2023}, with a few key differences in the derivation of $g^{*t}$, which is not separable  in the change point setting. With a product distribution on $\pi_{\bar{\bPsi}}$, we recover  mixed regression and  \eqref{eq:opt_ensemble_ft_lem}--\eqref{eq:opt_ensemble_gt_lem} reduce to the optimal denoisers in \cite{tan_mixed_2023}

The denoiser $f^{*t}$ is separable and  can be easily computed for sufficiently regular distributions $\P_{\bar{\B}}$ such as discrete,  Gaussian, or Bernoulli-Gaussian  distributions (see Appendix C.2 of \cite{arpino2024inferring}). In Appendix \ref{sec:full_computation_gt} we show how $g^{*t}$ can be efficiently computed for the linear and rectified linear models with additive Gaussian noise, as well as for the logistic model. As detailed in Appendix \ref{sec:computational_cost}, $f^{*t}$ and $g^{*t}$ can be computed in $O(nL^3)$ time, yielding a total computational complexity of $O(npL^3)$ for AMP with these denoisers.

\subsection{Change Point Estimation and Inference} \label{sec:chgpt_estimation}
We now show how the AMP algorithm can be used for estimation and inference of the change points $\{\eta_1, \dots, \eta_{L^* - 1}\}$.  We first define some notation. Let $\mathcal{X} \subset [L]^n$ be the set of all piece-wise constant vectors with respect to $i \in [n]$ with at most $(L - 1)$ jumps. This set includes  all  possible instances of the signal configuration vector $\bPsi$ in $[L]^n$, and without loss of generality can account for duplicate signals by increasing $L$.  Let the function $U: \eeta \mapsto \bPsi$ denote the one-to-one mapping between change point vectors $\eeta$ and signal configuration vectors $\bPsi$. For a vector $\hat{\eeta}$, we let $|\hat{\eeta}|$ denote its dimension (number of elements). 

\paragraph{Change Point Estimation} Theorem \ref{thm:SE} states that $(\TTheta^t, \y)$ converges in a specific sense to the random variable $(\V_{\TTheta}^t, q(\Z, \bPsi, \bvarepsilon))$, whose distribution crucially captures information about $\bPsi$. Hence, it is natural to consider point estimators for $\eeta$ of the form $\hat{\eeta}(\TTheta^t, \y)$. In the linear model, for instance, one example is an estimator that searches for a signal configuration vector $\bpsi \in \mathcal{X}$ such that $\TTheta^t$ indexed along $\bpsi$ is closest in $\ell_2$ distance to the observed vector $\y$. That is,
\begin{align}
\hat{\bPsi}(\TTheta^t, \y)  &= \argmin_{\bpsi \in \mathcal{X}} \,  \sum_{i = 1}^n (y_i - (\TTheta^t)_{i, {\psi}_i})^2 \, , \label{eq:example_estimator}
\end{align}
 and $\hat{\eeta}(\TTheta^t, \y) = U^{-1}(\hat{\bPsi}(\TTheta^t, \y))$. In \eqref{eq:argmax_approx_posterior} we propose a general form for $\hat{\bPsi}$ that is suitable for a generic nonlinear model. A common metric for evaluating the accuracy of change point estimators   is the Hausdorff distance \citep{wang_high_2018, xu_change_2022, li_divide_2023}. The Hausdorff distance between two non-empty subsets $X,Y$  of $\reals$ is 
\[d_{H}(X, Y) = \max\Big\{ \sup_{x \in X} d(x, Y) , \,  \sup_{y \in Y} d(X, y) \Big\},  \]
where $d(x, Y) := \min_{y \in Y} \|x - y\|_2$. The Hausdorff distance is a metric, and can be viewed as the largest of all distances from a point in $X$ to its closest point in $Y$ and vice versa.
We interpret the Hausdorff distance between $\eeta$ and an estimate $\hat{\eeta}$ as the Hausdorff distance between the sets formed by their elements. The following theorem states that any well-behaved estimator $\hat{\eeta}$ produced using the AMP iterate  $\TTheta^t$ admits a precise asymptotic characterization in terms of Hausdorff distance and size. 
\begin{proposition} \label{prop:hausdorff_asymptotics}
Consider the AMP in \eqref{eq:amp}. Suppose the model assumptions in Section \ref{sec:prelim} as well as $\ref{it:main-res-ass-1}-\ref{it:main-res-ass-3}$ are satisfied. Let $\hat{\eeta}(\TTheta^t, \y) = \hat{\eeta}(\TTheta^t, q(\X\B, \bPsi, \bvarepsilon))$ be an estimator such that $(\V, \z) \mapsto U(\hat{\eeta}(\V, q(\z, \bPsi, \bvarepsilon))$ is uniformly pseudo-Lipschitz. Then:
\begin{align}
    \frac{d_{H}(\eeta, \hat{\eeta}(\TTheta^t, \y))}{n} \stackrel{\P}\simeq \E_{\V_{\TTheta}^t, \Z} \frac{d_{H}(\eeta, \hat{\eeta}(\V_{\TTheta}^t, q(\Z, U(\eeta), \bvarepsilon)))}{n}.\label{eq:hausdorff_convergence}
\end{align}
Moreover, if $(\V, \z) \mapsto |\hat{\eeta}(\V, q(\z, \bPsi, \bvarepsilon))|$ is uniformly pseudo-Lipschitz, then: 
\begin{align}
    &|\hat{\eeta}(\TTheta^t, \y)| \stackrel{\P}\simeq \E_{\V_{\TTheta}^t, \Z} |\hat{\eeta}(\V_{\TTheta}^t, q(\Z, U(\eeta), \bvarepsilon))| .\label{eq:size_convergence}
\end{align} 
\end{proposition}
The proof is given in Appendix \ref{sec:proof_hausdorff}. To prove  \eqref{eq:hausdorff_convergence}, we  show that ${d_{H}(\eeta, \hat{\eeta}(\TTheta^t, \y))}/{n}$ is uniformly pseudo-Lipschitz. For each $\eeta$, Proposition \ref{prop:hausdorff_asymptotics} precisely characterizes the asymptotic Hausdorff distance and size errors for a large class of estimators $\hat{\eeta}(\TTheta^t, \y)$. 
\paragraph{Uncertainty Quantification} 
  The random variable ${\V}_{\TTheta}^t = \Z \brho^{-1} {\bnu}_{\TTheta}^t + {\G}^t_{\TTheta}$ in \eqref{eq:V_TTheta}, combined with an observation of the form $q(\Z, \bPsi, \bar{\bvarepsilon})$, yields a recipe for constructing a posterior distribution over $\bPsi$. Recalling that $\bPsi$ is the unknown ground-truth signal configuration, we let $\bar{\bPsi}$ denote a random signal configuration drawn according to density $\pi_{\bar{\bPsi}}$, and $\bpsi$  a realization of $\bar{\bPsi}$. Using the prior  
$\pi_{\bar{\bPsi}}$, the posterior is: 
\begin{align}
p_{\bar{\bPsi} | {\V}_{\TTheta}^t, q(\Z, \bar{\bPsi}, \bar{\bvarepsilon})}(\bpsi | \V, \u) &= \frac{  \pi_{\bar{\bPsi}}(\bpsi) \cL(\V, \u | \bpsi)}{\sum_{\tilde{\bpsi}} \pi_{\bar{\bPsi}}(\tilde{\bpsi}) \cL(\V, \u | \tilde{\bpsi})}, \label{eq:posterior_density_2}
\end{align}
where $\V \in \reals^{n \times L}$, $\u \in \reals^n$. Here, $\cL$ is the likelihood of $({\V}_{\TTheta}^t, q(\Z, \bar{\bPsi}, \bar{\bvarepsilon}))$ given $\bar{\bPsi} = \bpsi \in \mathcal{X}$, where the state evolution parameters $\bnu_{\TTheta}^t, \bkappa_{\TTheta}^t$ associated to $\V_{\TTheta}^t$ are computed as in \eqref{eq:nu_B_SE}--\eqref{eq:kappa_theta_SE} with $\bPsi$ replaced by $\bpsi$. The likelihood $\cL$ can be computed in closed form for the linear and rectified linear models under the assumption of additive and independent Gaussian noise; for the logistic model, the likelihood can be well approximated. The computations are given in Appendix \ref{sec:likelihood_computation}.

Since Theorem \ref{thm:SE} states that $(\TTheta^t, \y)$ converges in a specific sense to $(\V_{\TTheta}^t, q(\Z, \bPsi, \bvarepsilon))$, we can obtain an uncertainty estimate over $\bPsi$ by plugging in $(\TTheta^t, \y)$ for $(\V, \u)$ in \eqref{eq:posterior_density_2}. It then follows from our theory that this uncertainty estimate converges point-wise in probability to a faithful posterior distribution over the change points in the high-dimensional limit. 
\begin{proposition}\label{prop:pointwise_posterior}
Let $(\TTheta^t, \y)$ be as in \eqref{eq:amp}, $(\V_{\TTheta}^t, q(\Z, \bPsi, \bvarepsilon))$ be as in \eqref{eq:V_TTheta}, and $\bpsi \in \mathcal{X}$. Suppose the model assumptions in Section 2 as well as $(A1)$--$(A3)$ are satisfied. Assume that $(\V, \z) \mapsto p_{\bar{\bPsi} | {\V}_{\TTheta}^t, q(\Z, \bar{\bPsi}, \bar{\bvarepsilon})}(\bpsi | \V, q(\z, \bPsi, \bvarepsilon)) =: p(\bpsi | \V, q(\z, \bPsi, \bvarepsilon))$ is uniformly pseudo-Lipschitz. Then:   
\begin{align}
   p(\bpsi | \TTheta^t, \y) \stackrel{\P}{\simeq}  p(\bpsi | {\V}_{\TTheta}^t, q(\Z, \bPsi, \bvarepsilon)).\label{eq:posterior_convergence}
\end{align}
\end{proposition}
Given $\bpsi$ and ground-truth variables $\bPsi, \bvarepsilon$, the right-hand side of \eqref{eq:posterior_convergence} can be computed by sampling $({\V}_{\TTheta}^t, \Z)$, where the state evolution parameters $\bnu_{\TTheta}^t, \bkappa_{\TTheta}^t$ associated to $\V_{\TTheta}^t$ are computed exactly as in \eqref{eq:nu_B_SE}--\eqref{eq:kappa_theta_SE}, and plugging these into $p(\bpsi | \cdot , q( \cdot , \bPsi, \bvarepsilon))$. The proof, given in Appendix \ref{sec:proof_posterior}, is an application of Theorem \ref{thm:SE}. 
The state evolution predictions on the RHS of \eqref{eq:hausdorff_convergence} and \eqref{eq:posterior_convergence} can be computed under reasonable assumptions, as outlined in Appendix \ref{app:SE_limits}.

\section{Experiments} \label{sec:experiments}
In this section, we demonstrate the estimation and inference capabilities of the AMP algorithm in a range of settings, with both synthetic and real data. For synthetic data, for $i \in [n]$, we use i.i.d. Gaussian covariates $\X_i \distas{i.i.d} \N(0, \I_p/n)$, and we let $\bvarepsilon_i \distas{i.i.d} \P_{\bar{\varepsilon}} = \N(0, \sigma^2)$ in the linear and rectified linear models. The denoisers $\{g^{t}, f^{t + 1}\}_{t \geq 0}$ in the AMP algorithm  are chosen according to Proposition \ref{prop:opt_ensemble_ft_gt}. The computation of $f^t$ is standard  (outlined in Appendix C.2 of \cite{arpino2024inferring}) and the computation of $g^t$ is described in Appendix \ref{sec:full_computation_gt}.  The Jacobians of these denoisers are computed using Automatic Differentiation in Python JAX \citep{jax2018github}.  In all the experiments, we use  a uniform prior $\pi_{\bar{\Psi}}$ over all  configurations with change points at least $\Delta$ apart, for some $\Delta$ that is a fraction of $n$. Error bars represent one standard deviation, and all experiments are the result of at least $10$ independent trials with $t \leq 15$. Posterior densities are smoothed with a Gaussian kernel of lengthscale $0.5$ for interpretability.

For a given point estimator $\hat{\bPsi}(\TTheta^t, \y)$ of $\bPsi$, the iterate $\TTheta^t$ is obtained by running the AMP algorithm in \eqref{eq:amp} and we let $\hat{\eeta}(\TTheta^t, \y) := U^{-1}(\hat{\bPsi}(\TTheta^t, \y))$. As discussed in Section \ref{sec:chgpt_estimation}, the point estimation guarantees of Proposition \ref{prop:hausdorff_asymptotics} hold for a large class of sufficiently regular point estimators of $\bPsi$. For linear models we could use the point estimator in \eqref{eq:example_estimator}, and for nonlinear models,  the mean of the posterior distribution $p_{\bar{\bPsi} | \V_{\TTheta}^t, q(\Z, \bar{\bPsi}, \bar{\bvarepsilon})}(\bpsi | \TTheta^t, \y)$ defined in \eqref{eq:posterior_density_2}. The latter, although justified by the limiting posterior distribution in Proposition \ref{prop:pointwise_posterior}, involves tabulating the state evolution parameters $(\bnu_{\TTheta}^t, \bkappa_{\TTheta}^t)$ from \eqref{eq:nu_B_SE}--\eqref{eq:kappa_theta_SE} (with $\bPsi$ replaced by $\bpsi$), since these are required to compute the likelihood $\cL$ of $(\V_{\TTheta}^t, q(\Z, \bar{\bPsi}, \bar{\bvarepsilon}))$ given $\bar{\bPsi} = \bpsi$. In the following experiments, we  use an  alternative point estimate of $\bPsi$ that is computationally faster, by considering instead the likelihood $\bar{\cL}$ of $(\bar{\V}_{\TTheta}^t, q(\Z, \bar{\bPsi}, \bar{\bvarepsilon}))$ given $\bar{\bPsi} = \bpsi$. We recall that $\bar{\V}_{\TTheta}^t$ is defined in \eqref{eq:V_bar_TTheta} through the $\bpsi$-independent parameters $(\bar{\bnu}_{\TTheta}^t, \bar{\bkappa}_{\TTheta}^t)$. The point estimator of interest is then defined as follows: 
\begin{align} \label{eq:argmax_approx_posterior}
   \hat{\bPsi}(\TTheta^t, \y) = \argmin_{\bpsi \in \mathcal{X}} \, p_{\bar{\bPsi} | \bar{\V}_{\TTheta}^t, q(\Z, \bar{\bPsi}, \bar{\bvarepsilon})}(\bpsi | \TTheta^t, \y), 
\end{align}
where we call $p_{\bar{\bPsi} | \bar{\V}_{\TTheta}^t, q(\Z, \bar{\bPsi}, \bar{\bvarepsilon})}(\bpsi | \V, \u) := \frac{  \pi_{\bar{\bPsi}}(\bpsi) \bar{\cL}(\V, \u | \bpsi)}{\sum_{\tilde{\bpsi}} \pi_{\bar{\bPsi}}(\tilde{\bpsi}) \bar{\cL}(\V, \u | \tilde{\bpsi})}$ the \textit{approximate posterior}.  The explicit form of the likelihood $\bar{\cL}$ for each of the linear, logistic, and rectified linear models is presented in Appendix \ref{sec:likelihood_computation}. Full implementation details are provided in Appendix \ref{sec:further_implementation}. A Python implementation of our algorithm and code to run the experiments is available at \citep{arpino_impl_gen_24}.

\subsection{Linear Model with Change Points}

Figure \ref{fig:linear_estimation_size_SE} (left) plots the Hausdorff distance normalized by $n$  for varying $\delta$, for two different change point configurations $\bPsi$. We choose $p = 600$, $\P_{\bar{\B}} = \N(\0, \I)$, $\sigma = 0.1$, $\Delta = n/5$ and fix  two true change points, whose locations are indicated in the legend. The algorithm uses $L = L^*= 3$.
The state evolution prediction closely matches the performance of AMP, verifying \eqref{eq:hausdorff_convergence} in  Proposition \ref{prop:hausdorff_asymptotics}. 

The two right-most plots in Figure \ref{fig:linear_estimation_size_SE} display the approximate posterior over the \emph{number} of change points, that is, $\sum_{\bpsi\in \mathcal{S}}p_{\bar{\bPsi} | \bar{\V}_{\TTheta}^t, q(\Z, \bar{\bPsi}, \bar{\bvarepsilon})}(\bpsi |\TTheta^t, \y)$ where $\mathcal{S}$ contains all configurations with a specified number of change points. We  use $p=200, \P_{\bar{\B}}=N(\0,\I), \sigma=0.1, \Delta=n/10, L^*=3, w=4$, and a uniform prior over the number of change points (zero to three). We observe that the approximate posterior concentrates around the ground truth for moderately large $\delta$. 
\begin{figure}[t!]
\centering
\subfloat{%
  \centering
  \includegraphics[width=0.415\linewidth]{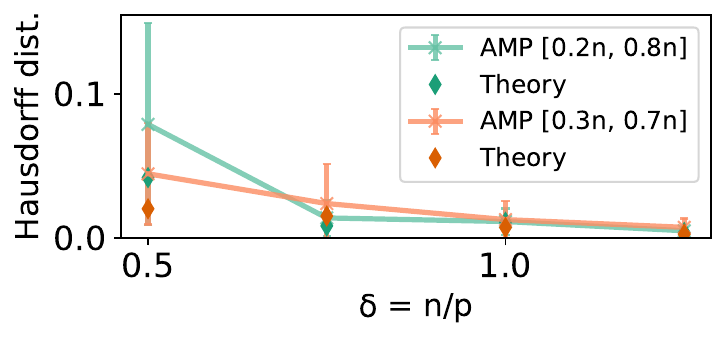}
}
\subfloat{%
  \centering
    \raisebox{4mm}{\includegraphics[height=0.168\linewidth]{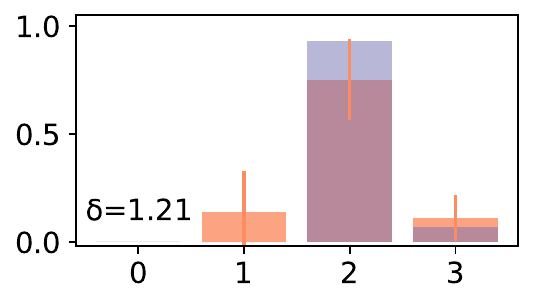}
    \includegraphics[height=0.168\linewidth]{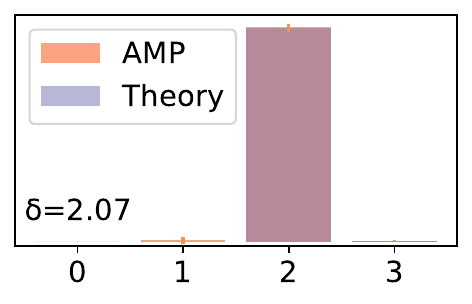}}
}
\caption{Performance of AMP for estimating change points in  the linear model under two ground-truth  configurations (left), approximate posterior density over the number of change points  (two plots on the right).}
\label{fig:linear_estimation_size_SE}
\end{figure}

Figure \ref{fig:exact_posterior} (left) plots the estimated posterior $p_{\bar{\bPsi} | \V_{\TTheta}^t, q(\Z, \bar{\bPsi}, \bar{\bvarepsilon})}(\cdot | \TTheta^t, \y)$ labelled `AMP', and the limiting \textit{exact} posterior $p_{\bar{\bPsi} | \V_{\TTheta}^t, q(\Z, \bar{\bPsi}, \bar{\bvarepsilon})}(\cdot | \V_{\TTheta}^t, q(\Z, \bPsi, {\bvarepsilon}))$ labelled `Theory', averaged over $30$ trials with true change point locations at  $n/3$ and $ 8n/15$. The posterior is computed over the grid $[0, 1]$ subsampled by a factor of ten. The experiment uses $p = 800, \P_{\bar{\B}} = \N(\0, \I), \sigma = 0.1, \Delta = n/5, L = L^* = 3$, with a uniform prior over all valid configurations. We observe that the estimated posterior closely matches the limiting posterior and concentrates around the ground truth as $\delta$ increases. 

\begin{figure}[t!]
\centering
\subfloat{%
  \centering
  \includegraphics[width=0.32\linewidth]{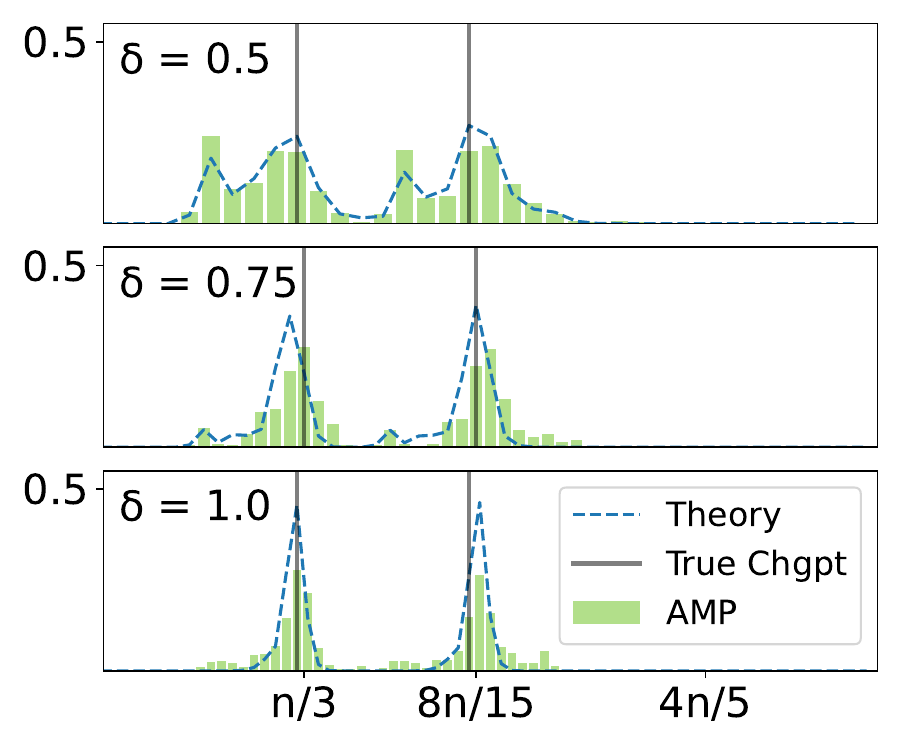}
}
\subfloat{%
  \centering
  \includegraphics[width=0.32\linewidth]{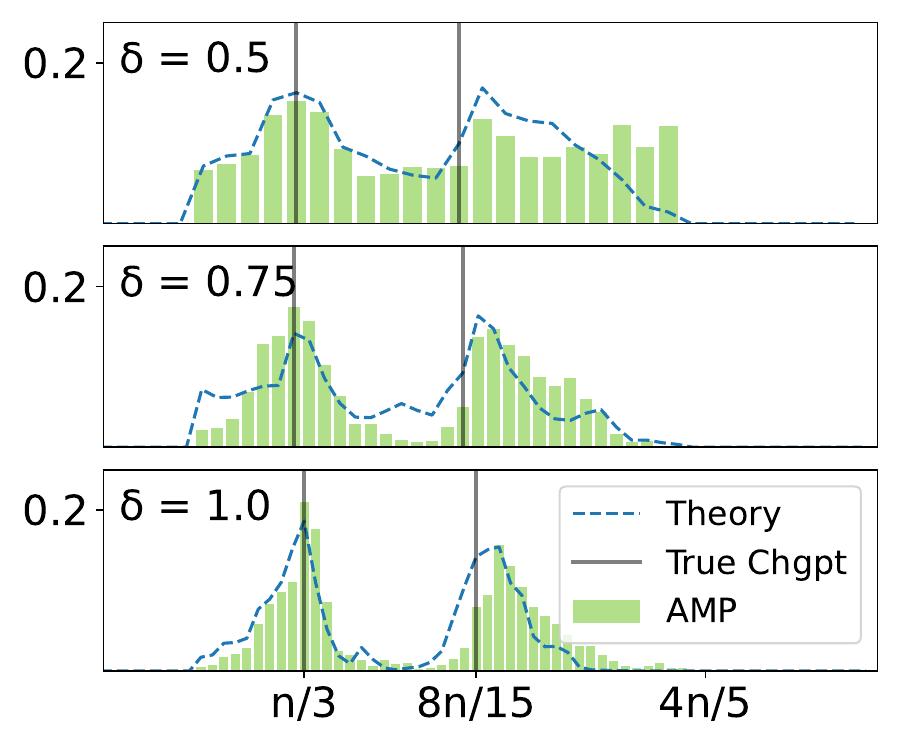}
}
\subfloat{%
  \centering
  \includegraphics[width = 0.32\linewidth]{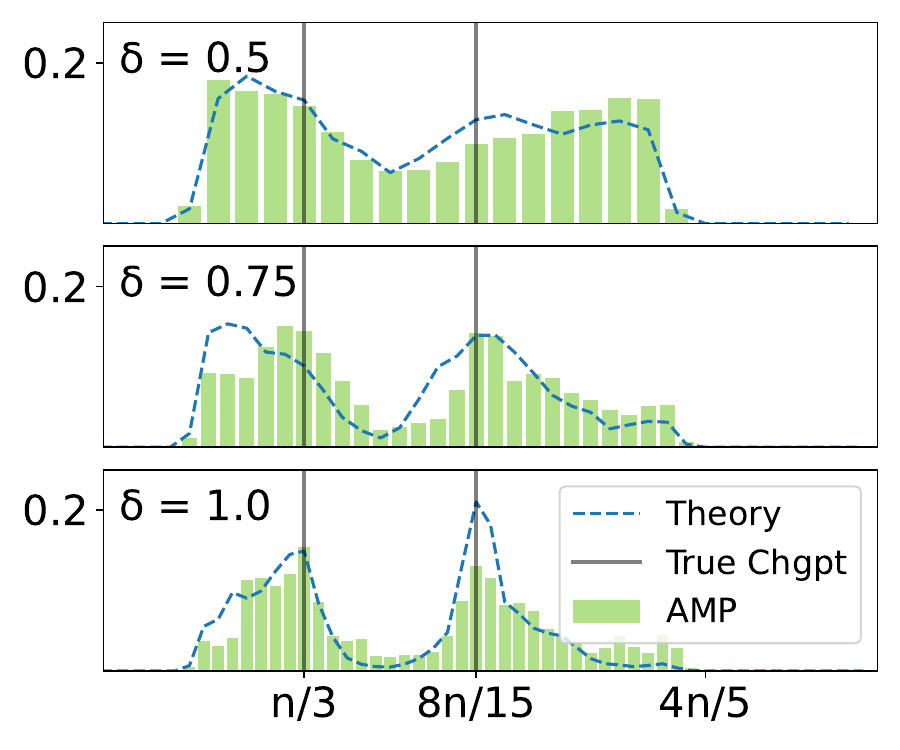}
}
\caption{Estimated posterior  $p(\cdot | \TTheta^t, \y)$ plotted against the limiting exact posterior $p(\cdot | {\V}_{\TTheta}^t, q(\Z, \bPsi, \bvarepsilon))$ as per \eqref{eq:posterior_convergence} for the linear (left), logistic (middle), rectified linear (right) models.}
\label{fig:exact_posterior}
\end{figure}
Figures \ref{fig:DPDU_L3} and \ref{fig:charcoal_L3} compare the  performance of AMP against four state-of-the-art algorithms: the dynamic programming ($\DP$) approach  in \citep{rinaldo_localizing_2021};  dynamic programming with dynamic updates $(\DPDU)$ \citep{xu_change_2022};  divide and conquer dynamic programming ($\DCDP$)  \citep{li_divide_2023}; 
and the complementary-sketching-based algorithm $\charcoal$ from \cite{gao_sparse_2022}.
Hyperparameters are chosen using cross validation (CV), as outlined in Section A.1 of \cite{li_divide_2023}.
 The first three algorithms, designed for sparse signals, combine LASSO-type estimators  with partitioning techniques based on dynamic programming. The $\charcoal$ algorithm is designed for the setting where the difference  $\bbeta^{(\eta_\ell)}- \bbeta^{(\eta_{\ell+1})}$ between adjacent signals 
is sparse.
None of these algorithms uses a prior on the change point locations, unlike AMP which can flexibly incorporate both  priors via $\P_{\bar{\B}}$  and $\pi_{\bar{\bPsi}}$. 
Figure \ref{fig:DPDU_L3} uses $p=200, \sigma=0.1,  \Delta=n/10, L^*=L=3$  and a sparse Bernoulli-Gaussian signal prior $\P_{\bar{\B}} =0.5 \N(\0, \delta\I) + 0.5 \delta_\0$. AMP assumes no knowledge of the true sparsity level 0.5 and estimates the sparsity level using CV over a set of values not including the ground truth (details in Appendix \ref{sec:further_implementation}). Figure \ref{fig:charcoal_L3}
 uses $p=300, \Delta=n/10, L^*=L=3$ and a Gaussian sparse difference prior with sparsity level $0.5$ (described in \eqref{eq:sparse_diff_prior}). AMP is run assuming a mismatched sparsity level of 0.9 and a mismatched magnitude for the sparse difference vector (details in Appendix \ref{sec:further_implementation}).
Figures \ref{fig:DPDU_L3} and \ref{fig:charcoal_L3} show that AMP  consistently achieves the lowest  Hausdorff distance among all algorithms and outperforms most algorithms in  runtime.  Figures  \ref{fig:extra_comparison_DCDP}  and \ref{fig:vary_p} in Appendix \ref{sec:further_implementation} show results from an additional set of experiments 
comparing AMP with $\DCDP$ (the fastest algorithm in Figure \ref{fig:DPDU_L3}).  Figure \ref{fig:extra_comparison_DCDP} shows the performance of AMP with different change point priors $\pi_{\bar{\bPsi}}$ and suboptimal denoising functions such as soft thresholding. Figure \ref{fig:vary_p} demonstrates the favourable runtime scaling of AMP over $\DCDP$ with respect to $p$, due to the LASSO computations involved in $\DCDP$. 
\begin{figure}
\centering
\subfloat[Comparison with $\DPDU$, $\DCDP$ and $\DP$ for sparse prior $\P_{\bar{\B}} =  0.5\N(\0, \delta\I) +0.5 \delta_{\0}$. $L^*=L=3$. Runtime shown is the average runtime per set of CV parameters.]{%
    \centering
    \includegraphics[width=0.475\linewidth]{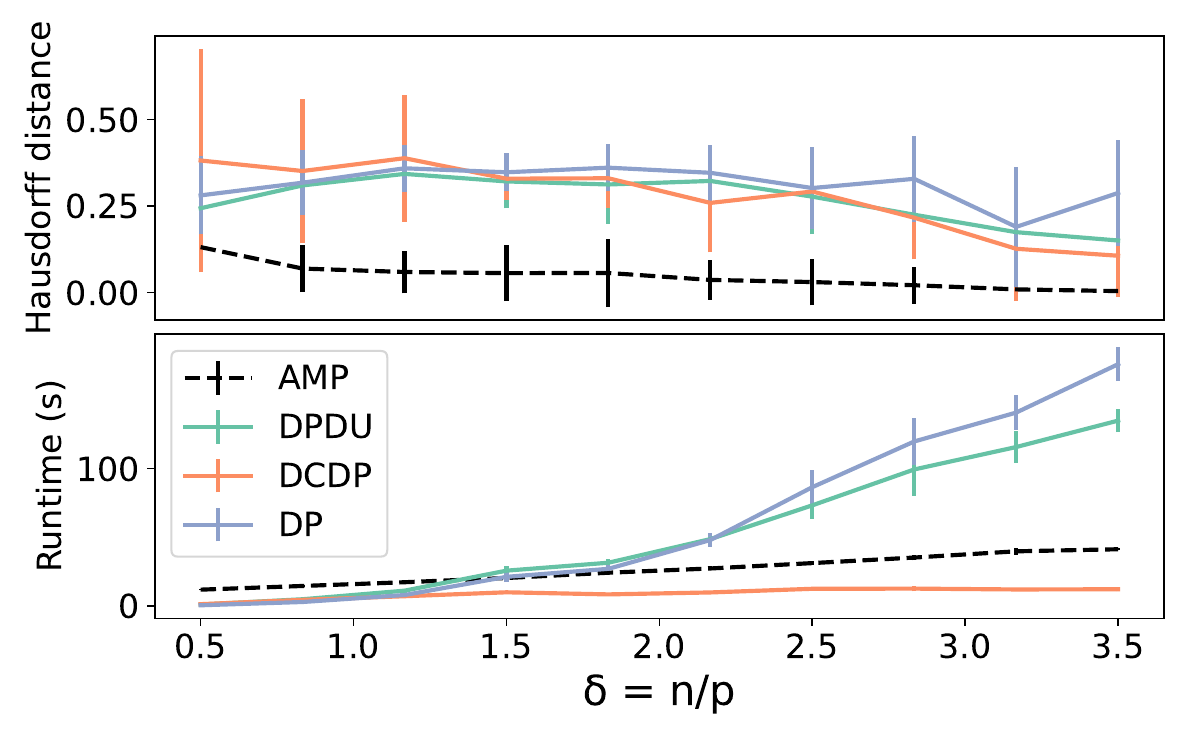}
    \label{fig:DPDU_L3}
}
\hspace{0.3cm}
\subfloat[Comparison with $\charcoal$ and $\McScan$ for a sparse difference prior with sparsity level 0.5. $L^*=L=3$.]{%
    \centering
    \includegraphics[width=0.475\textwidth]{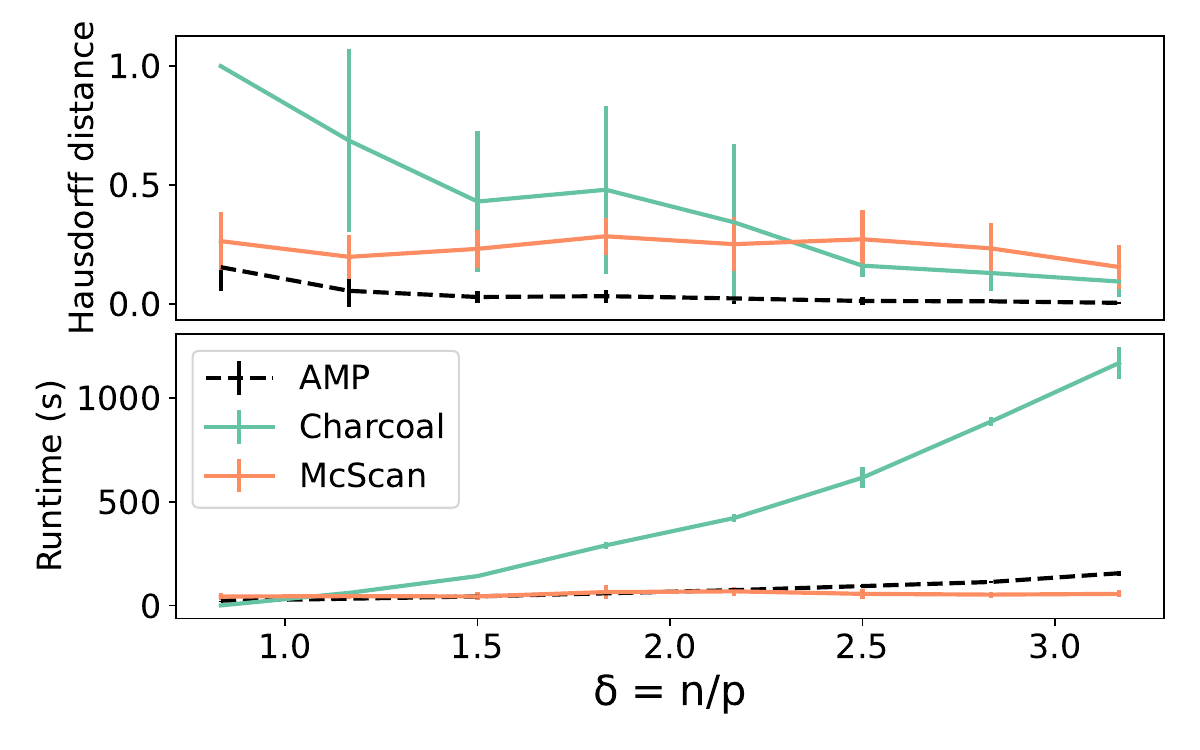}
\label{fig:charcoal_L3}
}
\caption{Linear model with sparse or sparse difference signal prior.}
\label{fig:sparse_sparse_diff}
\end{figure}
\paragraph{Compressed Sensing with Change Points}
In Figure \ref{fig:satellite}, we consider noiseless compressed sensing, where the signals $\{\bbeta^{(i)}\}_{i \in [n]}$ are rotated versions of a $(255, 255)$ sparse grayscale image used by \citet{Sch14}. The fraction of nonzero components in the image is  8645/50625. We downsample the original image by a factor of three and flatten, yielding an operation dimension of $p = 85^2 = 7225$.  We set $\{\bbeta^{(i)}\}_{i = 1}^{0.3n}$ to be the image, $\{\bbeta^{(i)}\}_{i = 0.3n}^{0.7n}$  to be a $30^{\circ}$ rotated version, and $\{\bbeta^{(i)}\}_{i = 0.7n}^{n}$ to be a $45^{\circ}$  rotated version. We run AMP with a Bernoulli-Gaussian prior $\P_{\bar{\B}}$,  $\Delta = n/4$ and $L = L^* = 3$. Figure \ref{fig:satellite} shows image reconstructions along with the approximate posterior $p_{\bar{\bPsi} | \bar{\V}_{\TTheta}^t, q(\Z, \bar{\bPsi}, \bar{\bvarepsilon})}(\cdot |\TTheta^t, \y)$. The approximate posterior concentrates around the true change point locations as $\delta$ increases, even when the image reconstructions are approximate. The experiment took one hour to complete on an Apple M1 Max chip, whereas competing algorithms did not return an output within 2.5 hours, due to the larger signal dimension compared to Figure \ref{fig:DPDU_L3} ($p=7225$ vs $p=200$).
\begin{figure}[t!]
\centering
\includegraphics[height=0.328\linewidth]{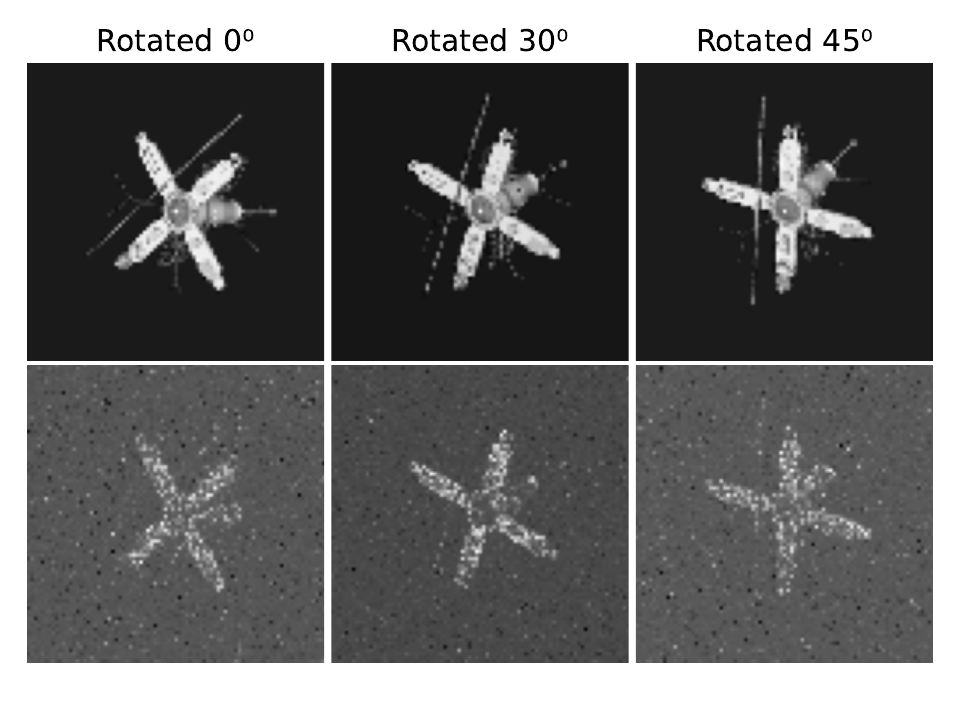}
\vspace{-0.3cm}
\includegraphics[height=0.328\linewidth]{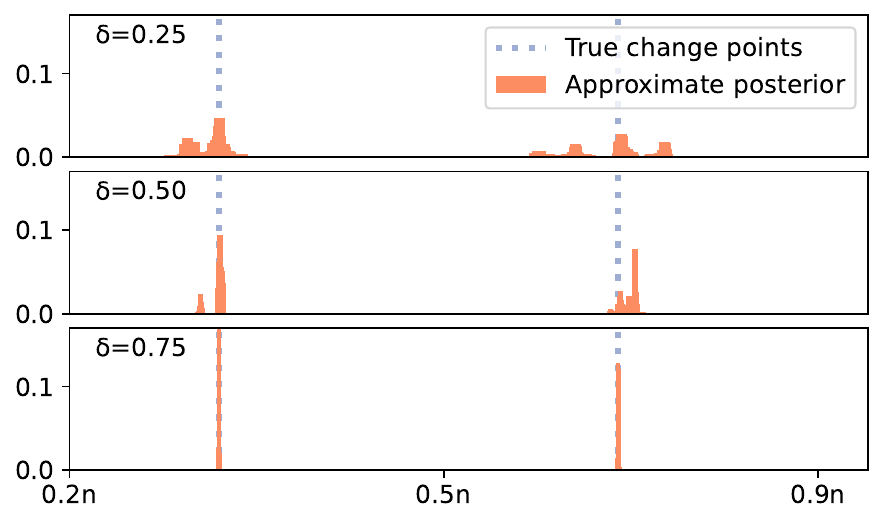}
\caption{Top: Ground truth images (first row) and reconstruction from AMP (second row) at $\delta = 0.75$. Bottom: Approximate posterior vs. fractional change point locations for different $\delta$.} 
\label{fig:satellite}
\vspace{-0.5cm}
\end{figure}

\subsection{Rectified Linear Regression with Change Points}
For the rectified linear model in \eqref{eq:ReLU-chgpt-model} with a single change point at $(0.6n)$, Figure \ref{fig:logistic_relu_synthetic_estimation} (right) plots the Hausdorff distance divided by $n$ as a function of $\delta$. We let $p = 600$, $L = 3$, $\P_{\bar{\B}} = \N(\0, \delta \I)$, $\sigma = 0.3$, and average over ten trials. Similarly, Table \ref{tab:synthetic_num_chgpts} displays the performance of AMP in estimating the number of change points in the same setup, where we vary $\sigma$ from $0.5$ to $0.3$. We observe a close match between the AMP performance and that dictated by theory in Proposition \ref{prop:hausdorff_asymptotics}. Figure \ref{fig:exact_posterior} (right) plots the estimated posterior $p_{\bar{\bPsi} | \V_{\TTheta}^t, q(\Z, \bar{\bPsi}, \bar{\bvarepsilon})}(\cdot | \TTheta^t, \y)$  against its the limiting  posterior $p_{\bar{\bPsi} | \V_{\TTheta}^t, q(\Z, \bar{\bPsi}, \bar{\bvarepsilon})}(\cdot | {\V}_{\TTheta}^t, q(\Z, \bPsi, \bvarepsilon))$, for the rectified linear model with $p = 600, \sigma = 0.5$, true change point locations at $n/3$ and $8n/15$, and $L = 3$ with a uniform prior over all configurations with two change points. We empirically observe a match between the estimated and limiting posteriors as per Proposition \ref{prop:pointwise_posterior}. 

\subsection{Logistic Regression with Change Points}
For the logistic model in \eqref{eq:logistic-chgpt-model} with change points  at $0.35n$ and  $0.7n$, Figure \ref{fig:logistic_relu_synthetic_estimation} (left) plots the Hausdorff distance normalized by $n$ for varying $\delta$. We choose $p = 600$, $\Delta = n/5$, and $L = 3$ with a uniform prior over all valid configurations. We let $\P_{\bar{\B}} = \N\left(\0, \begin{bmatrix}
    15 & 0.2 & 0\\
    0.2 & 15 & 0\\
    0 & 0 & 15 
\end{bmatrix}\right)$. We observe a close match between the AMP performance and that dictated by the state evolution theory in \eqref{eq:hausdorff_convergence}. 

Table \ref{tab:synthetic_num_chgpts} displays AMP's estimated number of change points $|\hat{\eeta}(\TTheta^t, \y)|$ in a variety of different settings. For the logistic model we set $\brho$ to be $50 \I$ and $75 \I$, achieved by respectively setting $\P_{\bar{\B}} = \N(\0, 50 \cdot \delta \I)$ and $\P_{\bar{\B}} = \N(\0, 75 \cdot \delta \I)$ in our experiment. We let $p = 600$, $L = 3$, and let the true (single) change point location be at $0.6n$ ($L^* = 2$). We observe that the average number of change points predicted by AMP is close to the truth value of $1$, and closely matches the theory from \eqref{eq:size_convergence}. 

Figure \ref{fig:exact_posterior} (middle) shows the estimated posterior $p(\cdot | \TTheta^t, \y)$ from \eqref{eq:posterior_convergence} and the limiting (exact) posterior $p(\cdot | {\V}_{\TTheta}^t, q(\Z, \bPsi, \bvarepsilon))$ in the logistic model with two true change points at $n/3$ and $8n/15$. We let $p = 600$, $\P_{\bar{\B}} = \N(\0, 20 \I)$, and subsample the domain $[0, 1]$ by a factor of eight. We let $L = 3$ and place uniform prior over all possible two-changepoint configurations. We observe a close match with the theory in Proposition \ref{prop:pointwise_posterior}, and concentration around the true change point locations as $\delta$ increases.

Figure \ref{fig:logistic_compare} compares AMP with the binary segmentation  algorithm proposed in \cite{wang_efficient_2023} for sparse logistic regression. We use $p=300,  \Delta=0.15n$,  and a sparse Bernoulli-Gaussian signal prior $\P_{\bar{\B}} =0.5 \N(\0, 15\delta\I) + 0.5 \delta_\0$. The true change points are at $n/3$ and $8n/15$ with $L^*=3$. The hyperparameters for binary segmentation are chosen as suggested in the source code of  \cite{wang_efficient_2023}. Both algorithms can take advantage of  prior information on the number of change points. 
In Fig.\@ \ref{fig:logistic_L3or4}, both  algorithms assume a maximum of $L=3$ or 4 signals while  $L^*=3$. The algorithms are observed to  achieve  similar Hausdorff distance, with AMP enjoying a much lower runtime  when assuming $L=3$. Fig.\@ \ref{fig:logistic_exactly_L3} compares the two algorithms by assuming there are \emph{exactly} 3 signals. The performance of AMP improves significantly with  larger $\delta$, both in terms of Hausdorff distance and runtime. We recall that the algorithm in \cite{wang_efficient_2023} is tailored to sparse logistic regression, whereas AMP can account for a broader range of signal priors. 
%

\begin{table}[t!]
    \centering
    \resizebox{\columnwidth}{!}{\begin{tabular}{@{}lcccccccccccl@{}}\toprule
    & \multicolumn{3}{c}{Logistic $\brho = 50 \I$} & \multicolumn{3}{c}{Logistic $\brho = 75 \I$} & \multicolumn{3}{c}{ReLU $\sigma = 0.50$} & \multicolumn{3}{c}{ReLU $\sigma = 0.30$}
    \\\cmidrule(lr){2-4}\cmidrule(lr){5-7}\cmidrule(lr){8-10}\cmidrule(lr){11-13}
               & Theory  & Mean & Std    & Theory  & Mean & Std & Theory  & Mean & Std & Theory  & Mean & Std\\\midrule
    $\delta = 0.75$           &  & -- &  &  & -- & & 1.10 & 1.10 & 0.70 & 1.00 & 1.10 & 0.30 \\
    $\delta = 1.25$ & 1.00 & 1.10 & 0.30 & 1.00 & 0.90 & 0.30 & 0.80 & 1.30 & 0.46 & 1.00 & 1.00 & 0.00\\
    $\delta = 1.75$ & 1.00 & 1.30 & 0.46 & 1.00 & 1.20 & 0.40 & 1.00 & 1.20 & 0.60 & 1.00 & 1.00 & 0.00\\
    $\delta = 2.25$   & 1.00 & 1.10 & 0.30 & 1.00 & 1.10 & 0.30 &  & -- & &  & -- & \\\bottomrule
    \end{tabular}}
    \caption{AMP performance when estimating the number of change points in the logistic and rectified linear (ReLU) models. The true (single) change point location is $0.6n$ ($L^* = 2$), AMP is run with $L = 3$.}
    \label{tab:synthetic_num_chgpts}
\end{table}

\begin{figure}[t!]
\centering
\subfloat{%
  \centering
  \includegraphics[width=0.49\linewidth]{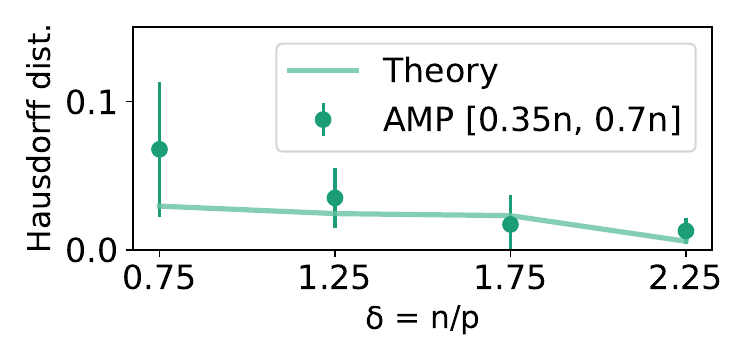}
}
\subfloat{%
  \centering
  \includegraphics[width=0.49\linewidth]{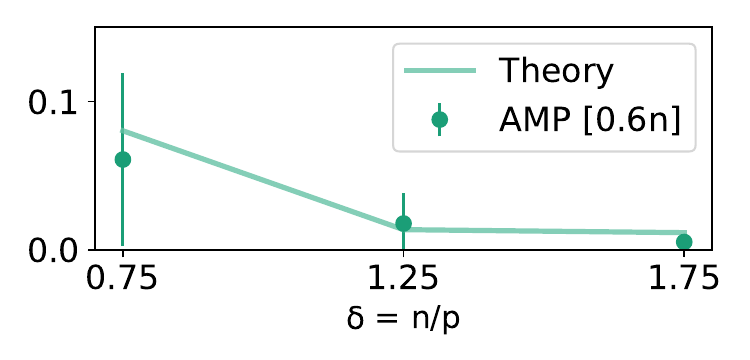}
}
\caption{AMP performance in terms of the normalized Hausdorff distance plotted against the theoretical predictions of Proposition \ref{prop:hausdorff_asymptotics} for the logistic (left) and rectified linear (right) models.}
\label{fig:logistic_relu_synthetic_estimation}
\end{figure}

\begin{figure}[t!]
\centering
\subfloat[Algorithms  assume a maximum of $3$ or $4$ signals (that is, $L=3$ or 4).]{%
\centering
\includegraphics[width=0.475\textwidth]{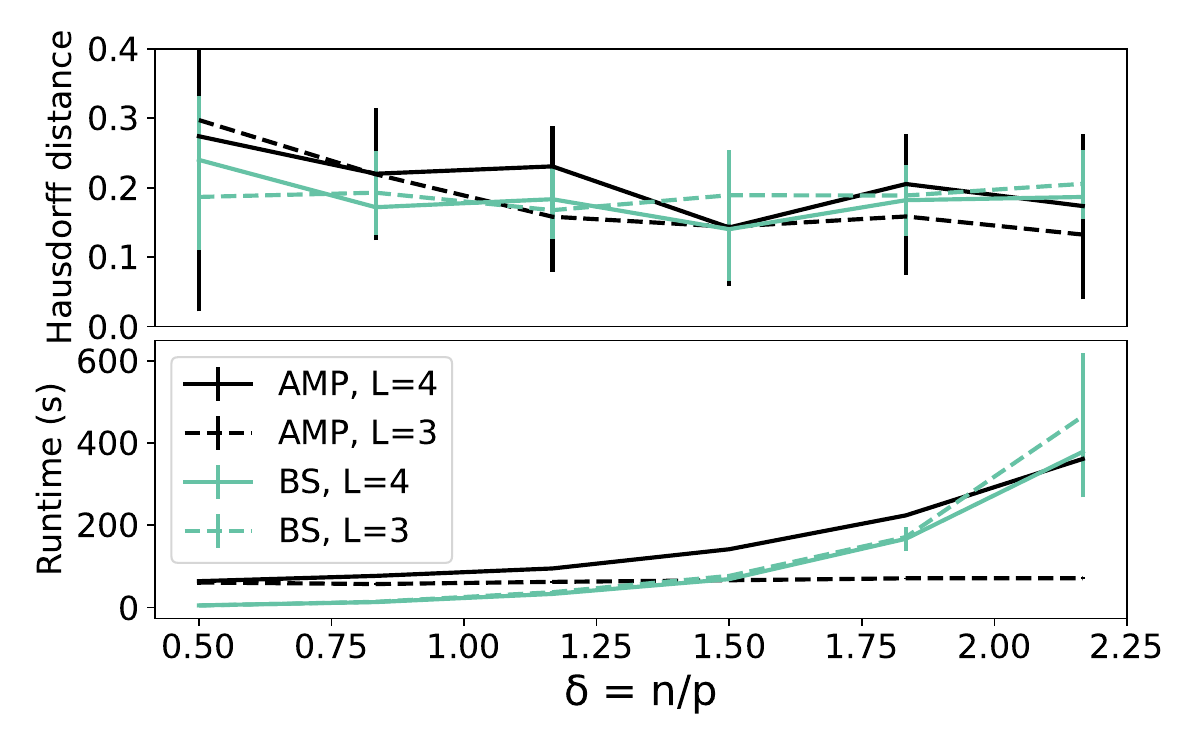}
    \label{fig:logistic_L3or4}
}
\subfloat[Algorithms assume  exactly 3 signals.]{%
    \centering
\includegraphics[width=0.475\linewidth]{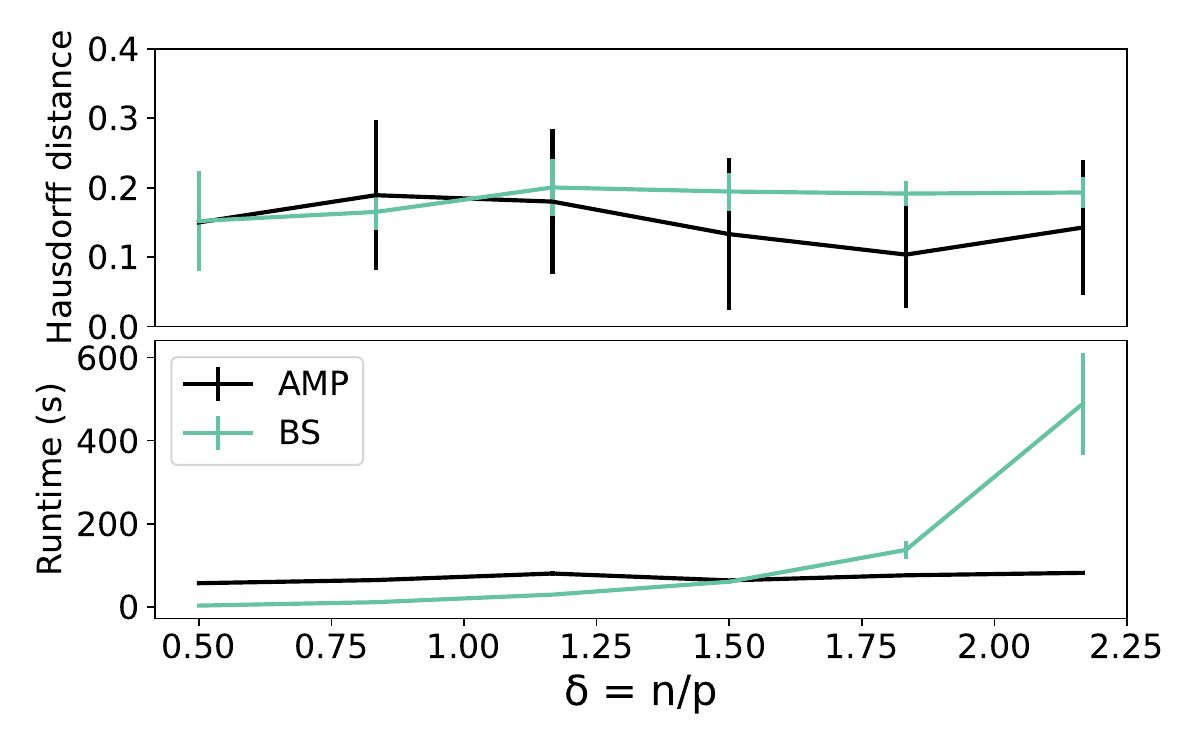}   
\label{fig:logistic_exactly_L3}
\vspace{0.4cm}
}
\caption{AMP vs.\@ the binary segmentation (BS) algorithm in \cite{wang_efficient_2023} for  logistic regression. The black plots correspond to AMP and the green plots correspond to BS. $\P_{\bar{\B}} =  0.5\N(\0, 15\delta \I) +0.5 \delta_{\0}$,  $L^*=3$. } 
\label{fig:logistic_compare}
\end{figure}
\paragraph{Inferring Change Points in Myocardial Infarction Data}
We consider the Myocardial Infarction (MI) complications data set from \cite{misc_myocardial_infarction_complications_579}, which contains the medical information of $n = 1700$  patients (samples) with MI complications. 
Each  sample has $p = 111$ medical features such as age, sex, heredity, and the presence of diabetes. The data set also contains $12$ binary response variables for each patient relating to the state of the patient's overall heart health, indicating the presence of complications such as `Atrial Fibration' and `Chronic Heart Failure' (CHF). We investigate the relation between the binary CHF response  variable   and the features of each patient, using the logistic model. 
We sort the data by age in the range $26-92$ and exclude age as a feature, with the aim of determining whether the relationship between the response and the features changes markedly at a certain age. We whiten the age-ordered $n \times (p-1)$ feature matrix, replacing missing values with interpolated estimates. We set $L = 2, \Delta = n/500$, and select the signal prior $\P_{\bar{\B}} = \N\left(\0, \begin{bmatrix}
   1 & 3/4\\
   3/4 & 1
\end{bmatrix}\right)$ independently of the data. For $\pi_{\bar{\bPsi}}$ we first specify a uniform prior over the number of change points (0 or 1 change points with probability $\frac{1}{2}$ each); we then define a uniform prior on the change point location for the case of 1 change point. We compute the estimated posterior $p(\bpsi| \TTheta^t, \y)$ from \eqref{eq:posterior_convergence}, reporting the existence of a change point with posterior probability close to 1. Assuming one change point, we plot the estimated posterior over all possible single change point locations in Figure \ref{fig:real_data_plots} (left). 
The plot displays strong concentration of the posterior within the age range $55-75$, attaining its maximum value at age $66$. This indicates the existence of a change point in the age range $55 - 75$ in how the presence of CHF relates to various medical features such as heredity, sex, and diabetes. This is consistent with findings in the medical literature, where older patients with the presence of CHF (aged $>75$ years) were found to be more often female, have fewer cardiovascular diseases, and have fewer associated risk factors than patients with CHF aged $55$ years or less \citep{stein_diversity_2012, Azad2014-ky,tromp_age_2021}. 

We further validate our findings by  running $\ell_2$-regularised  logistic regression separately on the data samples before and after the age of $66$. The resulting estimated parameter vectors are named `Estimated Vector $1$' and `Estimated Vector $2$' in Figure \ref{fig:real_data_plots} (right). We apply class weighting to the loss function to mitigate the effect of class imbalance, and use $5$-fold cross-validation to select the regularisation constant. We observe from the plot in Figure \ref{fig:real_data_plots} (right) that the found regression vectors differ significantly, with a correlation coefficient of $0.09$, yielding further evidence for a possible change in the underlying data generating mechanism before and after age $66$. Moreover, we observe that the squared sum of the coefficients of `Estimated Vector 2' ($1.9$) is significantly smaller than that of `Estimated Vector 1' ($5.9$), a finding that is consistent with the observation that traditional risk factors provide less combined evidence of heart failure risk in elderly patients \citep{tromp_age_2021}. 
\begin{figure}[t!]
\centering
\subfloat{%
  \centering
  \includegraphics[width=0.4999\linewidth]{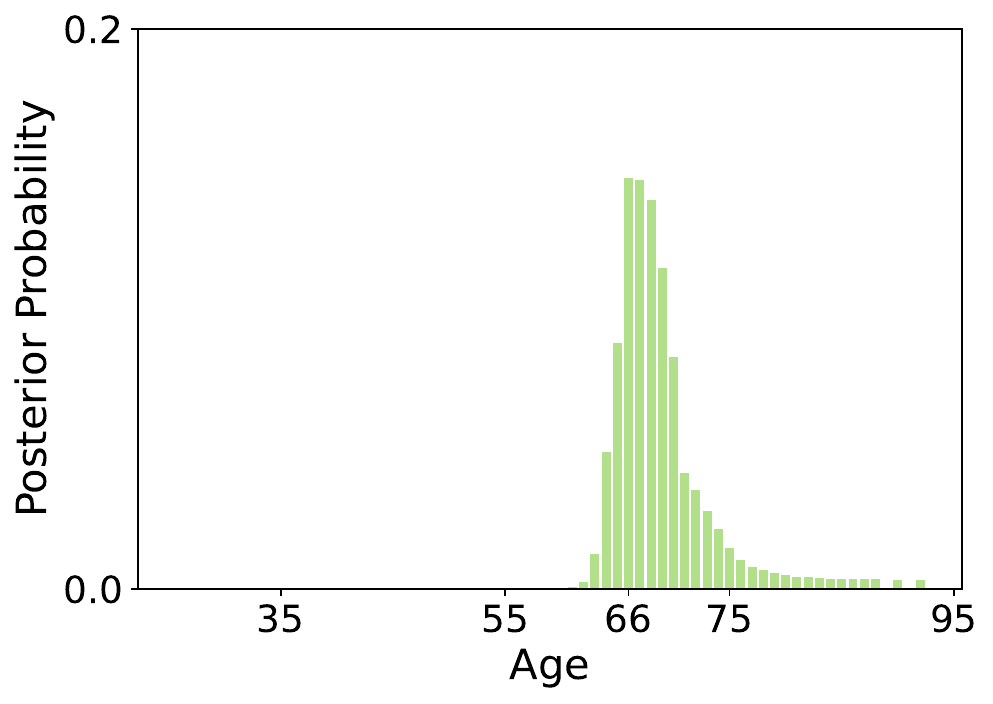}
  \label{fig:realdata_exact_post}
}
\subfloat{
  \centering
  \includegraphics[width=0.466\linewidth]{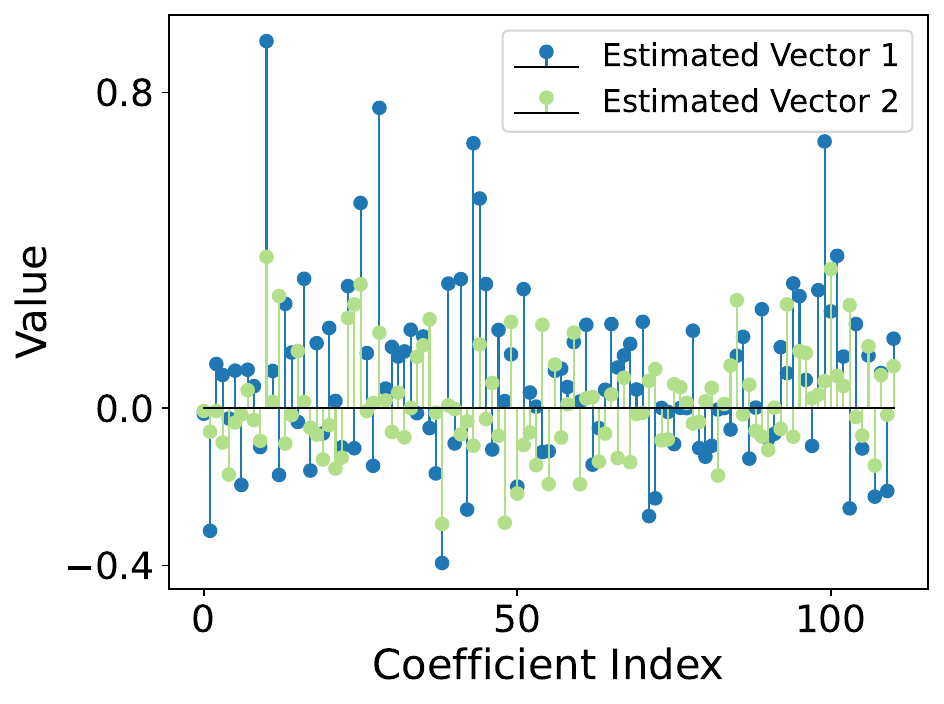}
  \label{fig:real_data_stemplot}
}
\vspace{-0.7cm}
\caption{Change point posterior $p(\cdot | \TTheta^t, \y)$ with respect to age in the Myocardial Infarction data set (left), coefficients of both regression vectors estimated using samples before and after the age of $66$ (right).}
\label{fig:real_data_plots}
\vspace{-0.4cm}
\end{figure}

\section{Discussion}
In this work we proposed an AMP algorithm  for regression with change points. The algorithm can efficiently incorporate prior information on the change points, and its state evolution characterization gives rigorous asymptotic guarantees on estimation accuracy. It also allows to quantify uncertainty in the change point estimates via a computable posterior distribution. 

In principle, the choice of change point prior could significantly affect the performance of our proposed algorithm---for example, the performance would be poor if we choose a prior that assigns very small probability mass to the true change-point locations. However, our experimental results indicate that an uninformative prior (uniform on both the number and locations of the change points) consistently gives good performance across a range of settings. The algorithm also uses  priors on the signal $\B$ and the measurement noise  $\bvarepsilon$ to tailor the AMP denoising functions $f^t, g^t$. A misspecified prior on $\bvarepsilon$ or $\B$ would lead to choices on $f^t, g^t$ that are suboptimal, but we emphasize the algorithm and the theoretical guarantees remain valid: the performance of the algorithm is characterized by state evolution parameters, which in turn depend on the denoising functions used within AMP.  Figures \ref{fig:sparse_sparse_diff}, \ref{fig:satellite}, \ref{fig:logistic_compare}, \ref{fig:real_data_plots}, \ref{fig:extra_comparison_DCDP} illustrate the robustness of our algorithm to lack of exact knowledge about signal or noise priors.

A limitation of our method is that the AMP algorithm and its theoretical guarantees are derived assuming that the covariates are i.i.d. Gaussian. We discuss below  how this assumption can be relaxed in certain settings, along with other directions for future work.


\paragraph{Non-Gaussian Design}
Based on recent universality results for AMP \citep{Wang22Universality}, we expect that our algorithm and theoretical guarantees will apply to a broad class of i.i.d. designs. We also expect that our approach can be generalized to the much broader class of rotationally invariant designs, for which AMP-like algorithms have been proposed for regression without change points  \citep{ma2017orthogonal, rangan2019vector, takeuchi2020rigorous, pandit2020inference}. 

\paragraph{Gaussian Design with General Covariance}
In the setting where $\X_i \distas{i.i.d} \N(\0, \bSigma)$ for $i \in [n]$ and for some covariance matrix $\bSigma \in \reals^{p \times p}$, we can write $\X_i =  \bSigma^{1/2}\tilde{\X}_i$,  where  $\tilde{\X}_i \distas{i.i.d} \N(\0, \I_p)$ for $i \in [n]$.  The model \eqref{eq:chgpt-model} can then be equivalently written as:
\begin{align*}
   y_i = q\big( (\tilde{\X}_i)^\top (\bSigma^{1/2} \bbeta^{(i)}), \, \varepsilon_i \big), \quad i = 1, \dots, n,
\end{align*}
where $\tilde{\X} \in \reals^{n \times p}$ is now our effective design matrix composed of stacked rows $(\tilde{\X}_i)^\top$, and $\bSigma^{1/2} \bbeta^{(i)}$ is our effective signal. If $\bSigma \in \reals^{p \times p}$ is known or can be approximated with reasonable accuracy, the user can run the AMP iteration \eqref{eq:amp} with $\tilde{\X}$ in place of $\X$. Since the entries of $\tilde{\X}$ are i.i.d. Gaussian, we can apply Theorem \ref{thm:SE} to this iteration to  obtain guarantees on the accuracy of change point estimation (via Proposition \ref{prop:hausdorff_asymptotics}) and guarantees on the recovery of $\bSigma^{1/2} \B$ (via Theorem \ref{thm:SE}). 
Because $\bSigma$ is known, this also gives us guarantees on the recovery of $\B$. 

The more realistic setting of unknown $\bSigma$ is challenging,  and developing effective AMP-like algorithms for this case is an interesting open question. One idea is to use weighted versions of standard convex estimators for high-dimensional regression such as LASSO or M-estimators, where the weights can be chosen to encode prior information on change point locations. By leveraging the connection between AMP and convex optimization (Section 4.4 of \citealt{feng_unifying_2022}), it is possible to analyze such convex estimators under unknown covariance $\bSigma$, as in \citep{zhao_asymptotic_2022}.

\paragraph{Design with Temporal Dependence}
\citet{xu_change_2022} recently studied the problem of change point detection in linear regression  with temporally dependent data, where the rows of the design matrix $\X$ form a high-dimensional time-series. An AMP framework has recently been proposed for linear regression when the data are generated from a high-dimensional time series  \citep{tieplova2025informationtheoreticlimitsapproximatemessagepassing}. Using this framework to generalize our change point AMP to temporal dependent data is a promising direction for future work. 
%

%
\acks{G. Arpino is supported by a Cambridge Trust Scholarship and X. Liu was supported   by a Schlumberger Cambridge International Scholarship.
}

\newpage
\appendix
\section{Proof of State Evolution} \label{appendix:SEproof}

The proof of Theorem \ref{thm:SE} relies on a generalization of Lemma 14 in \cite{gerbelot_graph-based_2023}, presented as Lemma \ref{lemma:generalized_lemma_14} below. Let $\W_{0}\in \mathbb{R}^{N \times Q}$ be a matrix such that $\|\W_0^\top \W_0\|_F / N$ converges to a finite constant as $N \to \infty$, and let $\A \in \mathbb{R}^{N \times N}$ be a symmetric GOE($N$) matrix, independent of $\W_{0}$. Then Lemma 14 in \cite{gerbelot_graph-based_2023} gives a state evolution result  involving an AMP iteration whose denoising function takes as input the output of the generalized linear model $\varphi(\A \W_0)$, where $\varphi : \mathbb{R}^{N \times Q} \to \mathbb{R}^{N}$. Lemma \ref{lemma:generalized_lemma_14} 
generalizes Lemma 14 of \citet{gerbelot_graph-based_2023} in two ways: i) it allows for the inclusion of an auxiliary matrix $\bXi \in \reals^{N \times L_{\bXi}}$ so that the  generalized linear model in question is of the form $\varphi(\A \W_0, \bXi)$, and ii) the state evolution convergence result holds for pseudo-Lipschitz test functions taking both the auxiliary matrix $\bXi$ and   $\A \W_0$ as inputs. 

We extend this lemma by considering an independent random matrix $\bXi \in \reals^{N \times L_{\bXi}}$, serving as input to a set of pseudo-Lipschitz functions $\varphi : \mathbb{R}^{N \times (Q + L_{\bXi})} \to \mathbb{R}^{N}$. We then analyze a similar AMP iteration whose denoising function $\tilde{f}^t$ takes $\varphi(\A \W_0, \bXi)$ as input instead of $\varphi(\A \W_0)$, initialized with an independent initializer $\X^{0} \in \mathbb{R}^{N \times Q}$ and $\bM^{-1} := \0 \in \reals^{N \times Q}$:
\begin{align}
    &\bX^{t+1} = \A\bM^{t}-\bM^{t-1}(\bb^{t})^{\top} && \in \R^{N\times Q} \, ,  \label{eq:symm_AMP_Xt}\\
    &\bM^{t} =\tf^{t}(\varphi\left(\A\W_{0}, \bXi\right),\bX^{t}) && \in \R^{N\times Q} \, ,  \label{eq:symm_AMP_ft} \\
    &\bb^t = \frac{1}{N} \sum_{i=1}^N \frac{\partial \tf^t_i}{\partial \bX_i^t}(\varphi\left(\A\W_{0}, \bXi\right),\bX^t) && \in
    \R^{Q\times Q}\,  \label{eq:symm_AMP_deriv}.
\end{align}

Our result in Lemma \ref{lemma:generalized_lemma_14} presents an asymptotic characterization of \eqref{eq:symm_AMP_Xt}--\eqref{eq:symm_AMP_deriv} via the following state evolution recursion: 
\begin{align}
        \label{eq:generalized_lemma_14_init_nu}
        &\boldsymbol{\nu}^{0}, \hat{\bnu}^0 = \0, \\
        \label{eq:generalized_lemma_14_init_SE}
        &\boldsymbol{\kappa}^{0,0} = \lim_{N \to \infty} \frac{1}{N} \tf^{0}(\bX^{0})^{\top}\tf^{0}(\bX^{0}), \\
        \label{eq:generalized_lemma_14_nu}
        &\boldsymbol{\nu}^{t+1} = \lim_{N \to \infty} \frac{1}{N} \mathbb{E}\left[\W_{0}^{\top} \right. \left. \tf^{t}\left(\varphi(\Z_{\W_{0}}, \bXi), \Z_{\W_{0}}\brho_{\W_{0}}^{-1}\boldsymbol{\nu}^{t}+\W_{0}\hat{\boldsymbol{\nu}}^{t}+\Z^{t}\right)\right],  \\
        \label{eq:generalized_lemma_14_nu_hat}
        &\hat{\boldsymbol{\nu}}^{t+1} = \lim_{N \to \infty} \frac{1}{N} \mathbb{E}\left[\sum_{i=1}^{N} \partial_{1i} \bar{f}_{i}^{t} \left(\Z_{\W_{0}}, \Z_{\W_{0}}\brho_{\W_{0}}^{-1}\boldsymbol{\nu}^{t}+\W_{0}\hat{\boldsymbol{\nu}}^{t}+\Z^{t}, \bXi\right) \right], \\ \label{eq:generalized_lemma_14_kappa} 
        &\boldsymbol{\kappa}^{t+1, s+1} = \boldsymbol{\kappa}^{s+1, t+1} \notag \\
        &= \lim_{N \to \infty} \frac{1}{N} \mathbb{E}\left[ \left(\tf^{s}\left(\varphi(\Z_{\W_{0}}, \bXi), \Z_{\W_{0}}\brho_{\W_{0}}^{-1}\boldsymbol{\nu}^{s}+\W_{0}\hat{\boldsymbol{\nu}}^{s}+\Z^{s}\right) -\W_{0}\brho_{\W_{0}}^{-1}\boldsymbol{\nu}^{s+1}\right)^{\top} \right. \nonumber \\
        & \left. \hspace{4.0cm}\left(\tf^{t}\left(\varphi(\Z_{\W_{0}}, \bXi), \Z_{\W_{0}}\brho_{\W_{0}}^{-1}\boldsymbol{\nu}^{t}+\W_{0}\hat{\boldsymbol{\nu}}^{t}+\Z^{t}\right)  -\W_{0}\brho_{\W_{0}}^{-1}\boldsymbol{\nu}^{t+1}\right)\right],
    \end{align}
where $\brho_{\W_{0}} = \lim_{N \to \infty} \frac{1}{N}\W_{0}^{\top}\W_{0}$, and for $i \in [n]$ we have $(\Z_{\W_{0}})_i \distas{i.i.d} \N(\zero, \brho_{\W_{0}})$. For $i \in [n], 0\leq s,r \leq t$, we have that $(\Z_{\W_{0}})_i$ is independent from $(\bZ^{s})_i \distas{i.i.d} \N\left(\0, \bkappa^{s,s}\right)$ with $\cov((\bZ^{s})_i, (\bZ^{r})_i) = \bkappa^{s,r}$.
In \eqref{eq:generalized_lemma_14_nu_hat}, we let $\bar{f}^t: (\z, \u, \v) \mapsto \tf^t(\varphi(\z, \v), \u)$ and we let $\partial_{1i} \bar{f}_i^t$ denote the partial derivative of $\bar{f}_i^t$ with respect to the $i$-th row of its first argument. We list the necessary assumptions for characterizing this AMP iteration, followed by the result:
\paragraph{Assumptions.}
\begin{enumerate}[font={\bfseries},label={(B\arabic*)}]
\item\label{it:ass-sym-1} $\A \in \mathbb{R}^{N \times N}$ is a GOE($N$) matrix, that is, $\A = \mathbf{G}+\mathbf{G}^{\top}$ for $\mathbf{G} \in \mathbb{R}^{N \times N}$ with i.i.d.~entries $G_{ij} \sim \N(0,1/(2N))$.
\item\label{it:ass-sym-2} For each $t \in \mathbb{N}_{>0}$, $\bar{f}^t: (\z, \u, \bXi) \mapsto \tf^t(\varphi(\z, \bXi), \u)$ is uniformly pseudo-Lipschitz. For each $t \in \mathbb{N}_{>0}$ and for any $1 \leqslant i \leqslant N$, $(\u, \z) \mapsto \frac{\partial \tf^t_i}{\partial \bX_i}(\varphi(\z, \bXi), \u)$ is uniformly pseudo-Lipschitz. The function $\tf^0: \mathbb{R}^{N \times Q} \to \mathbb{R}^{N \times Q}$ is uniformly pseudo-Lipschitz. 
\item \label{it:ass-sym-3} The initialization $\X^0$ is deterministic, and $\|\bX^{0}\|_{F}/\sqrt{N}$, $\|{\W_0^\top \W_0}\|_2 / N$, $\|{\bXi}\|_F / \sqrt{N}$ converge almost surely to finite constants as $N \to \infty$.
\item\label{it:ass-sym-4} The following limits exist and are finite:
\begin{equation*}
    \lim_{N \to \infty} \frac{1}{N} \tf^{0}(\bX^{0})^{\top} \tf^{0}(\bX^{0}), \quad \lim_{N \to \infty} \frac{1}{N} \W_0^\top \tf^{0}(\bX^{0}).
\end{equation*}
\item\label{it:ass-sym-5} For any $t \in \mathbb{N}_{>0}$ and any $\boldsymbol{\kappa} \in \cS_{Q}^{+}$, the following limit exists and is finite:
\begin{equation*}
    \lim_{N \to \infty} \frac{1}{N}  \mathbb{E}\left[\tf^{0}(\bX^{0})^{\top} \tf^{t}(\varphi(\Z, \bXi), \bZ)\right] 
\end{equation*}
where $\Z \in \mathbb{R}^{N \times Q}$, $\Z \sim \N\left(0,\boldsymbol{\kappa} \otimes \mathbf{I}_{N} \right)$.
\item\label{it:ass-sym-6} For any $s,t \in \mathbb{N}_{>0}$ and any $\boldsymbol{\kappa} \in \cS_{2Q}^{+}
$, the following limit exists and is finite:
\begin{equation*}
    \lim_{N \to \infty} \frac{1}{N}  \mathbb{E}\left[\tf^{s}(\varphi(\Z^s, \bXi), \bZ^{s})^{\top} \tf^{t}(\varphi(\Z^t, \bXi), \bZ^{t})\right]
\end{equation*}
where $(\Z^{s},\Z^{t}) \in (\mathbb{R}^{N \times Q})^{2}$,$(\Z^{s},\Z^{t}) \sim \N(0,\boldsymbol{\kappa} \otimes \mathbf{I}_{N})$.
\end{enumerate}

\begin{lemma} \label{lemma:generalized_lemma_14}
Consider the AMP iteration \eqref{eq:symm_AMP_Xt}--\eqref{eq:symm_AMP_deriv} and the state evolution recursion \eqref{eq:generalized_lemma_14_init_nu}--\eqref{eq:generalized_lemma_14_kappa}. 
Assume $\ref{it:ass-sym-1}-\ref{it:ass-sym-6}$. Then for any sequence of functions $\Phi_{N}:(\mathbb{R}^{N \times Q})^{\otimes (t+3)} \times \reals^{N \times L_{\bXi}} \to \mathbb{R}$ such that $(\X^1, \dots, \X^t, \V) \mapsto \Phi_{N}(\mathbf{X}^{1}, ..., \mathbf{X}^{t}, \V ; \W_0, \bXi)$ is uniformly pseudo-Lipschitz, we have that: 
    \begin{align}
        &\Phi_{N}\left(\X^{0}, \X^{1}, ..., \X^{t}, \A \W_{0} ; \W_0, \bXi\right) \stackrel{\P}\simeq \nonumber \\
        & \E_{\Z^1, \dots, \Z^t, \Z_{\W_0}}  \left[ \Phi_{N}\bigl(\X^0, \Z_{\W_0}\brho^{-1}_{\W_0}\bnu^1 + \W_0 \hat{\bnu}^t + \Z^1, \dots, \right. \nonumber \\
        & \hspace{5cm} \left. \Z_{\W_{0}}\brho_{\W_{0}}^{-1}\boldsymbol{\nu}^{t}+\W_{0}\hat{\boldsymbol{\nu}}^{t}+\Z^{t}, \Z_{\W_0} ; \W_0, \bXi \bigr) \right]. \label{eq:convergence_generalized_lemma14_SE}
    \end{align}
\end{lemma}
\begin{proof}
The main differences between the AMP result in \cite{gerbelot_graph-based_2023} and our characterization are: 1) the $\varphi$ function is allowed to depend on an auxiliary random variable $\bXi$, 2) the test function $\Phi_N$ is allowed to depend additionally on $\A \W_0$, and on the fixed parameters $\W_0$ and $\bXi$. Assumption \ref{it:ass-sym-2} guarantees that $\tf^t$ maintains the same required convergence properties despite modification $1)$. 

We now address modification 2). Since $\bXi$ and $\W_0$ are fixed and $(\X^0, \dots, \X^t, \V) \mapsto \Phi_{N}(\mathbf{X}^{0}, \mathbf{X}^{1}, ..., \mathbf{X}^{t}, \V ; \W_0, \bXi)$ is assumed to be uniformly pseudo-Lipschitz, these can be included in $\Phi_N$ and the right hand side of \eqref{eq:convergence_generalized_lemma14_SE} is unaffected. 
We now present the argument for why $\A\W_0$ can be included as in input to $\Phi_N$ on the LHS of \eqref{eq:convergence_generalized_lemma14_SE}, and yield a corresponding input $\Z_{\W_0}$ on the RHS of \eqref{eq:convergence_generalized_lemma14_SE}. Consider the following ``augmented'' AMP iteration, initialized with $\mathbf{X}^{aug, 0} := \begin{bmatrix}
    \X^0 & \X^0
\end{bmatrix} \in \mathbb{R}^{N \times 2Q}$ and $\bM^{aug, -1} := \0 \in \reals^{N \times 2Q}$: 
\begin{align}
    &\bX^{aug, t+1} = \A\bM^{aug, t}-\bM^{aug, t-1}(\bb^{aug, t})^{\top} && \in \R^{N\times 2Q} \, ,  \label{eq:symm_AMP_Xt_aug}\\
    &\bM^{aug, t} =\tf^{aug, t}(\varphi^{aug}\left(\A\W^{aug}_{0}, \bXi\right),\bX^{aug, t}) && \in \R^{N\times 2Q} \, ,  \label{eq:symm_AMP_ft_aug} \\
    &\bb^{aug, t} = \frac{1}{N} \sum_{i=1}^N \frac{\partial \tf^{aug, t}_i}{\partial \bX_i^{aug, t}}(\varphi^{aug}\left(\A\W^{aug}_{0}, \bXi\right),\bX^{aug, t}) && \in
    \R^{2Q\times 2Q}\,  \label{eq:symm_AMP_deriv_aug}.
\end{align}
In the above iteration, we define $\W_0^{aug} := \begin{bmatrix} \W_0 & \W_0 \end{bmatrix} \in \reals^{N \times 2Q}, \varphi^{aug}\left(\A\W^{aug}_{0}, \bXi\right) := \begin{bmatrix} \varphi\left(\A\W_{0}, \bXi\right) & \varphi\left(\A\W_{0}, \bXi\right) \end{bmatrix} \in \reals^{N \times 2Q}$. Letting $\X^{t + 1}, \tilde{\X}^{t+1} \in \reals^{N \times Q}$ take on the definition 
\[
\begin{bmatrix}
    \X^{t+1} & \tilde{\X}^{t+1}
\end{bmatrix} := \X^{aug, t+1}
\]
and defining: 
\[\tf^{aug, t}(\varphi^{aug}\left(\A\W^{aug}_{0}, \bXi\right),\bX^{aug, t}) := \begin{bmatrix}
    \tf^{t}(\varphi\left(\A\W_{0}, \bXi\right),\bX^{t}) & \xi^t(\tilde{\X}^{t})
\end{bmatrix}
\]
for some uniformly pseudo-Lipschitz function $\xi^t(\tilde{\X}^{t + 1})$ with $\xi^0(\tilde{\X}^0) := \W_0$, we obtain that: 
\begin{align}
    &\M^{aug, t} =: \begin{bmatrix} \M^t & \tilde{\M}^t \end{bmatrix} \\
    & \bb^{aug, t} = \begin{bmatrix}
        \bb^t & \0 \\
        \0 & *
    \end{bmatrix},
\end{align}
where $\M^t$ and $\bb^t$ are quantities from the initial iteration \eqref{eq:symm_AMP_Xt}--\eqref{eq:symm_AMP_deriv}, and $*$ denotes an unspecified matrix in $\reals^{Q \times Q}$. We therefore have that the first $Q$ columns of the matrix iteration \eqref{eq:symm_AMP_Xt_aug}--\eqref{eq:symm_AMP_deriv_aug} are equivalent to the original iteration \eqref{eq:symm_AMP_Xt}--\eqref{eq:symm_AMP_deriv}. The last $Q$ columns of the augmented iteration will be used to incorporate $\A \W_0$. 

For $t \geq 1$, let $(\X^{aug, t})_{[:, :Q]}, (\X^{aug, t})_{[:, Q:]}$ denote the $N \times Q$ matrices composed of the first $Q$ columns and of the last $Q$ columns of $\X^{aug, t}$, respectively. Let $\brho := \frac{1}{N} \W_0^\top \W_0$, and for $t \geq 1$ define the following random variables, independent of all else: $(\Z_{\W_0})_{i} \distas{i.i.d} \N(\0, \brho)$ and $(\Z^t)_i \distas{i.i.d} \N(0, \bkappa^{t, t})$ for $i \in [n]$. We compute the state evolution equations associated with the augmented AMP iteration for $t = 1$, outlined in $(C.25) - (C.28)$ of \cite{gerbelot_graph-based_2023}, to obtain: 
\begin{align*}
   &\bnu^{aug, 1} 
   = \lim_{N \to \infty} \frac{1}{N} \E \left[ \begin{bmatrix}
      \W_0^\top \\ \W_0^\top 
   \end{bmatrix} \begin{bmatrix}
       \tf^0(\X^0) & \xi^0(\tilde{\X}^0)
   \end{bmatrix}\right] 
   = \begin{bmatrix}
       \bnu^1 & \brho \\
       \bnu^1 & \brho
   \end{bmatrix}, \\
   &\hat{\bnu}^{aug, 1} = \begin{bmatrix}
       \hat{\bnu}^1 & \0 \\
       \0 & \0
   \end{bmatrix}, \\
   &\bkappa^{aug, 1, 1} = \lim_{N \to \infty} \frac{1}{N}  \\
   &\quad \E \left[ \left(\begin{bmatrix}
      \tf^0(\varphi(\Z_{\W_0}, \bXi), \Z_{\W_0} \brho^{-1} \bnu^1 + \W_0 \hat{\bnu}^1 + \Z^1) \\
      \W_0
   \end{bmatrix} - \begin{bmatrix}
      \W_0 \\
      \W_0
   \end{bmatrix}^\top \begin{bmatrix}
       \brho^{-1} \bnu^1 & \0 \\
       \0 & \I_{Q \times Q}
   \end{bmatrix}\right)^\top \right. \\
   &\hspace{1.5cm} \left. \left(\begin{bmatrix}
      \tf^0(\varphi(\Z_{\W_0}, \bXi), \Z_{\W_0} \brho^{-1} \bnu^1 + \W_0 \hat{\bnu}^1 + \Z^1) \\
      \W_0
   \end{bmatrix} - \begin{bmatrix}
      \W_0 \\
      \W_0
   \end{bmatrix}^\top \begin{bmatrix}
       \brho^{-1} \bnu^1 & \0 \\
       \0 & \I_{Q \times Q}
   \end{bmatrix}\right) \right]\\
   &\hspace{1.25cm} = \begin{bmatrix}
      \bkappa^{1, 1} & \0 \\
      \0 & \0
  \end{bmatrix}. 
\end{align*}
 Applying Lemma 14 in \cite{gerbelot_graph-based_2023} to the pseudo-Lipschitz function
\[
\tilde{\Phi}_N\left( \X^{aug, 1}, \dots, \X^{aug, t}\right) := \Phi_N\left( (\X^{aug, 1})_{[:, :Q]}, \dots, (\X^{aug, t})_{[:, :Q]}, (\X^{aug, 1})_{[:, Q:]}\right),\] 
we obtain:
\begin{align*}
    &\Phi_N\left( \X^1, \dots, \X^{t}, \A \W_0\right) \\
    & \Phi_N\left( (\X^{aug, 1})_{[:, :Q]}, \dots, (\X^{aug, t})_{[:, :Q]}, (\X^{aug, 1})_{[:, Q:]}\right) \\
    &= \tilde{\Phi}_N\left( \X^{aug, 1}, \dots, \X^{aug, t}\right) \\
    &\stackrel{\P}{\simeq} \E \tilde{\Phi}_N\left( \begin{bmatrix}
        \Z_{\W_0} & \Z_{\W_0}
    \end{bmatrix} \begin{bmatrix}
        \brho^{-1} \bnu^1 & \0 \\
        \0 & \I_{Q \times Q} 
    \end{bmatrix} + \begin{bmatrix}
        \W_0 & \W_0
    \end{bmatrix} \hat{\bnu}^{aug, 1} + \begin{bmatrix}
        \Z^1 & \0
    \end{bmatrix}, \right. \\
    & \qquad \left. \dots, \begin{bmatrix}
        \Z_{\W_0} & \Z_{\W_0}
    \end{bmatrix} \begin{bmatrix}
        \brho^{-1} \bnu^{t} & \0 \\
        \0 & \I_{Q \times Q} \end{bmatrix} + \begin{bmatrix}
        \W_0 & \W_0
    \end{bmatrix} \hat{\bnu}^{aug, t} + \begin{bmatrix}
        \Z^{t} & * 
    \end{bmatrix} \right) \\
    &= \E \tilde{\Phi}_N\left( \begin{bmatrix}
        \Z_{\W_0} \brho^{-1} \bnu^1 + \W_0 \hat{\bnu}^1 + \Z^1 & \Z_{\W_0}
    \end{bmatrix}, \dots, \right. \\
    & \qquad \qquad \left. \begin{bmatrix}
        \Z_{\W_0} & \Z_{\W_0}
    \end{bmatrix} \begin{bmatrix}
        \brho^{-1} \bnu^{t} & \0 \\
        \0 & \I_{Q \times Q} \end{bmatrix} + \begin{bmatrix}
        \W_0 & \W_0
    \end{bmatrix} \hat{\bnu}^{aug, t} + \begin{bmatrix}
        \Z^{t} & * 
    \end{bmatrix} \right) \\
    &= \E \Phi_N\left(\Z_{\W_0} \brho^{-1} \bnu^1 + \W_0 \hat{\bnu}^1 + \Z^1, \cdots, \Z_{\W_0} \brho^{-1} \bnu^t + \W_0 \hat{\bnu}^t + \Z^t, \Z_{\W_0} \right),
\end{align*}
and the result follows. 
\end{proof}
We next present the main reduction, mapping the AMP algorithm proposed in this work \eqref{eq:amp} to the symmetric one outlined in \eqref{eq:symm_AMP_Xt}--\eqref{eq:symm_AMP_deriv}. 
\begin{proof}[Proof of Theorem \ref{thm:SE}]
Consider the change point linear regression model \eqref{eq:chgpt-model}, and recall that it can be rewritten as \eqref{eq:model_psi}.  We reduce the algorithm \eqref{eq:amp} to \eqref{eq:symm_AMP_Xt}--\eqref{eq:symm_AMP_deriv}, following the alternating technique of \cite{javanmard_state-evolution_2013}. The idea is to define a symmetric GOE matrix with $\X$ and  $\X^\top$ on the off-diagonals. With a suitable initialization, the iteration \eqref{eq:symm_AMP_Xt}--\eqref{eq:symm_AMP_deriv} then yields $\B^{t+1}$ in the even iterations and for $\TTheta^t$ in the odd iterations. 

Let $N = n+p$. For a matrix $\bE\in\reals^{N\times L}$, we use $\bE_{[n]}$ and $\bE_{[-p]}$ to denote the first $n$ rows and  the last $p$ rows of $\bE$ respectively. Recall from \eqref{eq:amp} that $\X\in \R^{n\times p}, \B, \B^t\in \R^{p\times L}$ and $\TTheta = \X\B, \TTheta^t\in \R^{n\times L}$, and $n/p \to \delta$ as $n,p \to \infty$. We let
\begin{equation}\label{eq:def_A_W0}
\A = \sqrt{\frac{\delta}{\delta+1}} \begin{bmatrix}
\D_1 & \bX\\ \bX^\top & \D_2
\end{bmatrix} \sim \text{GOE}(N), \quad
\W_0=\begin{bmatrix}
\zero_{n\times L} \\
\B
\end{bmatrix}
\in \R^{N\times L}, \quad
\X^{0} = \begin{bmatrix}
\zero_{n \times L} \\
\B^{0}
\end{bmatrix}
\end{equation}
where $\D_1 \sim$ GOE($n$) and $\sqrt{\delta} \D_2\sim $ GOE($p$) are independent of each other and of $\X$. Let $\bXi := \begin{bmatrix} \bPsi & \bvarepsilon \end{bmatrix} \in \reals^{n \times 2}$ and define
\begin{equation}
    \varphi: (\W, \bXi) \mapsto q\left(\sqrt{\frac{\delta + 1}{\delta}} \W_{[n]}, \bPsi, \bvarepsilon \right) \in \reals^{N}, \label{eq:def_phi} 
\end{equation}
for any $\W \in \reals^{N \times L}$. We therefore have $\varphi(\A\W_0, \bXi) = q\left(\X \B, \bPsi, \bvarepsilon \right)$. Let \[\tilde{f}^t: \reals^{N \times L} \times \reals^N \to \reals^{N \times L}\] such that
\begin{align}
\begin{split}
&\tilde{f}^{2t+1}(\varphi(\A \W_0, \bXi), \U) = \sqrt{\frac{\delta+1}{\delta}} \begin{bmatrix}
g^t(\U_{[n]}, q(\X \B, \bPsi, \bvarepsilon)) \\ \0_{p \times L}
\end{bmatrix} \quad \text{and} \\
&\tilde{f}^{2t}(\varphi(\A \W_0, \bXi), \U)= \sqrt{\frac{\delta+1}{\delta}} \begin{bmatrix}
\zero_{n\times L} \\
f^t(\U_{[-p]})
\end{bmatrix}, 
\end{split} \label{eq:def_ftilde_reduction}
\end{align}
for any $\U$ in $\reals^{N \times L}$.

Next, consider the AMP iteration \eqref{eq:symm_AMP_Xt}--\eqref{eq:symm_AMP_deriv} with $\A, \W_0, \X^0, \bXi, \varphi, \tf^t$ defined as in \eqref{eq:def_A_W0}--\eqref{eq:def_ftilde_reduction}. Note that the assumptions $\ref{it:ass-sym-1}-\ref{it:ass-sym-6}$ are satisfied by construction, and hence the state evolution result in Lemma \ref{lemma:generalized_lemma_14} holds for the iteration \eqref{eq:symm_AMP_Xt}--\eqref{eq:symm_AMP_deriv}. We will now show that the state evolution equations decompose into those of interest, \eqref{eq:nu_B_SE}--\eqref{eq:kappa_theta_SE}. First, note that
\begin{align}
&\X^{2t + 1} \nonumber \\
&= \begin{bmatrix}
\X f^t(\X^{2t}_{[-p]}) - g^{t-1}(\X^{2t-1}_{[n]}, q(\X\B, \bXi)) \cdot \overbrace{ \frac{1}{n} \left(\sum_{i=1}^p \frac{\partial{f^t(\X^{2t}_{[-p]})_i}}{\partial{\X^{2t}_{[-p], i}}}\right)^\top }^{=(\F^{t})^\top}\\ \D_2 f^t(\X^{2t}_{[-p]})
\end{bmatrix}, \label{eq:SE_X_2t+1}\\
\nonumber\\
&\X^{2t} \nonumber \\
&= \begin{bmatrix}
\D_1 g^{t-1}(\X^{2t-1}_{[n]}, q(\X\B, \bXi)) \\
\X^\top g^{t-1}(\X^{2t-1}_{[n]}, q(\X\B, \bXi)) - f^{t-1}(\X^{2t-2}_{[-p]}) \cdot \underbrace{\left(\frac{1}{n} \sum_{i=1}^n \frac{\partial{g^{t-1}(\X^{2t-1}_{[n]}, q(\X\B, \bXi))}}{\partial{\X^{2t-1}_{[n],i}}} \right)^\top}_{=(\C^{t-1})^\top}
\end{bmatrix}, \label{eq:SE_X_2t}
\end{align}
and hence observe that $\X^{2t+1}_{[n]}$ and $\X^{2t}_{[-p]}$ are equal to $\TTheta^t$ and $\B^t$ in \eqref{eq:amp}  respectively. Define the main iterate $\bQ^t = \Z_{\W_0} \brho^{-1}_{\W_0} \bnu^{t} + \W_0 \hat{\bnu}^{t} + \Z^t \in\R^{N\times L}$. For $i \in [n]$, let $\tilde{g}_i^t: (\Z, \V, \bPsi, \bvarepsilon) \mapsto g_i^t(\V, q(\Z, \bPsi, \bvarepsilon))$ and let $\partial_{1i} \tilde{g}_i^t$ be the partial derivative (Jacobian) w.r.t. the $i$th row of the first argument.  Following \eqref{eq:generalized_lemma_14_init_nu}--\eqref{eq:generalized_lemma_14_kappa}, we then have that
\begin{align}
& \bnu^{2t+1} = \lim_{N \to \infty} \frac{1}{N} \sqrt{\frac{\delta + 1}{\delta}}\E\left[{\begin{bmatrix} \0_{L \times n} & \B^\top\end{bmatrix}} {\begin{bmatrix} \0_{n\times L} \\ 
f^t(\bQ^{2t}_{[-p]})\end{bmatrix}}\right], 
= \sqrt{\frac{\delta}{\delta + 1}} \lim_{n \to \infty} \frac{1}{n} \E\left[\B^\top f^t(\bQ^{2t}_{[-p]})\right]
\label{eq:nu_2t+1}\\
& \bnu^{2t} = \sqrt{\frac{\delta + 1}{\delta}} \lim_{N \to \infty} \frac{1}{N} \E \left[{\begin{bmatrix} \0_{L \times n} & \B^\top \end{bmatrix}} {\begin{bmatrix} g^{t-1}(\bQ_{[n]}^{2t-1}, \Y) \\ \0_{p\times L} \end{bmatrix}}\right] = \0_{L \times L}, \label{eq:nu_2t}\\
& \hat{\bnu}^{2t + 1} = \0_{L \times L}
\label{eq:nu_hat_2t+1}\\
& \hat{\bnu}^{2t} 
= \lim_{n \to \infty} \frac{1}{n} \E \left[\sum_{i=1}^{n} \partial_{1i}{\tilde{g}^{t-1}_i \left(\sqrt{\frac{\delta + 1}{\delta}}\Z_{\W_0,[n]}, \bQ^{2t-1}_{[n]}, \bPsi, \bvarepsilon \right)}\right] ,
\label{eq:nu_hat_2t}\\ 
& \bkappa^{2t, 2s} \notag\\
&= \lim_{n \to \infty} \frac{1}{n} \E\left[ g^{t-1}\left(\bQ^{2t-1}_{[n]}, q\left(\sqrt{\frac{\delta+1}{\delta}}\Z_{\W_0, [n]}, \bPsi, \bvarepsilon \right)\right)^\top \right. \nonumber \\
&\hspace{6cm} \left. \cdot \,  g^{s-1}\left(\bQ^{2s-1}_{[n]}, q\left(\sqrt{\frac{\delta+1}{\delta}}\Z_{\W_0, [n]}, \bPsi, \bvarepsilon\right)\right)\right] \label{eq:kappa_2t_2s}\\
& \bkappa^{2t + 1, 2s + 1} \notag \\
&= \lim_{n \to \infty}\frac{1}{n}\E\left[\left(f^t(\bQ_{[-p]}^{2t}) -\sqrt{\frac{\delta+1}{\delta}}\B\brho_{\B}^{-1}\bnu^{2t+1}\right)^\top\left(f^s(\bQ_{[-p]}^{2s}) -\sqrt{\frac{\delta+1}{\delta}}\B\brho_{\B}^{-1}\bnu^{2s+1}\right)\right]
\label{eq:kappa_2t+1_2s+1},
\end{align}
where in \eqref{eq:kappa_2t+1_2s+1} we used $\brho_{\W_0} = \frac{1}{N}\bB^{\top} \bB = \frac{\delta}{\delta+1}\brho_{\B}$.
Hence, we can associate
\begin{align}
&\bnu^{t}_{\TTheta} = \sqrt{\frac{\delta + 1}{\delta}} \bnu^{2t+1}, \quad \quad \label{eq:our_nu_Theta_in_terms_of_Lem14}
{\bnu}^{t}_{\B} = \hat{\bnu}^{2t}, \quad \quad \;\,
\bkappa_{\TTheta}^{t, s} = \bkappa^{2t+1, 2s+1},\quad 
\bkappa_{\B}^{t, s} = \bkappa^{2t, 2s},\\
&\Z  \stackrel{d}{=} \sqrt{\frac{\delta + 1}{\delta}} \Z_{\W_0, [n]},\quad\; 
\G^{t}_{\TTheta} \stackrel{d}{=} \Z^{2t+1}_{[n]}, \quad\;
\G^{t}_{\B} \stackrel{d}{=} \Z^{2t}_{[-p]},\\
&\V^{t}_{\TTheta} \stackrel{d}{=} \bQ^{2t+1}_{[n]}, \quad\quad \quad\qquad
\V^{t}_{\B} \stackrel{d}{=} \bQ^{2t}_{[-p]}.
\label{eq:our_V_B_in_terms_of_Lem14}
\end{align}
Substituting the change of variables \eqref{eq:our_nu_Theta_in_terms_of_Lem14}--\eqref{eq:our_V_B_in_terms_of_Lem14} into \eqref{eq:nu_2t+1}--\eqref{eq:kappa_2t+1_2s+1},  we obtain \eqref{eq:nu_B_SE}--\eqref{eq:kappa_theta_SE}. Moreover, substituting \eqref{eq:SE_X_2t+1}--\eqref{eq:SE_X_2t} and \eqref{eq:our_V_B_in_terms_of_Lem14} into Lemma \ref{lemma:generalized_lemma_14} yields Theorem \ref{thm:SE}. 
\end{proof}
\section{State Evolution Limits and Simplifications} \label{app:SE_limits}
\subsection{Parametric Dependence of State Evolution on Signal and Noise}
In this section, we give a set of sufficient conditions for the existence of the limits in the state evolution equations \eqref{eq:nu_B_SE}--\eqref{eq:kappa_theta_SE}, which also allow for removing the parametric dependence of the state evolution on $\B, \bvarepsilon$. Assume \ref{it:simp-ass-1} on p.\pageref{it:simp-ass-1}, and also that the following assumptions hold:
\begin{enumerate}[font={\bfseries},label={(S\arabic*)}]
\item \label{it:simp-ass-2} As $n\to \infty$, the entries of the normalized change point vector $\eeta/n$ converge to constants $\alpha_0, \dots, \alpha_L$ such that $0 = \alpha_0 < \alpha_1 < \dots <\alpha_{L^*} < \ldots < \alpha_L = 1$. 
\item \label{it:simp-ass-3} For $t \geq 0$, $f^t$ is separable, and $g^t$ acts row-wise on its input. (We recall that a separable function acts row-wise \emph{and} identically on each row.)
\item \label{it:simp-ass-4} For $\ell \in [L^* - 1]$, the empirical distributions  of 
$\{\tilde{g}_i^t(\Z_1, (\V^t_{\TTheta})_1, \Psi_{\eta_\ell}, \varepsilon_i) \}_{i \in [\eta_\ell, \eta_{\ell + 1})}$ and  $\{\partial_{1}{\tilde{g}_i^t(\Z_1, (\V^t_{\TTheta})_1, \Psi_{\eta_\ell}, \varepsilon_i)}\}_{i \in [\eta_\ell, \eta_{\ell + 1})}$ converge weakly to the laws  of random variables $\hat{g}^t_{\eta_\ell}(\Z_1, (\V^t_{\TTheta})_1, \Psi_{\eta_\ell},, \bar{\varepsilon}))$ and  ${\check{g}^t}_{\eta_\ell}(\Z_{1}, (\V_{\TTheta}^t)_1, \Psi_{\eta_\ell}, \bar{\varepsilon})$,  respectively, 
where $\partial_{1}$ denotes Jacobian with respect to the first argument. 
\end{enumerate}
The assumption  \ref{it:simp-ass-2} is natural in the  regime where the number of samples $n$ is proportional to $p$, and the number of degrees of freedom in the signals also grows linearly in $p$.  Without change points,   $f^t, g^t$ can both be assumed separable without loss of optimality (due to \ref{it:simp-ass-1}).  To handle the heterogeneous dependence structure induced by change points, we require  $g^t_i$ to depend on $i$, for $i \in [n]$. However, Proposition \ref{prop:opt_ensemble_ft_gt} shows that it can be chosen to act row-wise, that is, $g^{*t}_i(\TTheta^t, \by) = g^{*t}_ i(\TTheta^t_i, y_i)$. This justifies \ref{it:simp-ass-3}. 
When $g_i^t$ is chosen to be $g_i^{*t}$ for $i \in [n]$, the condition \ref{it:simp-ass-4} can be translated into regularity conditions on the prior marginals $\pi_{\bar{\Psi}_i}$ and distributional convergence conditions on the noise $\varepsilon_i$, as these are the only quantities that differ along the elements of the sets $\{\tilde{g}_i^t(\Z_1, (\V^t_{\TTheta})_1, \Psi_{\eta_\ell}, \varepsilon_i)\}_{i \in [\eta_\ell, \eta_{\ell + 1})}$ and $\{\partial_{1}{\tilde{g}_i^t(\Z_1, (\V^t_{\TTheta})_1, \Psi_{\eta_\ell}, \varepsilon_i)}\}_{i \in [\eta_\ell, \eta_{\ell + 1})}$ for $\ell \in [L^* - 1]$. Under assumptions \ref{it:simp-ass-1}-\ref{it:simp-ass-4}, the state evolution equations in \eqref{eq:nu_B_SE}--\eqref{eq:kappa_theta_SE} reduce to: 
\begin{align}
&{\bnu}_{\B}^{t+1} = \sum_{\ell = 0}^{L-1} \E\left[  {\check{g}^{t}}_{\eta_\ell}(\Z_{1}, (\V_{\TTheta}^t)_1, \Psi_{\eta_\ell}, \bar{\varepsilon}) \right], \label{eq:nu_B_SE_simplified}\\
&\bkappa_{\B}^{s+1, t+1} = \sum_{\ell = 0}^{L-1} \E\left[\hat{g}^s_{\eta_\ell}((\V_{\TTheta}^s)_1, q_1(\Z_{1}, \Psi_{\eta_\ell}, \bar{\varepsilon}))
\ \hat{g}^t_{\eta_\ell}((\V_{\TTheta}^t)_1, q_1(\Z_{1}, \Psi_{\eta_\ell}, \bar{\varepsilon}))^\top   \right], \label{eq:kappa_B_SE_simplified} \\
&\bnu^{t+1}_{\TTheta} = \frac{1}{\delta} \E\left[\bar{\B} f^{t+1}_1(({\bnu}^{t+1}_{\B})^\top \bar{\B}  + (\G^{t+1}_{\B})_1 )^\top  \right], \label{eq:nu_theta_SE_simplified}\\
&\bkappa_{\TTheta}^{s+1, t+1}  = \frac{1}{\delta} \E\left[\left(f^{s+1}_1(({\bnu}^{s+1}_{\B})^\top \bar{\B}  + ({\G}^{s+1}_{\B})_1) - (\bnu^{s+1}_{\TTheta})^\top \brho^{-1} \bar{\B} \right) \right. \notag \\
&\hspace{4cm} \left. \left(f^{t+1}_1(({\bnu}^{t+1}_{\B})^\top \bar{\B} + (\G^{t+1}_{\B})_1) -(\bnu^{t+1}_{\TTheta})^\top \brho^{-1} \bar{\B} \right)^\top \right], \label{eq:kappa_theta_SE_simplified}
\end{align}
where \ref{it:simp-ass-3} has allowed us to reduce the matrix products into sums of vector outer-products, and assumptions \ref{it:simp-ass-1}--\ref{it:simp-ass-4} have removed the parametric dependence on $\B$ and $\bvarepsilon$ due to the law of large numbers argument in Lemma 4 of \cite{Bayati_2011_dynamics}. 
\subsection{Removing Dependencies on Noise on the RHS of \eqref{eq:hausdorff_convergence} and \eqref{eq:posterior_convergence}}
Under assumptions \ref{it:simp-ass-1}--\ref{it:simp-ass-4} above, the parametric dependence of $\V_{\TTheta}^t$ on the RHS of \eqref{eq:hausdorff_convergence}--\eqref{eq:size_convergence} on $\B, \bvarepsilon$ can be removed. Moreover, the parametric dependence of the RHS on $\bvarepsilon$ through $\hat{\eeta}$ can be removed on a case-by-case basis. We expect this to hold, for example, under \ref{it:simp-ass-1}--\ref{it:simp-ass-4} for estimators whose parametric dependence on $\bvarepsilon$ is via a normalized inner product. For example the estimator in  \eqref{eq:example_estimator} can be expressed as:
\begin{align}
    &U(\hat{\eeta}(\V_{\TTheta}^t, q(\Z, U(\eeta), \bvarepsilon))) = \argmax_{\psi \in \mathcal{X}} \frac{\langle (\V_{\TTheta}^t)_{[:, \bpsi]}, \Z_{[:, \bpsi]} \rangle}{n} + \frac{\langle (\V_{\TTheta}^t)_{[:, \bpsi]}, \bvarepsilon \rangle}{n}.
\end{align}
Indeed, the numerical experiments in Section \ref{sec:experiments} for chosen estimators demonstrate a strong agreement between the left-hand and right-hand sides of \eqref{eq:hausdorff_convergence} and of \eqref{eq:posterior_convergence}, when $\bvarepsilon$ on the right-hand side is substituted by an independent copy with the same limiting moments. 

\section{Proof and Computation of Optimal Denoisers $f^t, g^t$ and likelihood $\cL$}
\subsection{Proof of Proposition \ref{prop:opt_ensemble_ft_gt}}\label{sec:optimal_ft_gt_proof}
This proof is similar to the derivation of the optimal denoisers for mixed regression in \cite{tan_mixed_2023}, with a few differences in the derivation of $g^{*t}$.  We include both parts here for completeness. 
The proof relies on  Lemma \ref{lem:cauchy} and Lemma \ref{lem:steins}, Cauchy-Schwarz inequality and Stein's Lemma extended to vector or matrix random variables. 
Recall that we treat vectors, including rows of matrices, as column vectors, and that functions $h: \R^q\to \R^q$ have column vectors as input and output. 
\begin{lemma} [Matrix Cauchy-Schwarz inequality, Lemma 2 in \cite{lavergne2008Cauchy}]\label{lem:cauchy}
\leavevmode \\ 
Let $\bA,\bB\in\reals^{n\times L}$ be random matrices such that $\E\|\bA\|_F^2<\infty, \E\|\bB\|_F^2<\infty,$ and $\E [\A^\top \A]$ is non-singular, then 
\begin{equation}
\E [\B^\top\B] - \E [\B^\top\A] (\E [\A^\top\A])^{-1} \E [\A^\top\B] \succcurlyeq 0.   
\end{equation}
\end{lemma}
\begin{lemma}[Extended Stein's Lemma]\label{lem:steins}
     Let $\x\in\reals^L$ and $h: \reals^{ L} \to \reals^L $ be such that for  $\ell\in[L]$, the function $h_\ell: x_{\ell} \to [h(\x)]_\ell$ is absolutely continuous for Lebesgue almost every $(x_i: i\neq \ell)\in \reals^{L-1} $, with weak derivative $\partial h_\ell(\x)/\partial x_\ell: \reals^{L} \to \reals$ satisfying $\E |\partial h_\ell({\x})/\partial {x}_{\ell} | <\infty$.
      If ${\x}\sim\N(\bmu, \bSigma)$ with $\bmu\in\reals^{L}$ and $\bSigma\in\reals^{L\times L} $ positive definite, then 
      \begin{equation}
      \E\left[(\x-\bmu) h(\x)^\top\right] = \bSigma \, \E[h'(\x)] ^\top,
      \end{equation}
      where 
      $h'(\x)$ is the Jacobian matrix of $h$. 
\end{lemma}
The proof of Lemma \ref{lem:steins} follows from Lemma 6.20 of \citep{feng_unifying_2022}.

\paragraph{Proof of part \ref{enum:ensemble_ft} (optimal $f^t$).} Using the law of total expectation and applying \eqref{eq:opt_ensemble_ft_lem}, 
we can rewrite  $ \bar{\bnu}_{\TTheta}^{t}$ in \eqref{eq:nu_bar_theta_SE}  as:
\begin{align}
\bar{\bnu}_{\TTheta}^{t} & = \frac{1}{\delta} \lim_{p \to \infty} \frac{1}{p}\sum_{j=1}^p \E\left[\bar{\B} \left(f_j^{t}(\bar{\V}_{\B}^{t})\right)^\top \right] 
=\frac{1}{\delta} \E \left[\bar{\B} \left( f_j^{t}(\bar{\V}_{\B}^t) \right)^\top\right] 
= \frac{1}{\delta}\E \left[\E\left[\bar{\B}\left(  f_j^{t}(\bar{\V}_{\B}^t)\right)^\top \big| \bar{\V}_{\B}^t \right] \right] \nonumber \\
&=\frac{1}{\delta}\E \left[\E\left[\bar{\B}  | \bar{\V}_{\B}^t \right]  \left(f_j^{t}(\bar{\V}_{\B}^t)\right)^\top\right] = \frac{1}{\delta} \E\left[f_j^{*t} (f_j^{t})^\top\right],\label{eq:nu_theta_expanded}
\end{align}
where we used the shorthand $f_j^t\equiv f_j^{t}(\bar{\V}_{\B}^t)$ and $f_j^{*t}\equiv f_j^{*t}(\bar{\V}_{\B}^t)= \E[\bar{\B}|\bar{\V}_{\B}^t]$. Applying Lemma \ref{lem:cauchy} yields
\begin{align}
    \E\left[f_j^{*t}(f_j^{*t})^\top\right] - \E\left[f_j^{*t}(f_j^t)^\top\right]
    \E\left[f_j^{t}(f_j^t)^\top\right]^{-1}
    \E\left[f_j^{t}(f_j^{*t})^\top\right] 
    \succcurlyeq 0.\label{eq:f_cauchy}
\end{align}
Since \eqref{eq:kappa_bar_theta_SE} can be simplified into 
\begin{align}
\bar{\bkappa}_{\TTheta}^{t,t}=\frac{1}{\delta}\E\left[f_j^{t}(f_j^t)^\top\right] - (\bar{\bnu}_{\TTheta}^t)^\top \brho^{-1} \bnu_{\TTheta}^t  ,\label{eq:kappa_theta_simplified}
\end{align}
using \eqref{eq:nu_theta_expanded} and \eqref{eq:kappa_theta_simplified} in \eqref{eq:f_cauchy}, we obtain that
\begin{align*}
   \bDelta :=\frac{1}{\delta} \E\left[f_j^{*t}(f_j^{*t})^\top\right] - \bar{\bnu}_{\TTheta}^t
    \left[ \bar{\bkappa}_{\TTheta}^{t, t}+ (\bar{\bnu}_{\TTheta}^t)^\top \brho^{-1} (\bnu_{\TTheta}^t)\right]^{-1}
    (\bar{\bnu}_{\TTheta}^t)^\top  \succcurlyeq 0.
\end{align*}
Adding and subtracting $\bar{\bkappa}_{\TTheta}^{t,t}$ on the LHS gives  
 $ \bar{\bkappa}_{\TTheta}^{t,t}-\bar{\bkappa}_{\TTheta}^{t,t}+\bDelta\succcurlyeq 0$. 
Left multiplying 
 by $ \brho^\top \left((\bar{\bnu}_{\TTheta}^t)^{-1}\right)^\top$ and right multiplying by $ (\bar{\bnu}_{\TTheta}^t)^{-1}\brho$ maintains the positive semi-definiteness of the LHS and  further yields
 \begin{align}
  \brho^\top \left((\bar{\bnu}_{\TTheta}^t)^{-1}\right)^\top\bar{\bkappa}_{\TTheta}^{t,t}(\bar{\bnu}_{\TTheta}^t)^{-1}\brho - \underbrace{\brho^\top \left((\bar{\bnu}_{\TTheta}^t)^{-1}\right)^\top (\bar{\bkappa}_{\TTheta}^{t,t}-\bDelta)(\bar{\bnu}_{\TTheta}^t)^{-1}\brho }_{=:\bGamma_{\TTheta}^t} \succcurlyeq 0,
 \end{align}
 which implies
\begin{align}\label{eq:ft_trace_inequality}
     \trace\left( \brho^\top \left((\bar{\bnu}_{\TTheta}^t)^{-1}\right)^\top\bar{\bkappa}_{\TTheta}^{t,t}(\bar{\bnu}_{\TTheta}^t)^{-1}\brho\right) \ge   \trace\left(\bGamma_{\TTheta}^t \right).
 \end{align}
Recall from \eqref{eq:alternate_trace_theta} that the LHS of \eqref{eq:ft_trace_inequality} is the objective we wish to minimise via optimal $f^t$. Indeed, setting  $f^{t}=f^{*t}$, \eqref{eq:ft_trace_inequality} is satisfied with equality, proving part \ref{enum:ensemble_ft} of Proposition \ref{prop:opt_ensemble_ft_gt}.

\begin{remark}\label{rmk:opt_ft_se_simplification}
  When $f^t=f^{*t}$,  \eqref{eq:nu_theta_expanded} and \eqref{eq:kappa_theta_simplified}   reduce to 
$
\bar{\bkappa}_{\TTheta}^{t,t} =\bar{\bnu}_{\TTheta}^t - (\bar{\bnu}_{\TTheta}^t)^\top \brho^{-1}\bar{\bnu}_{\TTheta}^t
$.
\end{remark}

\paragraph{Proof of part \ref{enum:ensemble_gt} (optimal $g^t$)} 
Recall from \eqref{eq:nu_bar_B_SE}  that 
\[\bar{\bnu}_{\B}^{t+1} = \lim_{n \to \infty} \frac{1}{n} \sum_{i=1}^n \E\left[\partial_{1i} \tilde{g}_i^t(\Z_1, (\bar{\V}_{\TTheta}^t)_1, \bar{\Psi}_i, \bar{\varepsilon}) \right].\]
We can rewrite the transpose of each summand as follows:
\begin{align}
&\E\left[\partial_{1i} \tilde{g}_i^t(\Z_1, (\bar{\V}_{\TTheta}^t)_1, \bar{\Psi}_i, \bar{\varepsilon}) \right]^\top\nonumber
 \\
&\stackrel{(a)}{=}
\E_{(\bar{\V}_{\TTheta}^t)_1}\left\lbrace
\E_{\Z_1, \bar{\varepsilon}}\left[\partial_{1i}\tilde{g}_i^t\left(\Z_1, (\bar{\V}_{\TTheta}^t)_1, \bar{\Psi}_i, \bar{\varepsilon}\right) \bigg| (\bar{\V}_{\TTheta}^t)_1\right] \right\rbrace^\top\nonumber\\
&\stackrel{(b)}{=} 
\E_{(\bar{\V}_{\TTheta}^t)_1}\left\lbrace
\cov\left(\Z_1|(\bar{\V}_{\TTheta}^t)_1\right)^{-1} \right. \nonumber\\
&\hspace{0.5cm} \left. 
\cdot \E_{\Z_1, \bar{\Psi}_i, \bar{\varepsilon}}\left[\left(\Z_1 - \E[\Z_1|(\bar{\V}_{\TTheta}^t)_1]\right) \tilde{g}_i^t\left(\Z_1, (\bar{\V}_{\TTheta}^t)_1, \bar{\Psi}_i, \bar{\varepsilon}\right)^\top \bigg|(\bar{\V}_{\TTheta}^t)_1 \right]
\right\rbrace\nonumber\\
&\stackrel{(c)}{=}
\E_{(\bar{\V}_{\TTheta}^t)_1}\left\lbrace
\cov\left(\Z_{1}|(\bar{\V}_{\TTheta}^t)_1\right)^{-1} \right. \nonumber\\
& \hspace{0.5cm} \left. \cdot \E_{q(\Z_1, \bar{\Psi}_i, \bar{\varepsilon})}\left[\E_{\Z_1}\left[\left(\Z_1 - \E[\Z_1|(\bar{\V}_{\TTheta}^t)_1]\right) g_i^t\left( (\bar{\V}_{\TTheta}^t)_1, q(\Z_1, \bar{\Psi}_i, \bar{\varepsilon})\right)^\top \bigg|(\bar{\V}_{\TTheta}^t)_1, q(\Z_1, \bar{\Psi}_i, \bar{\varepsilon})\right] \right]
\right\rbrace\nonumber\\
&=\E_{(\bar{\V}_{\TTheta}^t)_1}\left\lbrace
\cov\left(\Z_{1}|(\bar{\V}_{\TTheta}^t)_1\right)^{-1} \right. \notag \\
& \hspace{0.5cm} \left.
\cdot \E_{q(\Z_1,  \bar{\Psi}_i, \bar{\varepsilon})}\left[\left(
\E[\Z_1|(\bar{\V}_{\TTheta}^t)_1, q(\Z_1, \bar{\Psi}_i, \bar{\varepsilon})]- \E[\Z_1|(\bar{\V}_{\TTheta}^t)_1]\right) g_i^t\left((\bar{\V}_{\TTheta}^t)_1, q(\Z_1, \bar{\Psi}_i, \bar{\varepsilon})\right)^\top\right]  
\right\rbrace\nonumber \\[18pt]
& \stackrel{(d)}{=} 
\E\left\lbrace
\smash[t]{\overbrace{\cov\left(\Z_{1}|(\bar{\V}_{\TTheta}^t)_1\right)^{-1} 
\left(\E\left[
\Z_1 |(\bar{\V}_{\TTheta}^t)_1, q(\Z_1, \bar{\Psi}_i, \bar{\varepsilon})\right] - \E[\Z_{1}|(\bar{\V}_{\TTheta}^t)_1]\right)}^{=g_i^{*t}\left((\bar{\V}_{\TTheta}^t)_1, q(\Z_1, \bar{\Psi}_i, \bar{\varepsilon})\right)}} \right. \nonumber \\
&\hspace{10.0cm} \left. \cdot g_i^t\left((\bar{\V}_{\TTheta}^t)_1, q(\Z_1, \bar{\Psi}_i, \bar{\varepsilon})\right)^\top
\right\rbrace \nonumber\\
&= \E\left[g_i^{*t} (g^{t}_i)^\top\right]
\label{eq:E_partial_g_expand},
\end{align}
where (a) and (c) follow from the law of total expectation; (b) uses Lemma \ref{lem:steins}; and (d) uses \eqref{eq:opt_ensemble_gt_lem}.  In the last line of \eqref{eq:E_partial_g_expand} we have used the shorthand $g_i^t\equiv g_i^t\left((\bar{\V}_{\TTheta}^t)_1, q(\Z_1, \bar{\Psi}_i, \bar{\varepsilon})\right)$ and $g_i^{*t}\equiv g_i^{*t}\left((\bar{\V}_{\TTheta}^t)_1, q(\Z_1, \bar{\Psi}_i, \bar{\varepsilon})\right) $. Substituting \eqref{eq:E_partial_g_expand} in \eqref{eq:nu_bar_B_SE} yields
\begin{align}
    \bar{\bnu}_{\B}^{t+1} = \lim_{n\to \infty}\frac{1}{n}\sum_{i=1}^n \E[g_i^{t}(g_i^{*t})^\top]
    = \lim_{n\to \infty} \E\left[\frac{1}{n}(g^{t})^\top g^{*t}\right].\label{eq:nu_B_expanded}
\end{align}
 Note that since $\pi_{\bar{\Psi}_i}$ differs across $i$, $g^t_i$ differs across $i$ and  the sum in \eqref{eq:nu_B_expanded} cannot be reduced.  Lemma \ref{lem:cauchy} implies that
 \begin{align}
\E\left[\frac{1}{n}(g^{*t})^\top g^{*t}\right] - \E\left[\frac{1}{n}(g^{*t})^\top(g^t)\right]
    \E\left[\frac{1}{n}(g^{t})^\top g^t\right]^{-1}
    \E\left[\frac{1}{n}(g^{t})^\top g^{*t}\right] 
    \succcurlyeq 0.\label{eq:g_cauchy}
 \end{align}
 Recalling from \eqref{eq:nu_bar_B_SE}-\eqref{eq:kappa_bar_B_SE} that the limiting
parameters $ \bar{\bnu}_{\B}^{t+1} = \lim_{n\to \infty}\E\left[\frac{1}{n} (g^{t})^\top g^{*t}\right]$ and $\bar{\bkappa}_{\B}^{t+1,t+1} = \lim_{n\to \infty} \E\left[\frac{1}{n} (g^t)^\top g^t\right]$ exist, we can take limits to obtain
 \begin{align}
   &\lim_{n\to\infty} \E\left[\frac{1}{n}(g^{*t})^\top g^{*t}\right] \notag \\
&\qquad - \underbrace{\lim_{n\to\infty}\E \left[\frac{1}{n}\left(g^{*t}\right)^\top g^t\right] }_{\left(\bar{\bnu}_{\B}^{t+1}\right)^\top}
\underbrace{\left(\lim_{n\to\infty} \E \left[\frac{1}{n}(g^t)^\top g^t\right]\right)^{-1}}_{\left(\bar{\bkappa}_{\B}^{t+1,t+1}\right)^{-1}}
    \underbrace{\lim_{n\to\infty} \E \left[\frac{1}{n}(g^t)^\top g^{*t}\right]}_{\bar{\bnu}_{\B}^{t+1}} \succcurlyeq  0.
 \end{align}
 Left multiplying $\bar{\bnu}_{\B}^{t+1} \left((\bar{\bnu}_{\B}^{t+1})^{-1}\right)^{\top}$ and right multiplying $(\bar{\bnu}_{\B}^{t+1} )^{-1}\left(\bar{\bnu}_{\B}^{t+1}\right)^{\top}$ on the LHS maintains the positive semi-definiteness of the LHS to give
 \begin{equation}
 \underbrace{\bar{\bnu}_{\B}^{t+1} \left((\bar{\bnu}_{\B}^{t+1})^{-1}\right)^{\top}\lim_{n\to\infty} \E\left[\frac{1}{n}(g^{*t})^\top g^{*t}\right]
 (\bar{\bnu}_{\B}^{t+1} )^{-1}\left(\bar{\bnu}_{\B}^{t+1}\right)^{\top}}_{=:\bGamma_{\B}^{t+1}}
 -
 \bar{\bnu}_{\B}^{t+1} \left(\bar{\bkappa}_{\B}^{t+1, t+1}\right)^{-1}\left(\bar{\bnu}_{\B}^{t+1}\right)^{\top}
    \succcurlyeq 0,
\end{equation}
Moreover, 
 since for positive definite matrices $\bGamma_1$ and $ \bGamma_2$,  $\bGamma_1 -\bGamma_2\succcurlyeq  0$ implies $\bGamma_1^{-1} -\bGamma_2^{-1}\preccurlyeq  0$, we have that
\begin{equation}
[(\bar{\bnu}_{\B}^{t+1})^{-1}]^\top
    \bar{\bkappa}_{\B}^{t+1,t+1} (\bar{\bnu}_{\B}^{t+1})^{-1} 
    -\bGamma_{\B}^{t+1} 
    \succcurlyeq 0,
\end{equation}
which implies
\begin{align*}
\trace\left([(\bar{\bnu}_{\B}^{t+1})^{-1}]^\top
    \bar{\bkappa}_{\B}^{t+1,t+1} (\bar{\bnu}_{\B}^{t+1})^{-1} \right) \geq \trace\left(\bGamma_{\B}^{t+1}\right).
\end{align*}
Recall from \eqref{eq:alternate_trace_B} that $g^t$ is optimized by minimising $ \trace\left([(\bar{\bnu}_{\B}^{t+1})^{-1}]^\top
    \bar{\bkappa}_{\B}^{t+1,t+1} (\bar{\bnu}_{\B}^{t+1})^{-1}\right)$. Indeed, by choosing $g^t=g^{*t}$, the objective  achieves its lower bound, which completes the proof.
\begin{remark}\label{rmk:opt_gt_se_simplification}
    Note   setting $g^t=g^{*t}$ leads to  
   $        \bar{\bnu}_{\B}^{t+1} = \bar{\bkappa}_{\B}^{t+1, t+1}$. 
\end{remark}
\subsection{Computation of the Optimal Denoiser $g^{*t}$}\label{sec:full_computation_gt}
In this section, we compute the optimal denoiser $g^{*t}$ for the linear, logistic, and rectified linear models. We shall  use $\phi(\bx; \bmu, \bSigma)$ to denote the probability density function (PDF) of $\normal(\bmu, \bSigma)$ evaluated at $\bx$. We use  $\phi(x)$ to denote the PDF of a standard scalar Gaussian $\normal(0,1)$ evaluated at $x$, and $\Phi(x)$ to denote its cumulative density function (CDF) up to $x$.  For a generic random vector $\bU$, with a slight abuse of notation we will use $\P(
\bU = \bu)$ to denote its density evaluated at $\bu$.
We begin with the following standard lemmas  as useful preliminaries.
\begin{lemma}[Conditioning property of multivariate Gaussian]\label{lem:Gaussian_condition}
Let $\x\in \reals^n$ and  $\y\in \reals^m$ be jointly Gaussian, with \[\begin{bmatrix}
            \x\\ \y
        \end{bmatrix} \sim \N\left(\begin{bmatrix}
            \bmu_{x}\\\bmu_{y}
        \end{bmatrix},
        \begin{bmatrix}
            \bSigma_{x} & \bSigma_{xy}\\
            \bSigma_{xy}^\top &  \bSigma_{y}
        \end{bmatrix}\right).\] 
Then $(\bx\mid \by=\tilde{\by})\sim \normal(\bmu, \bSigma)$ where
$$
\bmu=\bmu_{x} + \bSigma_{xy}\bSigma_{y}^{-1}(\tilde{\y} - \bmu_{y}), 
\quad 
\bSigma= \bSigma_{x} - \bSigma_{xy}\bSigma_{y}^{-1} \bSigma_{xy}^{\top}.$$
\end{lemma} 
\begin{lemma}[Conditioning property of  Gaussian mixtures]\label{lem:gau_mixture_cond}
Let  $\Gamma \sim \pi_\Gamma$ be a categorical variable supported on $[1, 2, \dots, K]$. With $\x\in\reals^{n}$ and $\y\in\reals^{m}$, suppose  
\begin{align}
\begin{bmatrix}
    \x \\ \y
\end{bmatrix}\sim
 \N\left( \begin{bmatrix}
    \bmu_{k, x} \\ \bmu_{k, y} 
\end{bmatrix},
\begin{bmatrix}
 \bSigma_{k,x}, \bSigma_{k,xy}\\
  \bSigma_{k,xy}^\top, \bSigma_{k,y}
\end{bmatrix}\right)   \qquad \text{if}\quad \Gamma=k.
\end{align}
Then, with $\Gamma$ is independent of $\bx, \by$, we have
\begin{align}
    \E[\x \mid  \y = \y_0] & = \frac{\sum_{k=1}^K \pi_\Gamma(k)\phi(\y_0; \bmu_{k, y}, \bSigma_{k, y}) \E[\x \mid \y = \y_0, \Gamma=k]}{\sum_{\tilde{k}=1}^K \pi_\Gamma(\tilde{k})\phi(\y_0; \bmu_{\tilde{k}, y}, \bSigma_{\tilde{k},y})}.
\label{eq:gau_mixture_cond_expectation}
\end{align}
\end{lemma}
\subsubsection{Computation of $g^{*t}$ for Linear Regression} \label{sec:full_computation_gt_linear}
Assume   Gaussian noise $\bar{\varepsilon}\sim \N(0, \sigma^2)$. 
Using \eqref{eq:V_bar_TTheta}, we have  
\begin{align}
\label{eq:Z1_vtheta_cov}
\begin{bmatrix}
 \Z_1\\   (\bar{\V}_{\TTheta}^t)_1
\end{bmatrix}\sim   \N_{2L}(\bzero, \bGamma^t), \quad  \text{where}\quad  
\bGamma^t: = \begin{bmatrix}
        \brho & \bar{\bnu}_{\TTheta}^t \\
        (\bar{\bnu}_{\TTheta}^t)^\top & \bSigma_{\V}^t
    \end{bmatrix}, \quad \bSigma_{\V}^t=(\bar{\bnu}_{\TTheta}^t)^\top \brho^{-1} \bar{\bnu}_{\TTheta}^t + \bar{\bkappa}_{\TTheta}^{t, t}.
\end{align}
 Lemma \ref{lem:Gaussian_condition} then gives the formulas for $\E[\Z_1 | (\bar{\V}_{\TTheta}^t)_1 = \V_i]$ and $\cov(\Z_1 | (\bar{\V}_{\TTheta}^t)_1= \V_i) $, the first and last terms of $g^{*t}_i$ in \eqref{eq:opt_ensemble_gt_lem}:
 \begin{align}
   \E[\Z_1 | (\bar{\V}_{\TTheta}^t)_1 = \V_i] & =
  \bar{\bnu}_{\TTheta}^t \left(\bSigma_{\V}^t\right)^{-1}  \V_i,\nonumber\\
    \cov(\Z_1 | (\bar{\V}_{\TTheta}^t)_1 =\V_i) & = \brho - (\bar{\bnu}_{\TTheta}^t)^\top 
  \left(\bSigma_{\V}^t\right)^{-1}\bar{\bnu}_{\TTheta}^t.
  \label{eq:cov_avg_Z1_given_V}
 \end{align}
 When $f^t=f^{*t}$, we can use the simplifications in Remark \ref{rmk:opt_ft_se_simplification} to obtain:
\begin{align*}
\cov(\Z_1 | (\bar{\V}_{\TTheta}^t)_1 = \V_i) = \brho - \bar{\bnu}_{\TTheta}^t \quad  \text{and}\quad    \E[\Z_1 | (\bar{\V}_{\TTheta}^t)_1 = \V_i] = \V_i.
\end{align*}
We now calculate the middle term $\E[\Z_1 | (\bar{\V}_{\TTheta}^t)_1  = \V_i, q(\Z_1, \bar{\Psi}_i, \bar{\varepsilon}) = u_i]$ of $g^{*t}_i$ in \eqref{eq:opt_ensemble_gt_lem}. Recalling from \eqref{eq:model_psi} that $ q(\Z_1, \bar{\Psi}_i, \bar{\varepsilon}) = (Z_1)_{\bar{\Psi}_i} +\bar{\varepsilon}$ where  $\Z_1\sim \N(\bzero, \brho)$, we have 
\begin{equation}
\label{eq:Z1q_corr}
\begin{split}
   & \E[\Z_1 q(\Z_1, \bar{\Psi}_i, \bar{\varepsilon})\ |\bar{\Psi}_i=\ell]
    =\E[\Z_1 (Z_1)_{\ell}] = \brho_{[:, \ell]} ,\\
    &\E[q(\Z_1, \bar{\Psi}_i, \bar{\varepsilon})^2\  |\bar{\Psi}_i=\ell]
=\E[(Z_1)_\ell^2 +\bar{\varepsilon}^2]
=\rho_{\ell,\ell} +\sigma^2,\\
&   \E[(\bar{\V}_{\TTheta}^t)_1q(\Z_1, \bar{\Psi}_i, \bar{\varepsilon})\ |\bar{\Psi}_i=\ell] 
= \E[(\Z \brho^{-1}\bar{\bnu}_{\TTheta}^t)_1 (Z_1)_{\ell} ] 
= (\bar{\bnu}_{\TTheta}^t)^\top\brho^{-1}\brho_{[:,\ell]}
=\left((\bar{\bnu}_{\TTheta}^t)_{[\ell,:]}\right)^\top,
\end{split}
\end{equation}
which implies that conditioned on $ \bar{\Psi}_i=\ell$, 
\begin{align}
\label{eq:vtheta_q_cov}
 &\begin{bmatrix}
    (\bar{\V}_{\TTheta}^t)_1\\
        q(\Z_1, \bar{\Psi}_i, \bar{\varepsilon})
    \end{bmatrix} 
\sim\N_{L+1}\left(\bzero, \bSigma^t\right), \; \text{where} \; \bSigma^t := 
\begin{bmatrix}
\bSigma_{\V}^t
&  \left((\bar{\bnu}_{\TTheta}^t)_{[\ell, :]}\right)^\top \\
(\bar{\bnu}_{\TTheta}^t)_{[\ell, :]} & \rho_{\ell, \ell} + \sigma^2
\end{bmatrix}.
\end{align}
Conditioned on $\Psi_i = \ell$ the random variables $(\Z_1, (\bar{\V}_{\TTheta}^t)_1, q_1(\Z_1, \bar{\Psi}_i, \bar{\varepsilon}))$ are jointly Gaussian with zero mean and covariance matrix determined by \eqref{eq:Z1_vtheta_cov}-\eqref{eq:vtheta_q_cov}.  Applying Lemmas \ref{lem:Gaussian_condition}--\ref{lem:gau_mixture_cond}  on these jointly Gaussian variables,   we  obtain: 
\begin{align}
&\E\Big[\Z_1 \mid  (\bar{\V}_{\TTheta}^t)_1  = \V_i, \, q(\Z_1, \bar{\Psi}_i, \bar{\varepsilon}) = u_i  \Big]  
 = \frac{\sum_{\ell=1}^L  \pi_{\bar{\Psi}_i}(\ell)\,\blambda^t(\V_i, u_i, \ell)  \, \mathcal{L}_1(\V_i, u_i| \ell)}{\sum_{\tilde{\ell}=1}^L \pi_{\bar{\Psi}_i}(\tilde{\ell}) \, \mathcal{L}_1(\V_i, u_i| \tilde{\ell})} ,
\label{eq:g_middle_term_expansion}
\end{align}
where $\pi_{\bar{\Psi}_i}: \ell \mapsto \sum_{\{\bpsi : \psi_i = \ell\}} \pi_{\bar{\bPsi}}(\bpsi)$ denotes the marginal probability of $\bar{\Psi}_i$, and
\begin{align}
\cL_1(\V_i, u_i | \ell) 
 &: =\phi \left(\begin{bmatrix} \V_i \\ u_i \end{bmatrix}; \0, \bSigma^t \right) \label{eq:L_1_term},\\
\blambda^t(\V_i, u_i, \ell) 
&:=
 \E[\Z_1 \mid  (\bar{\V}_{\TTheta}^t)_1  = \V_i, \,  q_1(\Z_1, \bar{\Psi}_i, \bar{\varepsilon}) = u_i, \, \bar{\Psi}_i = \ell] \nonumber\\
 & 
 \;= \begin{bmatrix}
    \bar{\bnu}_{\TTheta}^t & \brho_{[:, \ell]}
\end{bmatrix} [\bSigma^t]^{-1}
\begin{bmatrix}
    \V_i \\ u_i
\end{bmatrix} \in \reals^L. \label{eq:lambda_term}
\end{align}
\subsubsection{Computation of   $g^{*t}$ for Logistic Regression}\label{sec:full_computation_gt_logistic} 
The only term in the expression for $g^{*t}$ in \eqref{eq:opt_ensemble_gt_lem} that differs from the linear regression case, and therefore needs to be re-computed, is $\E[\Z_1 \mid  (\bar{\V}_{\TTheta}^t)_1  = \V_i, \, q(\Z_1, \bar{\Psi}_i, \bar{\varepsilon}) = u_i ] $. For $k\in [L]$, we  can write
\begin{align*}
&\E\Big[(Z_1)_k \mid  (\bar{\V}_{\TTheta}^t)_1  = \V_i, \, q(\Z_1, \bar{\Psi}_i, \bar{\varepsilon}) = u_i \Big]  
 = \frac{\sum_{\ell=1}^L \pi_{\bar{\Psi}_i}(\ell)\, \tlambda^t(\V_i, u_i, \ell)  }{\sum_{\tilde{\ell}=1}^L \pi_{\bar{\Psi}_i}(\tilde{\ell}) \, \mathcal{L}_1(\V_i, u_i| \tilde{\ell})} \,,
\end{align*}
where
\begin{align}
&\cL_1(\V_i,u_i|\ell):=\P\left((\bar{\V}_{\TTheta}^t)_1 = \V_i, q(\Z_1, \bar{\Psi}_i, \bar{\varepsilon})=u_i\mid \bar{\Psi}_i=\ell\right), \label{eq:L1_logistic}\\
 & \tlambda^t(\V_i, u_i, \ell)
:= \int_{-\infty}^{\infty} z \,  \P\left((Z_1)_k=z, (\bar{\V}_{\TTheta}^t)_1  = \V_i, \, q(\Z_1, \bar{\Psi}_i, \bar{\varepsilon}) = u_i \mid \bar{\Psi}_i=\ell\right) \de z.\label{eq:lambda_logistic}
\end{align}
In the rest of this subsection, we provide expressions for $\cL_1(\V_i, u_i|\ell)$ and  $\tlambda^t(\V_i, u_i, \ell)$. As preliminaries, we adopt the following shorthand  for $\ell\in[L]$:
\begin{align}\label{eq:shorthand_mu_sigma_l}
     \mu_\ell\equiv \mu_\ell(\V_i):=\left(\E[\Z_1| (\bar{\V}_{\TTheta}^t)_1=\V_i]\right)_\ell\,,\quad 
     \sigma_\ell\equiv \sigma_\ell(\V_i) := \sqrt{\left[\cov(\Z_1| (\bar{\V}_{\TTheta}^t)_1 = \V_i)\right]_{\ell,\ell}}\,.
 \end{align} 
where the quantities on the right are defined in \eqref{eq:cov_avg_Z1_given_V}.  
We also recall from \eqref{eq:cov_avg_Z1_given_V}  that conditioned on $(\bar{\V}_{\TTheta}^t)_1=\V_i$, 
\begin{align*}
   &\begin{bmatrix}
       ( Z_1)_k\\
       ( Z_1)_\ell
   \end{bmatrix}
   \sim \normal\left(\begin{bmatrix}
       \mu_k\\\mu_\ell
   \end{bmatrix}, \cov\Big((Z_1)_k, (Z_1)_\ell\Big)\right),
   \end{align*}
   where 
   \begin{align*}
  &\cov\Big((Z_1)_k, (Z_1)_\ell\Big)
   = \begin{bmatrix}
      \rho_{kk}& \rho_{k\ell}\\
      \rho_{k\ell} & \rho_{\ell\ell}
  \end{bmatrix} -
  \begin{bmatrix}
      \bar{\bnu}_{\TTheta[k,:]}^t\\
        \bar{\bnu}_{\TTheta[\ell,:]}^t
  \end{bmatrix}
  \left(\bSigma_{\V}^t\right)^{-1}
   \begin{bmatrix}
      \bar{\bnu}_{\TTheta[k,:]}^t\\
        \bar{\bnu}_{\TTheta[\ell,:]}^t
  \end{bmatrix}^\top=:
  \begin{bmatrix}
       \sigma_k^2& \lambda_{k\ell}\\
       \lambda_{k\ell}& \sigma_\ell^2
   \end{bmatrix}.
\end{align*}
Applying Lemma \ref{lem:Gaussian_condition}, this implies that conditioned on $(\bar{\V}_{\TTheta}^t)_1=\V_i$,
\begin{align*}
    (Z_1)_k\mid (Z_1)_\ell \sim \normal\left(\mu_{k|\ell}, \sigma^2_{k|\ell}\right),
\end{align*}
where 
\begin{align}\label{eq:params_k_given_l}
    \mu_{k|\ell} = \mu_k + \frac{\lambda_{k\ell}}{\sigma_\ell^2}\big( (Z_1)_\ell -\mu_\ell\big)\,, \qquad 
    \sigma^2_{k|\ell}=\sigma^2_k - \frac{\lambda_{kl}^2}{\sigma_\ell^2}\,.
\end{align}
We will also use Lemmas \ref{lem:approx_logistic_by_gau}--\ref{lem:integral_Gaussian} below  to avoid intractable integrals in the sequel.
\begin{lemma}[Logistic approximation, Section 4.5.2 in \cite{bishop2006pattern}]\label{lem:approx_logistic_by_gau}

The logistic function $\zeta'(z) $ in \eqref{eq:logistic-chgpt-model} can be approximated by a standard Gaussian CDF as:
    \begin{align}\label{eq:approx_logistic}
        \zeta'(z) \approx \Phi\left(\gamma z\right), \quad \text{where}\quad \gamma=\sqrt{\frac{\pi}{8}}\;.
    \end{align}
\end{lemma}
\begin{lemma}[Useful Gaussian Identities, Table 1 in \cite{owen1980gaussian}]\label{lem:integral_Gaussian}It holds that 
\begin{align}
  &  \int_{-\infty}^{\infty}\Phi(a+bx)\phi(x)\de x = \Phi\left(\frac{a}{\sqrt{1+b^2}}\right), \\ 
   & \int_{-\infty}^{\infty}x\Phi(a+bx)\phi(x)\de x = \frac{b}{\sqrt{1+b^2}}\,\phi\left(\frac{a}{\sqrt{1+b^2}}\right).
\end{align}
\end{lemma}
We will compute $\tlambda^t(\V_i, u_i, \ell)$ in  the  $k=\ell$ and $k\neq \ell$ cases separately.
\paragraph{Case 1: $k=\ell$.}
We can rewrite the probability inside the integral in \eqref{eq:lambda_logistic} as follows:
\begin{align}
&    \P\left((Z_1)_\ell=z, (\bar{\V}_{\TTheta}^t)_1  = \V_i, \, q(\Z_1, \bar{\Psi}_i, \bar{\varepsilon}) = u_i \mid \bar{\Psi}_i=\ell\right) \nonumber\\
& = \P\left((\bar{\V}_{\TTheta}^t)_1  = \V_i\right)\P\left((Z_1)_\ell=z\mid (\bar{\V}_{\TTheta}^t)_1  = \V_i\right) \P\left(q(\Z_1, \bar{\Psi}_i, \bar{\varepsilon}) = u_i\mid (Z_1)_\ell=z, \bar{\Psi}_i=\ell\right)\nonumber\\
& = \phi(\V_i; \bzero, \bSigma_{\V}^t)
\phi(z; \mu_\ell, \sigma_\ell^2)
\left[\ind\{u_i=1\}\zeta'(z) + \ind\{u_i=0\}(1-\zeta'(z))\right],\label{eq:joint_pdf_logistic_k=l}
\end{align}
where we recalled from \eqref{eq:logistic-chgpt-model} that 
$ \P\left(q(\Z_1, \bar{\Psi}_i, \bar{\varepsilon})=1\mid (Z_1)_\ell=z, \bar{\Psi}_i=\ell\right)=\zeta'(z).$
Substituting \eqref{eq:joint_pdf_logistic_k=l} into  \eqref{eq:lambda_logistic}, and applying Lemmas \ref{lem:approx_logistic_by_gau}--\ref{lem:integral_Gaussian} yields
\begin{align*}
   \tlambda^t(\V_i, u_i, \ell)  \approx \phi(\V_i; \bzero, \bSigma_{\V}^t)& \left\lbrace\ind\{u_i=1\}\left[\tau_\ell\sigma_\ell^2\phi(\tau_\ell\mu_\ell) + \mu_\ell \Phi(\tau_\ell \mu_\ell)\right] + 
   \nonumber\right.\nonumber\\
    &\;\;  \left.\ind\{u_i=0\}\left[\mu_\ell- \tau_\ell\sigma_\ell^2\phi(\tau_\ell\mu_\ell) - \mu_\ell \Phi(\tau_\ell \mu_\ell)\right]\right\rbrace,
\end{align*}
where $\tau_\ell=\gamma/\sqrt{1+\gamma^2\sigma_\ell^2}$ with $\gamma=\sqrt{\pi/8}$ as defined in \eqref{eq:approx_logistic}. 
\paragraph{Case 2: $k\ne \ell$.} We rewrite the density inside the integral in \eqref{eq:lambda_logistic} as follows:
\begin{align}
&    \P\left((Z_1)_k=z, (\bar{\V}_{\TTheta}^t)_1  = \V_i, \, q(\Z_1, \bar{\Psi}_i, \bar{\varepsilon}) = u_i \mid \bar{\Psi}_i=\ell\right) \nonumber\\
& =\int_{-\infty}^{\infty} \P\left((\bar{\V}_{\TTheta}^t)_1  = \V_i\right)\P\left((Z_1)_\ell=s\mid (\bar{\V}_{\TTheta}^t)_1  = \V_i\right) \P\left(q(\Z_1, \bar{\Psi}_i, \bar{\varepsilon}) = u_i\mid (Z_1)_\ell=s, \bar{\Psi}_i=\ell\right)\nonumber\\
& \quad \;\P\left((Z_1)_k = z\mid (Z_1)_\ell=s, (\bar{\V}_{\TTheta}^t)_1=\V_i\right)\de s\nonumber\\
& = \phi(\V_i; \bzero, \bSigma_{\V}^t)\int_{-\infty}^{\infty}
\phi(s; \mu_\ell, \sigma_\ell^2)
\left[\ind\{u_i=1\}\zeta'(s) + \ind\{u_i=0\}(1-\zeta'(s))\right] \phi\left(z; \mu_{k|\ell}, \sigma_{k|\ell}^2\right)\de s.\label{eq:joint_pdf_logistic_k_ne_l}
\end{align}
Substituting \eqref{eq:joint_pdf_logistic_k_ne_l} into  \eqref{eq:lambda_logistic} and applying Lemmas \ref{lem:approx_logistic_by_gau}--\ref{lem:integral_Gaussian} yields
\begin{align*}
   \tlambda^t(\V_i, u_i, \ell) &\approx \phi(\V_i; \bzero, \bSigma_{\V}^t)\left\lbrace\ind\{u_i=1\} \left[ \lambda_{k\ell} \tau_\ell\phi(\tau_\ell\mu_\ell) + \mu_k\Phi(\tau_\ell \mu_\ell)\right] + \right.\nonumber\\
& \qquad \quad \qquad \qquad  \;\left. \ind\{u_i=0\}\left[\mu_k - \lambda_{k\ell}  \tau_\ell\phi(\tau_\ell \mu_\ell)-\mu_k \Phi(\tau_\ell \mu_\ell)\right]\right\rbrace.
\end{align*}
Finally, to compute $\cL_1(\V_i, u_i|\ell)$, we  sum over $z$ in \eqref{eq:joint_pdf_logistic_k=l} and apply Lemma \ref{lem:approx_logistic_by_gau} to obtain
\begin{align}
    \cL_1(\V_i, u_i|\ell) \approx \phi(\V_i; \bzero, \bSigma_{\V}^t)\left\lbrace\ind\{u_1=1\}  \Phi(\tau_\ell\mu_\ell)+ \ind\{u_1=0\}[1- \Phi(\tau_\ell\mu_\ell)]\right\rbrace. \label{eq:logistic_approx_likelihood}
\end{align}

\subsubsection{Computation of $g^{*t}$ for Rectified Linear Regression}\label{sec:full_computation_gt_ReLU}

Similar to Appendix \ref{sec:full_computation_gt_logistic}, the only term in the expressions for  $g^{*t}$ that differs from linear regression, and therefore needs to be re-computed, is $\E[\Z_1| (\bar{\V}_{\TTheta}^t)_1=\V_i, q(\Z_1, \bar{\Psi}_i, \bar{\varepsilon})=u_i]$. We will express this term using a binary auxiliary variable 
 \begin{align*}
 \Omega:=\ind\{ (Z_1)_\ell \ge 0\}    ,
 \end{align*}
 which indicates whether the data are generated through the sloped part or flat part of the rectified linear unit, conditioned on $\bar{\Psi}_i =\ell$. In many parts of the derivation, variables depending on $\Omega$  exhibit  truncated Gaussian distributions. We begin with some useful results on these distributions.
\begin{lemma}[Useful truncated Gaussian identities]\label{lem:trun_gau}
Consider the   truncated Gaussian variables $X^+$ and $X^-$ with probability densities 
\begin{align*}
    \phi^+(x;\mu, \sigma^2) &:= \ind\{x\ge 0 \}\frac{\phi(x; \mu,\sigma^2)}{\Phi(\mu/\sigma)}\,,  \qquad
     \phi^-(x;\mu, \sigma^2) := \ind\{x < 0 \}\frac{\phi(x; \mu,\sigma^2)}{\Phi(-\mu/\sigma)}\,.
\end{align*}
It holds that 
\begin{align*}
    \E[X^+]&=\mu +\sigma \frac{\phi(0; \mu, \sigma^2)}{\Phi(\mu/\sigma)}\,,\qquad
    \E[X^-] = \mu -\sigma \frac{\phi(0; \mu, \sigma^2)}{ \Phi(-\mu/\sigma)}\,,
\end{align*}
and it follows that:
    \begin{align*}
        \phi^+(x; \mu, \sigma^2) \phi(z; x, \tau^2)
        =\phi(z; \mu, \sigma^2+\tau^2)
        \frac{\Phi(\mu_*/\sigma_*)}{\Phi(\mu/\sigma)}\phi^+(x; \mu_*, \sigma_*^2),
    \end{align*}
    where 
$\mu_* = \frac{z\sigma^2 + \mu\tau^2}{\sigma^2+\tau^2}$ and   
    $\sigma_*^2 = \frac{\sigma^2\tau^2}{\sigma^2+\tau^2}$.
\end{lemma}
 Equipped with the definitions and lemmas above,  we can write for $k\in [L]$,
 \begin{align*}
     &\E\left[(Z_1)_k \mid  (\bar{\V}_{\TTheta}^t)_1= \V_i, q(\Z_1, \bar{\Psi}_i, \bar{\varepsilon})=u_i\right] = \frac{\sum_{\ell=1}^L \sum_{\omega\in \{0,1\}} \pi_{\bar{\Psi}_i}(\ell) \,\tlambda^t(k, \V_i, u_i, \ell, \omega) }{\sum_{\tilde{\ell}=1}^L \sum_{\tilde{\omega} \in \{0,1\}}\pi_{\bar{\Psi}_i}(\tilde{\ell}) \, \mathcal{L}_1(\V_i, u_i, \tilde{\omega}| \tilde{\ell})} ,
     \end{align*}
     where
    \begin{align}
      &\cL_1(\V_i, u_i, \omega|\ell) 
      := \P\left((\bar{\V}_{\TTheta}^t)_1 = \V_i, q(\Z_1, \bar{\Psi}_i, \bar{\varepsilon})=u_i, \Omega=\omega \mid \bar{\Psi}_i=\ell\right),  \nonumber \\
      &\tlambda^t(k, \V_i, u_i, \ell, \omega) 
      \;=  
      \int_{-\infty}^{\infty} z \P\left((Z_1)_k = z, (\bar{\V}_{\TTheta}^t)_1 = \V_i, q(\Z_1, \bar{\Psi}_i, \bar{\varepsilon})=u_i,   \Omega = \omega \mid  \bar{\Psi}_i=\ell\right)\de z\label{eq:expand_lambda_ReLU}.
      \end{align}
 In the rest of this subsection, we provide the expression for the terms $\cL_1(\V_i, u_i, \omega|\ell) $ and $\tlambda^t(k, \V_i, u_i, \ell, \omega)$ for $\omega=0$ or 1, respectively. 
\paragraph{Case 1: $\Omega=0$.} This implies  that $(Z_1)_\ell<0$, and  $q(\Z_1, \bar{\Psi}_i=\ell, \bar{\varepsilon}) = \max\{(Z_1)_\ell, 0\}+ \bar{\varepsilon}=\bar{\varepsilon}$. Using the law of conditional probability, we have
 \begin{align}
      \cL_1(\V_i, u_i, \Omega=0|\ell) &=\P\left((\bar{\V}_{\TTheta}^t)_1=\V_i\right)\P\left(\Omega =0\mid (\bar{\V}_{\TTheta}^t)_1=\V_i, \bar{\Psi}_i=\ell\right) \cdot \nonumber\\
    & \quad\; \P\left(q(\Z_1, \bar{\Psi}_i, \bar{\varepsilon})=u_i\mid (\bar{\V}_{\TTheta}^t)_1=\V_i, \Omega=0, \bar{\Psi}_i=\ell \right)\nonumber\\
    & = \phi\left(\V_i; \bzero,   \bSigma_{\V}^t\right)
    \Phi\left(-\mu_\ell/\sigma_\ell\right)
    \phi\left(u_i; 0, \sigma^2\right),
    \label{eq:L1_ReLU_w=0}
 \end{align}
 where $\bSigma_{\V}^t$ is defined in \eqref{eq:Z1_vtheta_cov}, and $\mu_\ell, \sigma_\ell$ in \eqref{eq:shorthand_mu_sigma_l}.
To evaluate $\tlambda^t(k, \V_i, u_i, \ell, \omega)$, we consider the  $k=\ell$ and  $k\neq \ell$ cases separately.

\paragraph{Case 1a: $\Omega=0, k=\ell$.}
The density inside the integral of \eqref{eq:expand_lambda_ReLU} can be rewritten as
\begin{align}
  &  \P\left((Z_1)_\ell = z, (\bar{\V}_{\TTheta}^t)_1 = \V_i, q(\Z_1, \bar{\Psi}_i, \bar{\varepsilon})=u_i,   \Omega = 0 \mid  \bar{\Psi}_i=\ell\right)\nonumber\\
  & = \P\left((\bar{\V}_{\TTheta}^t)_1=\V_i\right)\P\left(\Omega =0\mid (\bar{\V}_{\TTheta}^t)_1=\V_i, \bar{\Psi}_i=\ell\right) \cdot \nonumber\\
    & \quad\; \P\left( (Z_1)_\ell = z\mid (\bar{\V}_{\TTheta}^t)_1=\V_i, \Omega=0, \bar{\Psi}_i=\ell \right)\cdot \nonumber\\
    & \quad \; \P\left(q(\Z_1, \bar{\Psi}_i, \bar{\varepsilon})=u_i \mid  (Z_1)_\ell = z, (\bar{\V}_{\TTheta}^t)_1=\V_i,\Omega = 0, \bar{\Psi}_i=\ell\right)\nonumber\\
    & =\phi\left(\V_i; \bzero,   \bSigma_{\V}^t\right)
    \Phi\left(-\mu_\ell/\sigma_\ell\right)
    \phi^-(z; \mu_\ell, \sigma^2_{\ell})
    \phi(u_i; 0, \sigma^2).\label{eq:expand_joint_pdf_ReLU}
\end{align}
Substituting \eqref{eq:expand_joint_pdf_ReLU} into \eqref{eq:expand_lambda_ReLU} and applying Lemma \ref{lem:trun_gau} yields
\begin{align*}
    \tlambda^t(k=\ell, \V_i, u_i, \ell, \Omega=0)=
    \phi\left(\V_i; \bzero,   \bSigma_{\V}^t\right)
    \Phi\left(-\mu_\ell/\sigma_\ell\right)
    \left[ \mu_\ell -\sigma_\ell \frac{\phi(0; \mu_\ell, \sigma_\ell^2)}{\Phi(-\mu_\ell/ \sigma_\ell)}\right]
    \phi(u_i; 0, \sigma^2). 
\end{align*}
\paragraph{Case 1b: $\Omega=0, k\neq \ell$.}
Similar to \eqref{eq:expand_joint_pdf_ReLU},  the density inside the integral of \eqref{eq:expand_lambda_ReLU} can be rewritten as
\begin{align}\label{eq:joint_prob_w=0_k_neq_l}
     &\P\left((Z_1)_k = z, (\bar{\V}_{\TTheta}^t)_1 = \V_i, q(\Z_1, \bar{\Psi}_i, \bar{\varepsilon})=u_i,   \Omega = 0 \mid  \bar{\Psi}_i=\ell\right)\nonumber\\
     & = \phi\left(\V_i; \bzero,   \bSigma_{\V}^t\right)
    \Phi\left(-\mu_\ell/\sigma_\ell\right)
   \P\left( (Z_1)_k = z\mid (\bar{\V}_{\TTheta}^t)_1=\V_i, \Omega=0, \bar{\Psi}_i=\ell \right)
    \phi(u_i; 0, \sigma^2).
 \end{align}
To evaluate the second last term  in \eqref{eq:joint_prob_w=0_k_neq_l},we apply the law of total probability to obtain 
\begin{align}
   & \P\left( (Z_1)_k = z\mid (\bar{\V}_{\TTheta}^t)_1=\V_i, \Omega=0, \bar{\Psi}_i=\ell \right)\nonumber\\
   &=\int\P\left( (Z_1)_\ell=s| (\bar{\V}_{\TTheta}^t)_1=\V_i, (Z_1)_\ell <0 \right)    
    \P\left((Z_1)_k = z| (Z_1)_\ell=s, (\bar{\V}_{\TTheta}^t)_1=\V_i, (Z_1)_\ell <0 \right)\de s\nonumber\\
    & = \int \phi^-(s; \mu_\ell, \sigma^2_{\ell})\phi\left(z; \mu_{k|\ell}, \sigma^2_{k|\ell}\right)
    \de s,  \label{eq:cond_prob_k_neq_l_w=0}
\end{align}
where  $\mu_{k|\ell}, \sigma^2_{k|\ell}$ are given by \eqref{eq:params_k_given_l}. Substituting \eqref{eq:cond_prob_k_neq_l_w=0} into \eqref{eq:joint_prob_w=0_k_neq_l} then  into 
\eqref{eq:expand_lambda_ReLU} and applying Lemma \ref{lem:trun_gau} gives
\begin{align*}
    &\tlambda^t(k\neq \ell, \V_i, u_i, \ell, \Omega=0)\nonumber\\
    & =
    \phi\left(\V_i; \bzero,   \bSigma_{\V}^t\right)
    \Phi\left(-\mu_\ell/\sigma_\ell\right) 
    \phi(u_i; 0, \sigma^2)
    \int \phi^-(s; \mu_\ell, \sigma^2_\ell) \mu_{k|\ell} \, \de s\nonumber\\
    & =
    \phi\left(\V_i; \bzero,  \bSigma_{\V}^t\right)
    \Phi\left(-\mu_\ell/\sigma_\ell\right)
    \phi(u_i; 0, \sigma^2)
\left\lbrace\frac{\lambda_{k\ell}}{\sigma_\ell^2}\left[\mu_\ell-\sigma_\ell \frac{\phi(0; \mu_\ell, \sigma_\ell^2)}{\Phi(-\mu_\ell/\sigma_\ell)}\right] + \mu_k - \frac{\lambda_{k\ell}}{\sigma_\ell^2}\mu_\ell \right\rbrace.
\end{align*}
\paragraph{Case 2: $\Omega=1$.} 
In this case, $(Z_1)_\ell \ge 0$, and $q(\Z_1, \bar{\Psi}_i=\ell, \bar{\varepsilon}) = (Z_1)_\ell + \bar{\varepsilon}$. Applying Lemma \ref{lem:trun_gau} yields 
\begin{align}
  &\P\left(  q(\Z_1, \bar{\Psi}_i, \bar{\varepsilon})=u_i \mid  (\bar{\V}_{\TTheta}^t)_1=\V_i, \Omega=1, \bar{\Psi}_i=\ell\right)\nonumber\\
  &=\P\left((Z_1)_\ell + \bar{\varepsilon} = u_i\mid (\bar{\V}_{\TTheta}^t)_1=\V_i,(Z_1)_\ell \ge 0 \right)
   = \phi(u_i; \mu_\ell, \sigma^2_\ell + \sigma^2)\frac{\Phi(\mu_{\ell*}/\sigma_{\ell*})}{\Phi(\mu_\ell/\sigma_\ell)},\label{eq:P_q_ReLU}
\end{align}
where 
\begin{align}
 \mu_{\ell*} = \frac{u_i\sigma_\ell^2 + \mu_\ell \sigma^2}{\sigma^2_\ell + \sigma^2}, \qquad 
    \sigma_{\ell*}^2 = \frac{\sigma_\ell^2\sigma^2}{\sigma_\ell^2+\sigma^2 }
   .\label{eq:P_q_ReLU_params}
\end{align}
Using \eqref{eq:P_q_ReLU}--\eqref{eq:P_q_ReLU_params} and the law of conditional probability similarly to \eqref{eq:L1_ReLU_w=0}, we obtain 
\begin{align}
       \cL_1(\V_i, u_i, \Omega=1|\ell)
    & = \phi\left(\V_i; \bzero,   \bSigma_{\V}^t\right)
    \phi\left(u_i; \mu_\ell, \sigma_\ell^2+ \sigma^2\right)  \Phi(\mu_{\ell*}/\sigma_{\ell*})  .
    \label{eq:L1_ReLU_w=1}
\end{align}
Next, we evaluate $ \lambda^t(k, \V_i, u_i, \ell, \Omega=1)$ for $k=\ell$ and $k\neq \ell$ separately.
\paragraph{Case 2a: $\Omega=1, k=\ell$.}
Analogously to \eqref{eq:expand_joint_pdf_ReLU}, we apply the law of conditional probability and Lemma \ref{lem:trun_gau}  to obtain
\begin{align}
    &\P\left((Z_1)_\ell = z, (\bar{\V}_{\TTheta}^t)_1 = \V_i, q(\Z_1, \bar{\Psi}_i, \bar{\varepsilon})=u_i,   \Omega = 1\mid  \bar{\Psi}_i=\ell\right)\nonumber\\
    & = \phi\left(\V_i; \bzero,   \bSigma_{\V}^t\right)
    \Phi\left(\mu_\ell/\sigma_\ell\right)  \phi^+\left(z; \mu_\ell, \sigma_\ell^2\right)
    \phi(u_i; z, \sigma^2)\nonumber\\
    & = \phi\left(\V_i; \bzero,   \bSigma_{\V}^t\right)
     \phi\left(u_i; \mu_\ell, \sigma_\ell^2+\sigma^2\right)  \Phi(\mu_{\ell*}/ \sigma_{\ell*})  \phi^+(z; \mu_{\ell*}, \sigma_{\ell*}^2).
    \label{eq:expand_joint_pdf_ReLU_w1} 
\end{align}
Substituting \eqref{eq:expand_joint_pdf_ReLU_w1} 
 into \eqref{eq:expand_lambda_ReLU} and 
applying  Lemma \ref{lem:trun_gau} yields
\begin{align}
    &\tlambda^t(k=\ell, \V_i, u_i, \ell, \Omega=1)\nonumber\\
    & = 
    \phi\left(\V_i; \bzero,   \bSigma_{\V}^t\right)
    \phi\left(u_i; \mu_\ell, \sigma_\ell^2+\sigma^2\right)  \Phi(\mu_{\ell*}/ \sigma_{\ell*}) \left[\mu_{\ell*} + \sigma_{\ell*} \frac{\phi(0; \mu_{\ell*}, \sigma_{\ell*})}{\Phi(\mu_{\ell*}/ \sigma_{\ell*})}\right].
\end{align}
\paragraph{Case 2b: $\Omega=1, k\neq \ell$.} Following the same steps as \eqref{eq:joint_prob_w=0_k_neq_l}--\eqref{eq:cond_prob_k_neq_l_w=0} and applying Lemma \ref{lem:trun_gau}, we can write
\begin{align}
    &\P\left((Z_1)_k = z, (\bar{\V}_{\TTheta}^t)_1 = \V_i, q(\Z_1, \bar{\Psi}_i, \bar{\varepsilon})=u_i,   \Omega = 1\mid  \bar{\Psi}_i=\ell\right)\nonumber\\
    & = \phi\left(\V_i; \bzero,   \bSigma_{\V}^t\right)
    \Phi\left(\mu_\ell/\sigma_\ell\right)  \int\phi^+\left(s; \mu_\ell, \sigma_\ell^2\right) \phi\left(z; \mu_{k|\ell}, \sigma^2_{k|\ell}\right)\phi(u_i; s, \sigma^2)\de s\nonumber\\
    & = \phi\left(\V_i; \bzero,   \bSigma_{\V}^t\right)
       \phi(u_i; \mu_\ell, \sigma_\ell^2+ \sigma^2)  \Phi(\mu_{\ell*}/ \sigma_{\ell*}) \int\phi^+(s; \mu_{\ell*}, \sigma_{\ell*}^2)
    \phi\left(z; \mu_{k|\ell}, \sigma^2_{k|\ell}\right)
    \de s\label{eq:cond_prob_k_neq_l_w=1} 
\end{align}
where $\mu_{k|\ell}$ and $\sigma^2_{k|\ell}$ are defined in \eqref{eq:params_k_given_l}, while $\mu_{\ell*}$ and $\sigma_{\ell*}^2$ are defined in \eqref{eq:P_q_ReLU_params}. 
Substituting \eqref{eq:cond_prob_k_neq_l_w=1} into 
\eqref{eq:expand_lambda_ReLU} and applying Lemma \ref{lem:trun_gau} gives
\begin{align}
    &\tlambda^t(k\neq \ell, \V_i, u_i, \ell, \Omega=1)\nonumber\\
    & =
    \phi\left(\V_i; \bzero,   \bSigma_{\V}^t\right) \phi(u_i; \mu_\ell, \sigma_\ell^2+ \sigma^2)  \Phi(\mu_{\ell*}/ \sigma_{\ell*}) 
    \int \phi^+(s; \mu_{\ell*}, \sigma_{\ell*}^2)
    \int z\phi\left(z; \mu_{k|\ell}, \sigma^2_{k|\ell}\right)
   \de z \de s \nonumber\\
   & =
   \phi\left(\V_i; \bzero,   \bSigma_{\V}^t\right) \phi(u_i; \mu_\ell, \sigma_\ell^2+ \sigma^2)  \Phi(\mu_{\ell*}/ \sigma_{\ell*}) 
    \int \phi^+(s; \mu_{\ell*}, \sigma_{\ell*}^2)
      \mu_{k|\ell} \,  \de s \nonumber\\
   & =
   \phi\left(\V_i; \bzero,   \bSigma_{\V}^t\right) \phi(u_i; \mu_\ell, \sigma_\ell^2+ \sigma^2)  \Phi(\mu_{\ell*}/ \sigma_{\ell*}) \nonumber \\
   &\hspace{6cm} \cdot \left\lbrace\frac{\lambda_{k\ell}}{\sigma^2_\ell}\left[\mu_{\ell*} + \sigma_{\ell*}\frac{\phi(0; \mu_{\ell*}, \sigma^2_{\ell*})}{\Phi(\mu_{\ell*}/ \sigma_{\ell*})}\right] +  \mu_k - \frac{\lambda_{k\ell}}{\sigma^2_\ell}\mu_\ell  \right\rbrace.
\end{align}

\subsection{Computation of the Likelihood $\cL$} \label{sec:likelihood_computation}
Recall from \eqref{eq:posterior_density_2} that $\cL$ represents the likelihood of $({\V}_{\TTheta}^t, q(\Z, \bar{\bPsi}, \bar{\bvarepsilon}))$ given $\bar{\bPsi} = \bpsi$, where $\bnu_{\TTheta}^t, \bkappa_{\TTheta}^t$ are computed as in \eqref{eq:nu_B_SE}--\eqref{eq:kappa_theta_SE} with $\bpsi$ substituted for $\bPsi$. Under the assumption that $\bar{\varepsilon}_i \distas{i.i.d} \N(0, \sigma^2)$ for some $\sigma > 0$, the likelihood can be computed in closed form in the case of linear regression as follows: 
\begin{align*}
   \cL(\V, \u | \bpsi) 
 &:= \prod_{i = 1}^n \phi \left(\begin{bmatrix} \V_i \\ u_i \end{bmatrix}; \0, \bSigma_{\psi_i}^t \right), \;\; \text{ where } \bSigma_{\psi_i}^t := 
\begin{bmatrix}
({\bnu}_{\TTheta}^t)^\top \brho^{-1} {\bnu}_{\TTheta}^t + {\bkappa}_{\TTheta}^{t, t}
&  \left(({\bnu}_{\TTheta}^t)_{[\psi_i, :]}\right)^\top \\
({\bnu}_{\TTheta}^t)_{[\psi_i, :]} & \rho_{\psi_i, \psi_i} + \sigma^2
\end{bmatrix}.
\end{align*}
In the cases of logistic regression and rectified linear regression with $\bar{\varepsilon}_i \distas{i.i.d} \N(0, \sigma^2)$, the likelihood $\cL$ can be computed by letting $\bSigma_{\V}^t := ({\bnu}_{\TTheta}^t)^\top \brho^{-1} {\bnu}_{\TTheta}^t + {\bkappa}_{\TTheta}^{t, t}$, 
\begin{align*}
     \mu_\ell\equiv \mu_\ell(\V_i):=\left(\E[\Z_1| ({\V}_{\TTheta}^t)_1=\V_i]\right)_\ell\,,\qquad 
     \sigma_\ell\equiv \sigma_\ell(\V_i) := \sqrt{\left[\cov(\Z_1| ({\V}_{\TTheta}^t)_1 = \V_i)\right]_{\ell,\ell}}\,,
 \end{align*} 
in \eqref{eq:logistic_approx_likelihood}, \eqref{eq:L1_ReLU_w=0}, \eqref{eq:L1_ReLU_w=1}.
\section{Proof of Propositions \ref{prop:hausdorff_asymptotics} and \ref{prop:pointwise_posterior}} \label{sec:proposition_proofs}

\subsection{Proof of Proposition \ref{prop:hausdorff_asymptotics}} \label{sec:proof_hausdorff}
\begin{proof}
Recall that a signal configuration vector $\bPsi \in \mathcal{X}$ is a vector that is piece-wise constant with respect to its indices $i \in [n]$, with jumps of size $1$ occurring at the indices $\{ \eta_\ell \}_{\ell = 1}^{L^* - 1}$. Without loss of generality, we assume $\bPsi$ is also monotone (otherwise, any non-distinct signal can be treated as a new signal having perfect correlation with the first). Recall that $\eeta \in [n]^{L^* - 1}$ is the change point vector corresponding to $\bPsi$, that is, $\eeta = U^{-1}(\bPsi)$, and that for $i \in [n]$ the distance of $\eta_i$ to a change point estimate $\hat{\eeta} \in [n]^{L - 1}$ is defined as $d(\eta_i, \{\hat{\eta}_j\}_{j = 1}^{L - 1}) := \min_{\hat{\eta}_j \in \{\hat{\eta}_j\}_{j = 1}^{L - 1}} \|\eta_i - \hat{\eta}_j\|_2$. We then have that $d(\eta_i, \{\hat{\eta}_j\}_{j = 1}^{L - 1}) \leq \| U(\eeta) - U(\hat{\eeta})\|^2_F$ for all $i \in [L^* - 1]$, and similarly $d(\{\eta_i\}_{i=1}^{L^* - 1}, \hat{\eta}_j) \leq \| U(\eeta) - U(\hat{\eeta})\|^2_F$ for all $j \in [L - 1]$. This implies that $d_H(\eeta, \hat{\eeta}) \leq \|U(\eeta) - U(\hat{\eeta})\|_F^2$.

We first prove that $(\TTheta^t, \X\B) \mapsto d_H(\eeta, \hat{\eeta}(\TTheta^t, q(\X\B, \bPsi, \bvarepsilon)) / n$ is uniformly pseudo-Lipschitz. Consider two sets of inputs, $\A^{(1)} := \begin{bmatrix} (\TTheta^{t})^{(1)} & (\X\B)^{(1)} \end{bmatrix} \in \reals^{n \times 2L}$ and $\A^{(2)} := \begin{bmatrix} (\TTheta^{t})^{(2)} & (\X\B)^{(2)}) \end{bmatrix} \in \reals^{n \times 2L}$, and let $\hat{\eeta}(\A^{(1)}) := \hat{\eeta}\left((\TTheta^{t})^{(1)}, q((\X\B)^{(1)}, \bPsi, \bvarepsilon)\right)$, $\hat{\eeta}(\A^{(2)}) := \hat{\eeta}\left((\TTheta^{t})^{(2)}, q((\X\B)^{(2)}, \bPsi, \bvarepsilon)\right)$. We then have that: 
\begin{align}
    &\frac{1}{n} \left| d_{H}(\eeta, \hat{\eeta}(\A^{(1)})) - d_{H}(\eeta, \hat{\eeta}(\A^{(2)})) \right| \notag \\
    &\leq \frac{1}{n} d_{H}(\hat{\eeta}(\A^{(1)}), \hat{\eeta}(\A^{(2)})) \label{eq:rev_tri_ineq} \\
    &\leq \frac{1}{n} \|U(\hat{\eeta}(\A^{(1)})) - U(\hat{\eeta}(\A^{(2)}))\|_F^2 \label{eq:hausdorff_bound}\\
    &\leq L \cdot \frac{1}{\sqrt{n}} \|U(\hat{\eeta}(\A^{(1)})) - U(\hat{\eeta}(\A^{(2)}))\|_F \label{eq:bddness_of_psi} \\
    &\leq L \cdot C \left(1 + \left(\frac{\left\| \A^{(1)} \right\|_F}{\sqrt{n}}\right)^{r-1} + \left(\frac{\left\| \A^{(2)} \right\|_F}{\sqrt{n}}\right)^{r - 1} \right) \frac{\left\| \A^{(1)} - \A^{(2)} \right\|_F}{\sqrt{n}}, \label{eq:pl_assmpt_eta}
\end{align}
for some constants $C> 0, r \geq 1$. Here \eqref{eq:rev_tri_ineq} follows from the reverse triangle inequality for the metric $d_H$, \eqref{eq:hausdorff_bound} follows from the argument in the paragraph above, \eqref{eq:bddness_of_psi} follows from the fact that $\mathcal{X} \subseteq [L]^n$, and \eqref{eq:pl_assmpt_eta} follows from the pseudo-Lipschitz assumption on $(\V, \z) \mapsto U(\hat{\eeta}(\V, q(\z, \bPsi, \bvarepsilon))$. Applying Theorem \ref{thm:SE} with \[\bphi_n(\TTheta^t, \X\B ; \bPsi, \bvarepsilon) := \frac{1}{n} d_H(U^{-1}(\bPsi), \hat{\eeta}(\TTheta^t, q(\X \B, \bPsi, \bvarepsilon)),\] we obtain the first result in Proposition \ref{prop:hausdorff_asymptotics}. The result \eqref{eq:size_convergence} follows directly by applying Theorem \ref{thm:SE} to $\bphi_n(\TTheta^t, \X\B ; \bPsi, \bvarepsilon) := | \hat{\eeta}(\TTheta^t, q(\X \B ,\bPsi, \bvarepsilon) )|$. 
\end{proof}
\subsection{Proof of Proposition \ref{prop:pointwise_posterior}} \label{sec:proof_posterior}
\begin{proof}
    Let $\bphi_n(\V, \z ; \bPsi, \bvarepsilon) := p(\bpsi | \V, q(\z, \bPsi, \bvarepsilon))$. By assumption, $\bphi_n$ is uniformly pseudo-Lipschitz with respect to $(\V, \z)$. Applying Theorem \ref{thm:SE} to $\bphi_n$, we obtain that:
    \begin{align*}
        p(\bpsi | \TTheta^t, \y) = p(\bpsi | \TTheta^t, q(\X \B, \bPsi, \bvarepsilon)) \stackrel{\P}{\simeq}  \E_{{\V}_{\TTheta}^t, \Z} \, p(\bpsi | {\V}_{\TTheta}^t, q(\Z, \bPsi, \bvarepsilon)) =: \mu_p.
    \end{align*}
    Similarly, we apply Lemma 19 from \citep{gerbelot_graph-based_2023}, a concentration result regarding pseudo-Lipschitz functions of Gaussian random variables, to $\bphi_n({\V}_{\TTheta}^t, \Z; \bPsi, \bvarepsilon)$, to obtain $p(\bpsi | {\V}_{\TTheta}^t, q(\Z, \bPsi, {\bvarepsilon})) \stackrel{\P}{\simeq} \mu_p$. 
    Combining the aforementioned results, we obtain that for any $\epsilon >0$: 
    \begin{align*}
        &\P[ | p(\bpsi | \TTheta^t, \y) - p(\bpsi | {\V}_{\TTheta}^t, q(\Z, \bPsi, {\bvarepsilon}))| > \epsilon] \\
        &  \leq \P[|p(\bpsi | \TTheta^t, \y) -  \mu_p| > \epsilon/2] + \P[|p(\bpsi | {\V}_{\TTheta}^t, q(\Z, \bPsi, {\bvarepsilon})) - \mu_p| > \epsilon/2] \to 0.
    \end{align*}
\end{proof}
\section{Further Implementation and Experiment Details}\label{sec:further_implementation}
\paragraph{State Evolution Implementation}
Our state evolution implementation involves computing \eqref{eq:nu_B_SE_simplified}--\eqref{eq:kappa_theta_SE_simplified}. We estimate $\bnu_{\TTheta}^{t+1}, \bkappa_{\TTheta}^{t+1, t+1}, \bnu_{\B}^{t+1}, \bkappa_{\B}^{t+1, t+1}$ in \eqref{eq:nu_B_SE_simplified}--\eqref{eq:kappa_theta_SE_simplified} with finite $n, p$ via empirical averages, for a given change point configuration $\{\Psi_{\eta_\ell}\}_{\ell=0}^{L-1}$. Specifically, assuming $\bnu_{\TTheta}^{t}, \bkappa_{\TTheta}^{t, t}, \bnu_{\B}^{t}, \bkappa_{\B}^{t, t}$ have been computed, we compute $\bnu_{\TTheta}^{t+1}, \bkappa_{\TTheta}^{t+1, t+1}, \bnu_{\B}^{t+1}, \bkappa_{\B}^{t+1, t+1}$ as follows:
\begin{align}
&\bnu_{\B}^{t+1}\approx \sum_{\ell = 0}^{L-1} \frac{1}{n} \sum_{i \in [\eta_\ell, \eta_{\ell + 1})} \hat{\E}\left[ \partial_{1} \tilde{g}_i^t(\Z_{1}, (\V_{\TTheta}^t)_1, \Psi_{\eta_\ell}, \bar{\varepsilon}) \right], \label{eq:nu_b_MC}\\
&\bkappa_{\B}^{t+1, t+1}\approx \sum_{\ell = 0}^{L-1} \frac{1}{n} \sum_{i \in [\eta_\ell, \eta_{\ell + 1})} \hat{\E}\left[g_i^t\left( (\V_{\TTheta}^t)_1, q_1(\Z_{1}, \Psi_{\eta_\ell}, \bar{\varepsilon})\right)^\top g_i^t\left((\V_{\TTheta}^t)_1, q_1(\Z_{1}, \Psi_{\eta_\ell}, \bar{\varepsilon})\right)  \right], \label{eq:kappa_b_MC}\\
&\bnu^{t+1}_{\TTheta} \approx \frac{1}{\delta} \hat{\E}\left[\bar{\B} f_1^{t+1}(\V_{\B}^{t+1})^\top \right], \label{eq:nu_theta_MC}\\
&\bkappa_{\TTheta}^{t+1, t+1} \approx \frac{1}{\delta} \hat{\E}\left[\left(f_1^{t+1}(\V_{\B}^{t+1}) - \bar{\B}\brho^{-1} \bnu^{t+1}_{\TTheta}\right) \left(f^{t+1}(\V_{\B}^{t+1}) - \bar{\B}\brho^{-1} \bnu^{t+1}_{\TTheta} \right)^\top \right]. \label{eq:kappa_theta_MC}
\end{align}
where $\hat{\E}$ denotes an expectation estimate via Monte Carlo. For example, in the case of \eqref{eq:kappa_b_MC}, we generate 300 to 1000 independent samples of $(\Z_1, (\G_{\TTheta}^t)_1 ,  \bar{\varepsilon})$, with
$\bZ_1 \sim \N(\0, \brho)$, $(\G_{\TTheta}^t)_1 \distas{} \N(\0, \bkappa_{\TTheta}^{t, t})$, and $\bar{\varepsilon} \sim \P_{\bar{\varepsilon}}$. We form $(\V_{\TTheta}^t)_1$ according to the first row of \eqref{eq:V_TTheta}, that is, $(\V_{\TTheta}^t)_1 = \Z_1 \brho^{-1} \bnu_{\TTheta}^t + (\G_{\TTheta}^t)_1$. For $\ell \in \{0, \dots, L-1\}$, this yields a set of samples of the random variables $g_i^t\left( (\V_{\TTheta}^t)_1, q_1(\Z_{1}, \Psi_{\eta_\ell}, \bar{\varepsilon})\right) $ and $\partial_{1} \tilde{g}_i^t(\Z_{1}, (\V_{\TTheta}^t)_1, \Psi_{\eta_\ell}, \bar{\varepsilon})$ (the latter function is computed using Automatic Differentiation (AD)). We then compute $\hat{\E}$ in \eqref{eq:nu_b_MC} and \eqref{eq:kappa_b_MC} by averaging. 
The $\hat{\E}$ terms in \eqref{eq:nu_theta_MC}--\eqref{eq:kappa_theta_MC} are similarly computed.

The ensemble state evolution recursion \eqref{eq:nu_bar_B_SE}--\eqref{eq:kappa_bar_theta_SE} is used to compute the optimal denoisers $\{g^{*t}, f^{*t+1}\}_{t\ge 0}$ according to Proposition \ref{prop:opt_ensemble_ft_gt}. 
The expectation in the ensemble state evolution iterates   \eqref{eq:nu_bar_B_SE}--\eqref{eq:kappa_bar_theta_SE} are estimated through sample averages and AD, similar to the simplified state evolution iterates. 
\paragraph{Additional Numerical Results}
\begin{figure}[t!]
    \centering
    \subfloat[Solid purple: AMP using $L=3, 4$ and the optimal denoiser $f^t$ with the sparsity level estimated using  CV.]{\includegraphics[width=0.45\columnwidth]{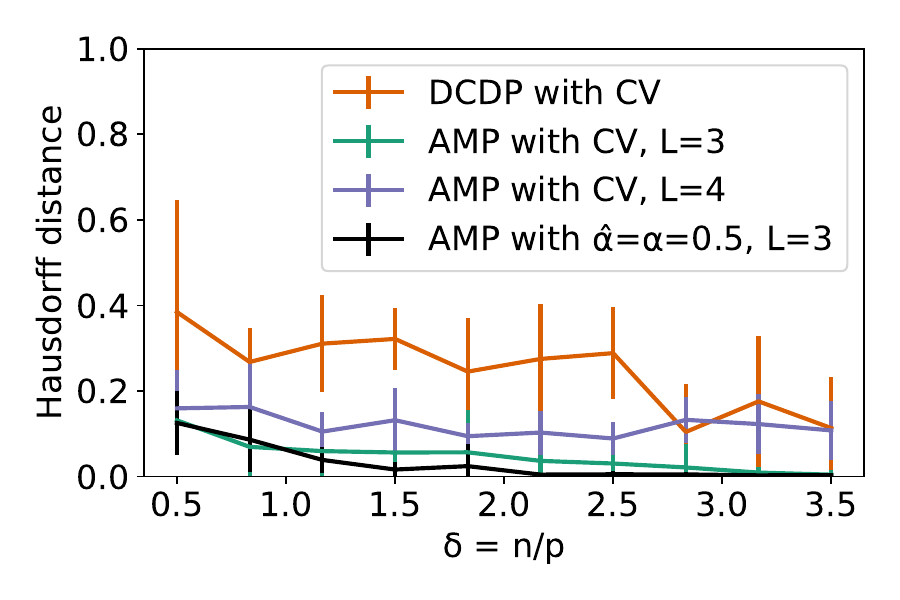}\label{fig:Lmax>L}}\hspace{1cm}
    \subfloat[Dashed purple: AMP using $L=3$ and a suboptimal soft thresholding (ST) denoiser $f^t$  whose threshold is selected using CV.]{\includegraphics[width=0.45\columnwidth]{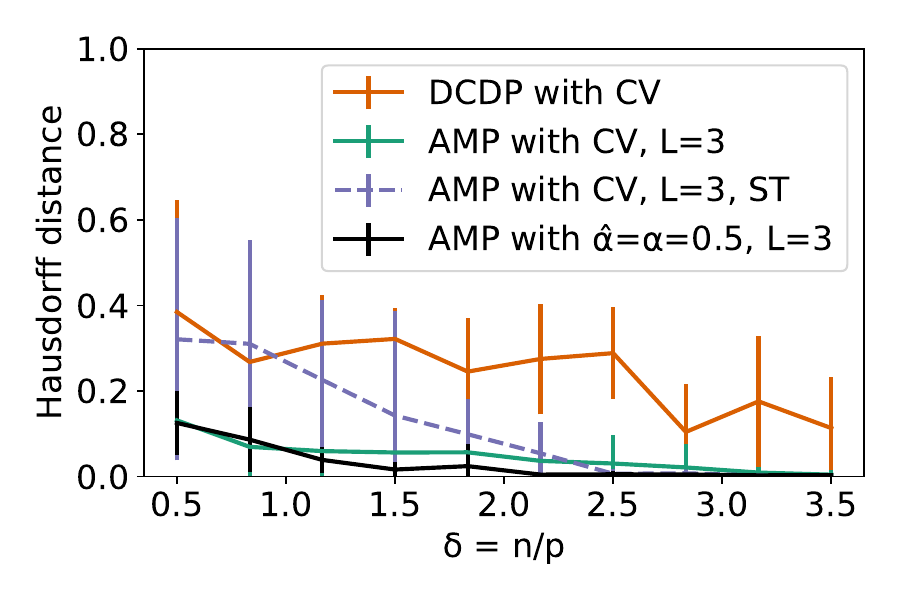}\label{fig:ST}}
    \caption{Comparison between AMP and  $\DCDP$ in the experimental setting of Figure \ref{fig:DPDU_L3}.
    Solid red: $\DCDP$ with hyperparameters chosen using CV. Solid green: AMP using $L=3$, the optimal denoiser $f^t$ with the sparsity level estimated using CV. Solid black: AMP using the true signal prior.} \label{fig:extra_comparison_DCDP}
\end{figure}
\begin{figure}[t!]
    \centering
     \subfloat{\includegraphics[width=0.45\columnwidth]{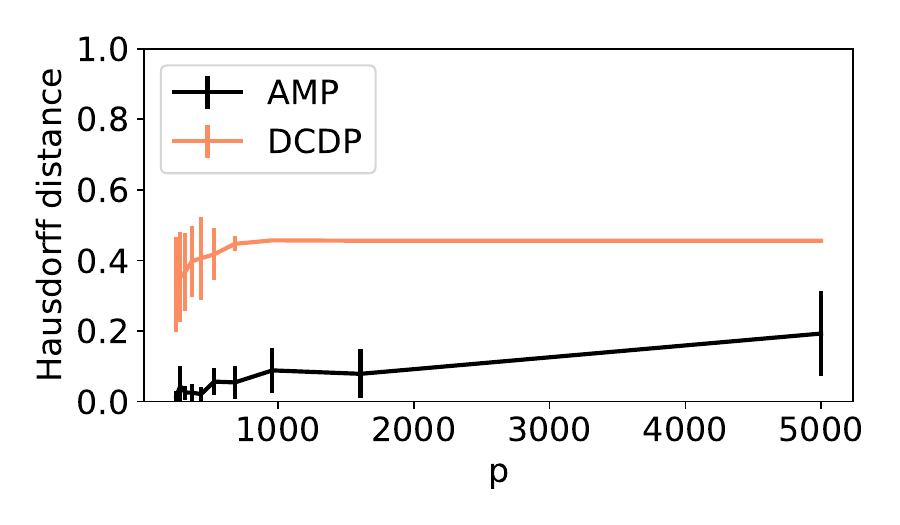}}\hspace{1cm}
    \subfloat{\includegraphics[width=0.45\columnwidth]{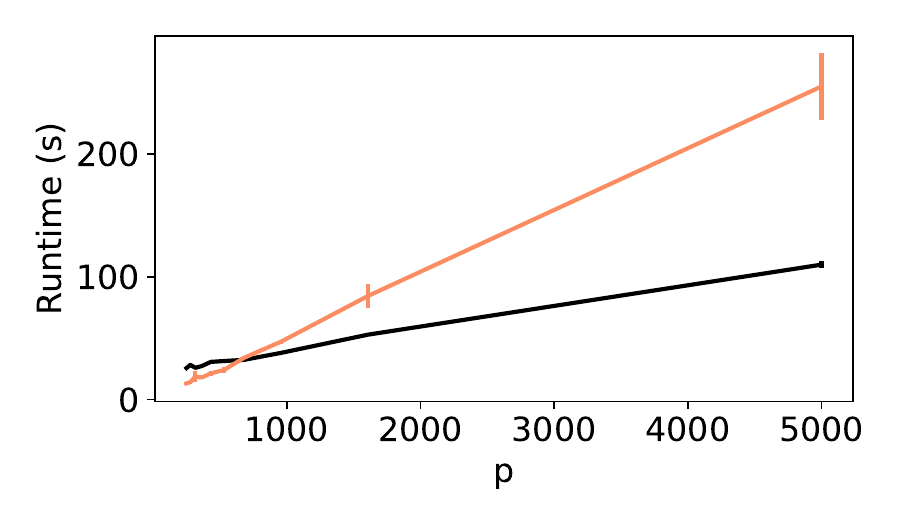}}
\vspace{-0.5cm}
    \caption{AMP vs. $\DCDP$ for fixed $n=500$ and varying $p$,  with $\sigma=0.1$, $L^*=3$ and   sparse signal  prior $\P_{\bar{\B}}=0.5\N(\0, \delta\I) + 0.5\delta_{\0}$. AMP uses $\Delta=n/10$ and $L=3$.}
    \label{fig:vary_p}
\vspace{-0.4cm}
\end{figure}
Figure \ref{fig:extra_comparison_DCDP} shows results from an additional set of  experiments comparing AMP with $\DCDP$ (the fastest algorithm in Figure \ref{fig:DPDU_L3}) for the linear model, using different change point priors $\pi_{\bar{\bPsi}}$ (Figure \ref{fig:Lmax>L}) or  denoisers $f^t$ (Figure \ref{fig:ST}). In both figures, the solid red plot shows  $\DCDP$ performance with hyperparameters chosen using CV.  Solid black corresponds to  AMP  using the true signal prior. Solid green  corresponds to  AMP  using $L=3$, the optimal denoiser $f^t$ with the sparsity level estimated using CV. 
In Figure \ref{fig:Lmax>L}, AMP performs slightly worse with $L=4$ instead of $L=3$, because the prior assigns non-zero probability to the change point configurations with $L = 4$ signals, which is mismatched from the ground truth $L^*=3$. In Figure \ref{fig:ST}, AMP performs slightly worse at lower $\delta$ when using  a suboptimal soft thresholding (ST) denoiser. Nevertheless, in both Figures \ref{fig:Lmax>L} and \ref{fig:ST}, AMP largely  outperforms $\DCDP$ despite the suboptimal choices of  prior or denoiser.

Figure \ref{fig:vary_p} compares AMP with $\DCDP$ for $n=500$ and varying $p$, both using hyperparameters chosen via cross-validation. The runtime shown is the average runtime per set of CV parameters. For a fixed set of hyperparameters,  the time complexity of $\DCDP$  scales as LASSO$(n,p)$ compared to $O(np)$ for AMP. The runtime of $\DCDP$ therefore grows faster with increasing $p$.
\paragraph{Experiment Details} 
For synthetic data experiments in Section \ref{sec:experiments}, we initialize AMP with $f^0(\B^0)$ sampled row-wise independently from the prior $\P_{\bar{\B}}$ used to define the ensemble state evolution \eqref{eq:nu_bar_B_SE}--\eqref{eq:kappa_bar_theta_SE}. Figures \ref{fig:DPDU_L3}, \ref{fig:satellite}, \ref{fig:logistic_compare}, \ref{fig:extra_comparison_DCDP} and \ref{fig:vary_p} use Bernoulli-Gaussian priors taking the form $\P_{\bar{\B}}=\alpha \normal(\bzero, \sigma^2\bI) + (1-\alpha) \delta_{\bzero}$ with $\alpha\in (0,1)$. Figure \ref{fig:satellite} uses $\alpha = 1/6$ and $\sigma  = 2.5$.  Figures \ref{fig:DPDU_L3}, \ref{fig:logistic_compare}, \ref{fig:extra_comparison_DCDP} and \ref{fig:vary_p} use $\alpha=0.5$ and $\sigma^2 = \delta$  to ensure that the signal power $ \E[(\X_i)^\top \bbeta^{(i)}]^2$  is held constant for varying $\delta$, and they use two change points at $n/3$ and $8n/15$ as the ground truth. In  Figures \ref{fig:DPDU_L3}, \ref{fig:extra_comparison_DCDP} and \ref{fig:vary_p},  AMP uses cross validation (CV) over 5 values of $\hat{\alpha}:\{ 0.1,  0.3,  0.45, 0.6,  0.9\}$ which do not contain the true $\alpha=0.5$.
 $\DCDP$, $\DPDU$, and $\DP$ each have two hyperparameters: one corresponding to the $\ell_1$ penalty and the other penalizing the number of change points. We run cross-validation on these hyperparameters for $12$, $12$, or  $42$ pairs of values, respectively, using the suggested values from the original papers.  

For the experiments in Figure \ref{fig:charcoal_L3}, we consider signals with sparse changes between adjacent signals. The prior takes the following form: 
\begin{align}
&\beta^{(\eta_0)}_j \distas{\text{i.i.d.}} \N(0, \kappa^2) \nonumber\\
&\beta^{(\eta_\ell)}_j = 
\left\{
    \begin{array}{lr}
        \beta^{(\eta_{\ell-1})}_j, & \text{with probability } 1 - \alpha\\
        \nu(\beta^{(\eta_{\ell-1})}_j + w_\ell), & \text{with probability } \alpha
    \end{array}
\right. , \quad \ell \in [L],
\label{eq:sparse_diff_prior}
\end{align}
where $w_\ell \distas{\text{i.i.d.}} \N(0, \sigma_w^2)$ creates the sparse change between adjacent signals, and $\nu := \sqrt{\frac{\kappa^2}{\kappa^2 + \sigma_w^2}}  $ is a rescaling factor that ensures uniform signal magnitude $\E[(\beta_j^{(\eta_{0})})^2]=\dots =\E[(\beta_j^{(\eta_{L})})^2] =\kappa^2$.
 We run the experiment with sparsity level $\alpha=0.5$, variance of the entries of the first signal $\kappa^2 = {8\delta}$, and perturbation to each consecutive signal of variance $\sigma_w^2=400{\delta}$, for varying $\delta$. This ensures that the signal power $ \E[(\X_i)^\top \bbeta^{(i)}]^2$ and the signal difference $\E[(\X_i)^\top \bbeta^{(i)} - (\X_{i+1})^\top \bbeta^{({i+1})}]^2$ for $i\in \{\eta_\ell\}_{\ell=1}^{L-1} $ are  held constant for varying $\delta$. There are two change points at $n/3$ and $8n/15$. Without perfect knowledge  of the magnitude $\sigma_w^2$ and sparsity level $\alpha$ of the sparse difference vector, AMP assumes $\hat{\sigma}_w^2=2500\delta$ and $\hat{\alpha}=0.9$ in the experiments in Figure \ref{fig:charcoal_L3}. 
 \vspace{-0.4cm}
\section{Computational Cost} \label{sec:computational_cost}
\paragraph{Computing $g^{*t}, f^{*t+1}$} For $i\in [n]$, the function $g_i^{*t}$ in Proposition \ref{prop:opt_ensemble_ft_gt} only depends on $\bar{\bPsi}$ in \eqref{eq:g_middle_term_expansion} through the marginal probability of $\bar{\Psi}_i$, that is, $\pi_{\bar{\Psi}_i}: \ell \mapsto \sum_{\{\bpsi : \psi_i = \ell\}} \pi_{\bar{\bPsi}}(\bpsi)$. This means $g_i^{*t}$ can be efficiently computed, involving only a sum over $\bar{\Psi}_i\in [L]$. Indeed, from the implementation details in Appendix \ref{sec:full_computation_gt}, both $f_j^{*t}$ and $g_i^{*t}$  can be computed in $O(L^3)$ time for each $i, j$. Thus $f^{*t}$ and $g^{*t}$ can be computed in  $O(n L^3)$. The per-iteration computational cost of AMP is therefore dominated by the  matrix multiplications in \eqref{eq:amp}, which are $O(npL)$.

\paragraph{Computing Change Point Estimators} For estimators $\hat{\eeta}$ with $O(np)$ runtime, the combined AMP and estimator computation can be made to run in $O(np)$ time by selecting denoisers $f^{t}, g^{t}$ as in Proposition \ref{prop:opt_ensemble_ft_gt}. For example, the $\argmax$ in \eqref{eq:example_estimator} can be replaced with a greedy best-first search: search for the location of one change point at a time, conditioning on past estimates of the other change points. This will yield at most $L$ rounds of searching over $O(n)$ elements, resulting in $O(np)$ total runtime. 

\paragraph{Computing the Approximate and Exact Posteriors} 
For $\bPsi \in \mathcal{X}$, the denominator in \eqref{eq:posterior_density_2} only requires evaluating a polynomial number of terms, ${n \choose {L-1}}$. Therefore, choosing $f^t, g^t$ as in Proposition \ref{prop:opt_ensemble_ft_gt}, the computational complexity of computing the quantity $p_{\bar{\bPsi} | \bar{\V}_{\TTheta}^t, q(\Z, \bar{\bPsi}, \bar{\bvarepsilon})}(\bpsi | \TTheta^t, \y)$ from \eqref{eq:argmax_approx_posterior} along with AMP is $O(npL + nL^3 + n^{L-1}) = O(np + n^{L-1})$. Further, assuming the expectations in \eqref{eq:nu_B_SE}--\eqref{eq:kappa_theta_SE} can be well-approximated in $O(n)$ time, $p_{\bar{\bPsi} | {\V}_{\TTheta}^t, q(\Z, \bar{\bPsi}, \bar{\bvarepsilon})}(\bpsi | \TTheta^t, \y)$ from \eqref{eq:posterior_density_2} along with AMP can be computed in $O(np + n^{L-1})$ time.

\section{Background on AMP for Generalized Linear Models} \label{app:background}

Here we review the  AMP algorithm and its state evolution characterization for  the standard Generalized Linear Model (GLM) without change points. As in Section \ref{sec:main}, let  $\B \in \reals^{p \times L}$ be a signal matrix and let $\X \in \reals^{n \times p}$ be a design matrix. The observation $\y \in \reals^{n \times L}$
is produced as 
\begin{equation}
    \y = q(\X \B, \bvarepsilon)  \in \reals^{n \times L}, \label{eq:model_psi_nochangept}
\end{equation}
where $\bvarepsilon \in \reals^n$ is a noise vector, and $q: \reals^L \to \reals^L$ is a known output function. The only difference between this model and the one in \eqref{eq:model_psi} is the absence of the signal configuration vector $\bPsi$. The AMP algorithm for the GLM in \eqref{eq:model_psi_nochangept}  was derived by \citet{Ran11} for the case of vector signals ($L=1$);  see also Section 4 of \cite{feng_unifying_2022}. Here we discuss the algorithm  for the general case ($L \ge 1$), which can be found in \cite{tan_mixed_2023}. 
For ease of exposition, we will make the following standard assumption (see \ref{it:simp-ass-1} on p.\pageref{it:simp-ass-1}): as $n, p \to \infty$, the empirical distributions of $\{\B_j \}_{j \in [p]}$  and $\{\varepsilon_i\}_{i \in [n]}$ converge weakly to laws $\P_{\bar{\B}}$  and  $\P_{\bar{\varepsilon}}$, respectively,   with bounded second moments. We also recall that $n/p \to \delta$ as $n,p \to \infty$.

The AMP algorithm for the model \eqref{eq:model_psi_nochangept} is the same as the one in \eqref{eq:amp}, but due to the assumption above, we can take $f^t, g^t$ to be separable, that is, $f^t: \reals^L \to \reals^L$  and  $g^t:  \reals^L  \times \reals^L \to \reals^L$ act row-wise on their matrix inputs. Then, the matrices $\F^t$ and $\C^t$ in \eqref{eq:amp} can be simplified to $\C^t= \frac{1}{n}\sum_{i=1}^n  \d g^t(\TTheta_i^t,y_i)$ and $\F^{t}=\frac{1}{n}\sum_{j=1}^p \d f^t(\B_j^{t})$, where $\d g^t$ and $\d f^t$ denote the $L \times L$ Jacobians with respect to the first argument.

\paragraph{State evolution}  The memory terms $-\hat{\bR}^{t-1} (\F^t)^\top$ and $-\hat{\B}^{t} (\C^t)^\top$ in \eqref{eq:amp} debias the iterates $\TTheta^t$ and $\B^{t+1}$ and enable a succinct distributional characterization, guaranteeing that their empirical distributions converge to well-defined limits as $n,p \to \infty$. Specifically, Theorem 1 in \cite{tan_mixed_2023} shows that for each $t \ge 1$, the empirical distribution of the rows of $\B^t$ converges to the law of a random vector $\tilde{\V}_{\B}^{t} := \bar{\B} \tilde{\bnu}^{t}_{\B} + \tilde{\G}^{t}_{\B} \, \in \reals^{1 \times L}$, where
$\tilde{\G}^{t+1}_{\B} \sim \normal(0, \tilde{\bkappa}_{\B}^{t, t})$ is independent of $\bar{\B} \sim \P_{\bar{\B}}$. Similarly, recalling that $\TTheta = \bX \B \in \reals^{n \times L}$, the empirical distribution of the rows of $(\TTheta, \, \TTheta^t)$ converges to the law of the random vectors $(\tilde{\bZ}, \tilde{\bZ}\tilde{\bnu}^{t}_{\TTheta} + \tilde{\G}^{t}_{\TTheta})$, where $\tilde{\G}^{t+1}_{\TTheta} \sim \normal(0, \tilde{\bkappa}_{\TTheta}^{t, t})$ is independent of $\tilde{\bZ} \sim \normal\left(0, \delta (\E[\bar{\bB} \bar{\bB}^\sT])^{-1} \right)$. The deterministic $L \times L$ matrices $\tilde{\bnu}^{t}_{\B}, \tilde{\bkappa}_{\B}^{t, t}, \tilde{\bnu}_{\TTheta}^{t, t} $, and $\tilde{\bkappa}_{\TTheta}^{t, t}$ can be recursively computed  for $t \ge 1$ via a state evolution recursion that depends on $f^t, g^t$ and the limiting laws $\P_{\bar{\B}}$  and  $\P_{\bar{\varepsilon}}$.

The state evolution characterization allows us to compute asymptotic performance measures such as the MSE of the AMP algorithm. Indeed, for each $t \ge 1$, we almost surely  have 
$\lim_{p \to \infty} \frac{1}{p} \| f^t(\B^t) - \B  \|_F^2 
= 
\E[\|f^t(\tilde{\V}_{\B}^{t}) - \bar{\B} \|^2_2]$, where the expectation on the right can be computed using the joint law of the $L$-dimensional random vectors  $\bar{\B}$ and  $\tilde{\V}_{\B}^{t} = \bar{\B} \tilde{\bnu}^{t}_{\B} + \tilde{\G}^{t}_{\B}$.

In the model \eqref{sec:main} with change points, we have an additional signal configuration vector $\bPsi$, because of which we cannot take the AMP denoising function $g^t$ to be separable, even under Assumption \ref{it:simp-ass-1}. This leads to a more complicated state evolution characterization, as described in Section \ref{sec:main}.

{\small{
\bibliography{references}
}}

\end{document}